\documentclass[smallextended,referee,envcountsect,]{svjour3}
\smartqed

\newtheorem{assumption}{Assumption}
\usepackage{amsmath,amssymb,mathtools,bm, bbm}
\usepackage{hyperref,lmodern,xcolor,subcaption, enumitem}
\usepackage{float,booktabs,url,cite}
\usepackage{graphicx,tabularx}
\journalname{JOTA}

\allowdisplaybreaks

\begin{document}
\sloppy

\title{Convergence Analysis of SGD under Expected Smoothness}
\subtitle{}

\author{Yuta Kawamoto \and Hideaki Iiduka}

\institute{Yuta Kawamoto \at
             Computer Science Course, Graduate School of Science and Technology, Meiji University\\
             Kawasaki, Japan \\
             \email{ee227008@meiji.ac.jp}
           \and
             Hideaki Iiduka  (Corresponding author) \at
             Department of Computer Science,  School of Science and Technology, Meiji University \\
             Kawasaki, Japan \\
             \email{iiduka@cs.meiji.ac.jp}
}

\date{Received: date / Accepted: date}

\maketitle

\begin{abstract}
    Stochastic gradient descent (SGD) is the workhorse of large-scale learning, yet classical analyses rely on assumptions that can be either too strong (bounded variance) or too coarse (uniform noise). The expected smoothness (ES) condition has emerged as a flexible alternative that ties the second moment of stochastic gradients to the objective value and the full gradient. This paper presents a self-contained convergence analysis of SGD under ES. We (i) refine ES with interpretations and sampling-dependent constants; (ii) derive bounds of the expectation of squared full gradient norm; and (iii) prove sublinear convergence rates with explicit residual errors for various step-size schedules. Furthermore, we empirically validate our theoretical framework by training a ResNet-18 model on the CIFAR-100 dataset. We uncover a two-phase dynamic in the gradient noise structure, transitioning from a suboptimality-dominated regime to a geometry-dominated regime. 
    Our treatment unifies and extends recent theoretical threads in the literature. 
\end{abstract}

Communicated by Wei-Shih Du



\section{Introduction}
    \subsection{Background}
        Stochastic Gradient Descent (SGD) has become the \emph{de facto} standard for large-scale optimization in modern machine learning, thanks to its remarkable scalability, modest memory requirements, and ability to cope with the nonconvex landscapes typical of deep learning, as discussed by Bottou et al. \cite{bottou2018optimization}. Although today it is mostly associated with training neural networks, SGD originated in the early days of stochastic approximation, as discussed by Bottou et al. \cite{bottou2018optimization}, with its application to neural networks developed by Hornik et al. \cite{hornik1989multilayer} and Rumelhart et al. \cite{rumelhart1986learning}.

        \subsubsection*{Historical foundations.}
            The story begins with the seminal work of Robbins and Monro \cite{robbins1951stochastic}, who introduced a stochastic approximation method to solve root-finding problems with noisy measurements. Shortly afterward, Kiefer and Wolfowitz \cite{kiefer1952stochastic} developed a gradient-free variant. These works laid the probabilistic and algorithmic foundations of SGD. In the 1960s, Polyak \cite{polyak1964some} introduced momentum acceleration \cite{nesterov1983method, polyak1964some, sutskever2013importance}, and later Polyak and Juditsky \cite{polyak1992acceleration} established the power of iterate averaging (Polyak-Ruppert averaging) by proving improved asymptotic efficiency \cite{polyak1992acceleration, glorot2010understanding, he2015delving}. By the 1990s, SGD had become a standard tool in statistics, signal processing, and control theory.

        \subsubsection*{Classical bounded variance assumption.}
            Analyses of SGD traditionally rely on the \emph{bounded variance} condition  \cite{nemirovski2009robust}, requiring
            \begin{align} \label{eq:bv}
                \exists \sigma \geq 0 \quad \forall \bm{x} \in \mathbb{R}^{d} \quad \mathbb{E}_{\bm{v}} \left \Vert g_{\bm{v}}(\bm{x}) - \nabla f(\bm{x}) \right \Vert^{2} \leq \sigma^{2},
            \end{align}
            where $\bm{x} \in \mathbb{R}^{d}$ is a parameter of a deep neural network (DNN), $f$ is the objective function used in training the DNN, and $g_{\bm{v}}(\bm{x})$ is the stochastic gradient of $f$ at $\bm{x}$, and $\mathbb{E}_{\bm{v}}[\cdot]$ is the expectation over randomness $\bm{v}$. This assumption yields clean rates in convex optimization but ignores that variance often depends on the iterate. In modern over-parameterized models, variance can be much larger away from the optimum, violating bounded variance \cite{bottou2018optimization}.

        \subsubsection*{Variance growth models.}
            To overcome this issue, researchers have proposed \emph{variance growth conditions} \cite{schmidt2013fast}. A representative form is 
            \begin{align}
                \exists A \geq 0 \quad \exists B \geq 0 \quad \mathbb{E}_{\bm{v}} \left \Vert g_{\bm{v}}(\bm{x}) - \nabla f(\bm{x}) \right \Vert^{2} \leq A \left \Vert \nabla f(\bm{x}) \right \Vert^{2} + B,
            \end{align}
            which allows the variance to shrink near stationary points but grow with gradient magnitude. This condition was instrumental in analyzing variance-reduced methods such as the SVRG method proposed by Johnson and Zhang \cite{johnson2013accelerating}, and SAGA proposed by Defazio et al. \cite{defazio2014saga}, which combine stochastic updates with memory or control variates to achieve linear convergence in convex settings.

        \subsubsection*{Strong growth condition.}
            A more aggressive relaxation is the \emph{strong growth condition}, requiring
            \begin{align} \label{eq:sg}
                \exists \rho > 0 \text{ } \forall \bm{x} \in \mathbb{R}^{d} \quad \mathbb{E}_{\bm{v}} \left \Vert g_{\bm{v}}(\bm{x}) \right \Vert^{2} \leq \rho \left \Vert \nabla f(\bm{x}) \right \Vert^{2}.
            \end{align}
            This implies that stochastic gradients are well aligned with the true gradient. The condition underpins recent analyses in over-parameterized learning \cite{vaswani2019fast}, but is overly restrictive in heterogeneous data regimes, such as federated learning, where client gradients may deviate substantially.

        \subsubsection*{Expected Smoothness: a unifying framework.}
            The limitations of earlier assumptions motivated the \emph{Expected Smoothness (ES)} framework introduced by Gower et al. \cite{gower2019sgd} and further developed by Khaled and Richt\'arik \cite{khaled2020better}, which posits that there exist $A, B, C \ge 0$ such that, for all $\bm{x} \in \mathbb{R}^{d}$, 
            \begin{align} \label{eq:es_condition}
                \mathbb{E}_{\bm{v}} \left \Vert g_{\bm{v}}(\bm{x}) \right \Vert^{2} \leq 2A(f(\bm{x}) - f^{\star}) + B \left \Vert \nabla f(\bm{x}) \right \Vert^{2} + C,
            \end{align}
            where $f^{\star}$ is the minimum value of $f$ over $\mathbb{R}^{d}$. The ES condition is also called the {\emph{$ABC$ condition}} introduced by Gorbunov et al. \cite{gorbunov2023unified}, since it uses positive constants $A$, $B$, and $C$. The ES condition subsumes all previous ones: bounded variance ($A = B = 0$), variance growth ($B = 0$), and strong growth ($A = 0$). Crucially, ES introduces a dependence on suboptimality $f(\bm{x}) - f^{\star}$, reflecting the empirical observation that stochastic noise decays as the iterates approach a minimizer. This makes ES both theoretically powerful and practically realistic, encompassing problems such as least-squares regression, logistic regression, and many machine learning models.

            Our work leverages the ES framework to provide fully explicit, sampling-specific convergence guarantees in terms of the smoothness constants $\{L_{i}\}$, sampling probabilities, and the data heterogeneity $\Delta^{\inf}$. Unlike prior analyses in Nonconvex World \cite{khaled2020better}, which required knowledge of the total iteration count $K$ for constant step sizes, Our general bound provides a unified framework for analyzing the convergence behavior of SGD under different step-size sequences.

            To make these features intuitive, Table~\ref{tab:intro-rates} summarizes the typical convergence rates for the four common step-size schedules under ES. These schedules exhibit distinct asymptotic behaviors, highlighting the flexibility and efficiency of our approach \cite{goyal2017accurate, loshchilov2017sgdr, smith2017superconvergence, smith2017dontdecay}. 

            \begin{table*}[htbp]
                \centering
                \caption{Typical convergence rates of SGD with different step-size schedules under ES \eqref{eq:es_condition} and $L$-smoothness. The constants are omitted for simplicity; see the main text for detailed ES-dependent bounds. For all schedules, the step-size parameters are chosen such that $\eta_{\max}:=\sup_{k\geq0}\eta_{k} < 2/(LB)$. \label{tab:intro-rates} }
                \small
                \begin{tabular}{@{} l l l @{}}
                    \toprule
                        Step-size Schedule & Typical Choice & Rate for $\min \limits_{0 \leq k \leq K} \mathbb{E} \left \Vert \nabla f(\bm{x}_{k}) \right \Vert^{2}$ \\ \hline 
                    \midrule
                        Constant & $\eta_{k} = \eta$ & $O(1/K) + O(\eta)$ \\
                        Harmonic & $\eta_{k} = \eta/(k+1)$ & $O(1/\log(K))$ \\
                        Polynomial Decay & $\eta_{k} = \eta/(k + 1)^{\alpha}$, $0.5 < \alpha < 1$ & $O(1/K^{1-\alpha})$ \\
                        Cosine Annealing & $\eta_{k} = \eta \cdot \frac{1}{2} (1+\cos(\pi k/K))$ & $O(1/K) + O(\eta)$ \\
                    \bottomrule
                \end{tabular}
            \end{table*}

        \subsubsection*{Modern developments.}
            The adoption of ES has enabled sharp convergence results in nonconvex optimization, including tight descent lemmas, explicit rates under constant or decaying step sizes, and improved analyses of distributed and federated SGD studied by Stich \cite{stich2019local}, Karimireddy et al. \cite{karimireddy2020mime}, and Umeda and Iiduka \cite{umeda2024accelerates}. At the same time, adaptive methods such as AdaGrad proposed by Duchi et al. \cite{duchi2011adaptive}, ADADelta proposed by Zeiler \cite{zeiler2012adadelta} and Adam proposed by Kingma and Ba \cite{kingma2015adam} have further expanded the scope of stochastic optimization, although their convergence properties are less well understood \cite{reddi2018on}. Recent studies have continued to refine the picture, connecting ES to Polyak-Łojasiewicz conditions studied by Karimi et al. \cite{karimi2016linear}, analyzing stochastic methods with momentum methods studied by Allen-Zhu and Li \cite{allen2019sgd}, and exploring batch-size-related phenomena discussed by Hardt et al. \cite{hardt2016train}, Keskar et al. \cite{keskar2017on}, and Dinh et al. \cite{dinh2017sharp}.
            
    \subsection{Motivation} \label{subsec:motivation}
        Classical analyses of SGD typically relied on restrictive variance assumptions, such as the bounded variance condition \eqref{eq:bv} or the strong growth condition \eqref{eq:sg}. While mathematically convenient, these assumptions fail to capture the heterogeneous and data-dependent noise observed in modern large-scale optimization. The ES framework \eqref{eq:es_condition} has recently emerged as a more realistic and unifying model of stochastic gradient noise \cite{gower2019sgd, khaled2020better}. ES characterizes stochastic gradient variance via three constants $(A, B, C)$ depend on the problem geometry and the sampling strategy. This framework subsumes earlier variance models and has become a central tool for analyzing SGD in nonconvex settings. Despite this progress, there remain two important gaps in the literature that motivate the present study: 

        \subsubsection*{Limitation $1$ of Khaled and Richt\'arik \cite{khaled2020better}:}
            Khaled and Richt\'arik \cite{khaled2020better} established sharp convergence guarantees for SGD under ES, showing it is a weaker yet more expressive than the classical assumptions. However, their analysis is primarily descriptive: it introduces ES and proves rates under general step-size policies, but it does not provide a \emph{self-contained and systematic derivation} of the descent lemmas and convergence proofs tailored to various practical step-size schedules. As a result, it remains difficult for readers and practitioners to directly trace how ES constants $(A, B, C)$ interact with specific algorithmic choices, such as of constant, polynomially decaying, or cosine step sizes. 

        \subsubsection*{Limitation $2$ of Umeda and Iiduka \cite{umeda2024accelerates}:}
            Umeda and Iiduka \cite{umeda2024accelerates} advocate increasing the batch size together with the learning rate, and they show that such joint scaling can accelerate SGD in theory and practice. However, it is important to clarify the theoretical distinction between our setting and prior work. The analysis in \cite{umeda2024accelerates} relies on the classical \emph{bounded variance} assumption \eqref{eq:bv} combined with standard $L$-smoothness. Under these assumptions, they derived convergence rates dependent on the batch size $\tau_k$. In contrast, our Expected Smoothness (ES) framework \eqref{eq:es_condition} generalizes \eqref{eq:bv}. Specifically, the bounded variance assumption corresponds to the special case where $A=0$, $B=0$, and $C=\sigma^{2}$ (after variance decomposition). By analyzing SGD under the full ES condition ($A, B \ge 0$), our work covers a strictly broader class of objective functions and noise structures, including those where variance grows with the gradient norm or function value. Consequently, their analysis is not explicitly grounded in the ES framework: the admissibility of scaling rules is derived via upper-bound comparisons between scheduler classes, without connecting these prescriptions to the sampling-dependent constants $(A, B, C)$ that govern the true variance-descent tradeoff. This leaves open how to systematically justify scheduler design from the ES perspective. 

        \subsubsection*{Our motivation.}
            The present paper addresses these two gaps by providing a \emph{self-contained and systematic convergence analysis of SGD under ES}. Specifically, we (i) revisit the meaning of the ES constants and collect explicit formulas for common sampling schemes, and (ii) establish convergence rates for a variety of step-size policies including constant, decaying, polynomial, and cosine schedules. In so doing, we bridge the descriptive analysis of \cite{khaled2020better} with the constructive scheduler insights of \cite{umeda2024accelerates}, offering both rigorous proofs and practical guidance. This contribution deepens our theoretical understanding of SGD under ES and provides a principled foundation for the design of step-size and batch-size schedules in large-scale nonconvex optimization.

    \subsection{Contributions} \label{subsec:contributions}
        Building on the ES paradigm, this paper provides a unified and self-contained analysis of SGD in nonconvex optimization. Our contributions can be summarized as follows: 

        \begin{itemize}
            \item 
                We provide detailed interpretations of the ES constants $(A, B, C)$ under various sampling strategies (with and without replacement, importance sampling, mini-batching) and heterogeneous data distributions (\textbf{Proposition \ref{prop:es-valid}}). Unlike \cite{khaled2020better}, which treated $(A, B, C)$  as abstract constants, we derive explicit formulas showing how they scale with batch size and heterogeneity, making ES a practical tool for understanding the variance-descent tradeoff. 

            \item 
                
                We establish a global convergence theorem (\textbf{Theorem~\ref{thm:1}}) for SGD under the ES condition ~\eqref{eq:es_condition}. Unlike the constant-step-size convergence analysis in \cite{khaled2020better}, our theorem is formulated for a general step-size sequence $(\eta_{k})_{k \geq 0}$. The resulting bound makes the role of the ES constants $(A, B, C)$ explicit: $B$ determines the admissible step-size range, while $A$ and $C$ appear directly in the corresponding convergence terms. This formulation provides a unified basis for analyzing a variety of step-size schedules. 

            \item 
                Beyond the general nonconvex rate bounds in \cite{khaled2020better}, we analyze concrete step-size rules-including constant, polynomially decaying, and cosine schedules -- within the ES framework (\textbf{Table \ref{tab:intro-rates}}) and \textbf{Corollary \ref{cor:general_rate}}. This systematic comparison clarifies the tradeoffs between variance reduction and descent speed, and provides practitioners with explicit guidance on how to select or adapt step sizes under different noise regimes. 

            \item 
                We conduct experiments on modern deep-learning tasks \cite{srivastava2014dropout, ioffe2015batch, glorot2010understanding} to show that the ES model captures the evolving two-phase structure of stochastic noise during training (\textbf{Section \ref{sec:empirical}}). Such an empirical validation is largely absent in \cite{khaled2020better}, which focused on theoretical guarantees, but is essential for demonstrating the practical relevance of ES-based analyses. 
        \end{itemize}

        In summary, whereas prior work established ES as a general analytical tool, our contribution is to make ES \emph{operational}: we provide explicit constants, self-contained proofs, systematic step-size analyses, and empirical demonstrations. This bridges the descriptive theory of \cite{khaled2020better} with constructive and practical guidance for modern stochastic optimization.

\section{Stochastic Minimization Problem and Assumptions \label{sec:ES}}
    We start with mathematical preliminaries. Let $\mathbb{R}^{d}$ be a $d$-dimensional Euclidean space with inner product $\langle \cdot, \cdot \rangle$ and its induced norm $\left \Vert \cdot \right \Vert$. In this paper, $\bm{x}_{k}$ denotes the $k$-th approximation of a minimizer of an objective function $f$ (e.g. the empirical loss $f$ defined by $f = \frac{1}{n} \sum_{i=1}^{n} f_{i}$, where $f_{i}$ is the loss function corresponding to the $i$-th labeled training data). The operator norm of a matrix $X$ is denoted by $\left \Vert X \right \Vert_{\mathrm{op}}$. 

    $f \colon \mathbb{R}^{d} \to \mathbb{R}$ is bounded below if there exists $f^{\star} \in \mathbb{R}$ such that, for all $\bm{x} \in \mathbb{R}^{d}$, $f(\bm{x}) \geq f^{\star}$. Let $\nabla f \colon \mathbb{R}^{d} \to \mathbb{R}^{d}$ be the gradient of a differentiable function $f \colon \mathbb{R}^{d} \to \mathbb{R}$. Let $L>0$. $f$ is said to be $L$-smooth if $\nabla f$ is $L$-Lipschitz continuous; that is, for all $\bm{x}, \bm{y} \in \mathbb{R}^{d}$, $\left \Vert \nabla f(\bm{x}) - \nabla f(\bm{y}) \right \Vert \leq L \left \Vert \bm{x} - \bm{y} \right \Vert$. The Hessian of a twice continuously differentiable function $f$ is denoted by $\nabla^{2} f(\bm{x})$. When $f$ is twice continuously differentiable function $f$ is $L$-smooth if and only if, for all $\bm{x} \in \mathbb{R}^{d}$, $\left \Vert \nabla^{2} f(\bm{x}) \right \Vert_{\mathrm{op}} \leq L$. 

    Let $\mathbb{E}_{\bm{v}} \left[ X \right]$ be the expectation of $X$ with respect to a random variable $v$. $\mathbb{E}_{\bm{v}} \left[ X \right]$ may also be written as $\mathbb{E}_{\bm{v}} X$. Let $\mathbb{E}$ be the total expectation. 

    \subsubsection*{Stochastic Minimization Problem}
        This paper considers the following stochastic minimization problem \cite[Problem (12)]{khaled2020better}: Let $n$ be the number of training samples and $f_{i} \colon \mathbb{R}^{d} \to \mathbb{R}$ $\left( i \in \{1, \cdots, n\} \right)$ be the loss function corresponding to the $i$-th labeled training data. Then, we would like to
        \begin{align}
            \text{minimize } f(\bm{x}) \coloneqq \mathbb{E}_{\bm{v}} \left[ f_{\bm{v}}(\bm{x}) \right] \text{ over } \mathbb{R}^{d}, \label{stochastic_minimization}
        \end{align}
        where $\bm{v} \coloneqq \left( v_{1}, \cdots, v_{n} \right)^{\top}$ comprises $n$ identically distributed variables with $\mathbb{E}_{v_{i}} \left[ v_{i} \right] = 1$ and is independent of $\bm{x}$, and
        \begin{align} \label{eq:f_v}
            f_{\bm{v}}(\bm{x}) \coloneqq \frac{1}{n} \sum_{i=1}^{n} v_{i} f_{i} (\bm{x}).
        \end{align}
        Note that $f_{v}(\bm{x})$ is a loss function chosen randomly from $f_{1}(\bm{x}), \cdots, f_{n}(\bm{x})$.

        Below, we analyze the convergence of SGD under the ES assumption \eqref{eq:es_condition}. The analysis requires a set of standard assumptions on the objective function and the stochastic gradient. We will also make use of standard inequalities for $L$-smooth functions throughout.

        \begin{assumption}[$L$-smoothness] \label{as:L-smooth}
            A function $f \colon \mathbb{R}^{d} \to \mathbb{R}$ is $L$-smooth, which leads to the following descent lemma: for all $\bm{x}, \bm{y} \in \mathbb{R}^d$,
            \begin{equation} \label{eq:L-smooth-descent}
                f(\bm{y}) \leq f(\bm{x}) + \langle {\nabla f(\bm{x})}, {\bm{y} - \bm{x}} \rangle + \frac{L}{2} \left \Vert \bm{y} - \bm{x} \right \Vert^{2}.
            \end{equation}
        \end{assumption}

        \begin{assumption}[Lower boundedness] \label{as:lower}
            $f(\bm{x})$ has a lower bound $f^{\star}$ such that  
            \begin{align*}
                f^{\star} \coloneqq \inf_{\bm{x} \in \mathbb{R}^{d}} f(\bm{x}) > - \infty.
            \end{align*}
        \end{assumption}

        We briefly explain why Assumptions \ref{as:L-smooth} and \ref{as:lower} are natural in the context of minimizing $f$. Let us consider minimizing $f$ defined by \eqref{stochastic_minimization} by using an optimizer generating a sequence $( \bm{x}_{k})_{k=0}^{K}$, where $K$ is the number of steps. We assume that $f_{i}$ is twice continuously differentiable. Since $K$ is finite, the set $(\bm{x}_{k})_{k=0}^{K}$ is a finite, and hence compact, subset of $\mathbb{R}^{d}$. Moreover, the function $\Vert \nabla^{2} f_{i}(\cdot) \Vert_{\mathrm{op}}$ is continuous by the twice continuous differentiability of $f_i$. Therefore, it attains a finite maximum on $(\bm{x}_{k})_{k=0}^{K}$. Thus, there exists $L_{i}>0$ such that, for all $k$,
        %
        \begin{align*}
            \left \Vert \nabla^{2} f_{i}(\bm{x}_{k}) \right \Vert_{\mathrm{op}} \leq L_{i}.
        \end{align*}
        The $L$-smoothness condition holds along the generated sequence. Hence, $f$ is $L$-smooth along the sequence with
        \begin{align*}
            L = \frac{1}{n} \sum_{i=1}^{n} L_{i},
        \end{align*}
        which implies that Assumption \ref{as:L-smooth} holds along the generated sequence. For Assumption \ref{as:lower}, the condition $f^{\star} > -\infty$ ensures that the objective function has a finite lower bound and therefore that the minimization problem has a finite infimum. In contrast, if $f$ is convex with $f^{\star} = -\infty$, then $f$ is unbounded below and there is no minimizer. Thus, an objective function with $f^{\star}=-\infty$ does not have a finite optimal objective value to which a minimization algorithm can converge. Therefore, Assumption \ref{as:lower} is a natural condition when analyzing an optimizer intended to minimize $f$. Moreover, in the case where $f$ is convex with $f^{\star} = - \infty$, there are no stationary points of $f$, which implies that no optimizer ever finds stationary points of $f$. Hence, Assumption \ref{as:lower} is natural in a situation which optimizers minimize $f$.

        \begin{assumption}[Unbiasedness] \label{as:unbiased}
            The stochastic gradient $g_{\bm{v}}(\bm{x}) = \nabla f_{\bm{v}} (\bm{x})$ is unbiased estimator of $f$ that is, for all $\bm{x} \in \mathbb{R}^{d}$,
            \begin{align*}
                \mathbb{E}_{\bm{v}} \left[g_{\bm{v}}(\bm{x}) \right] = \nabla f(\bm{x}).
            \end{align*}
        \end{assumption}

        This assumption is satisfied under various canonical settings. For a rigorous justification, we refer to the stochastic reformulation framework detailed in \cite[Proposition 3]{khaled2020better}. In their framework, the stochastic gradient $g(\bm{x})$ using a sampling vector $\bm{v}$ designed to satisfy $\mathbb{E} \left[ v_{i} \right] = 1$ for all $i$, which directly ensures the unbiasedness condition $\mathbb{E} \left[ g(\bm{x}) \right] = \nabla f(\bm{x})$.

        \begin{assumption}[Expected Smoothness \cite{khaled2020better}] \label{as:es}
            There exist constants $A$, $B$, $C \ge 0$ such that, for all $\bm{x} \in \mathbb{R}^{d}$,
            \begin{equation} \label{eq:ES}
                \mathbb{E}_{\bm{v}} \left \Vert \nabla f_{\bm{v}}(\bm{x}) \right \Vert^{2} \leq 2A (f(\bm{x}) - f^{\star}) + B \left \Vert \nabla f(\bm{x}) \right \Vert^{2} + C,
            \end{equation}
            where $f^{\star}$ is defined as in Assumption \ref{as:lower} and $\nabla f_{\bm{v}} (\bm{x})$ is the gradient of $f_{\bm{v}}$ defined as in \eqref{eq:f_v}, that is, 
            \begin{align} \label{f_v_grad}
                \nabla f_{\bm{v}}(\bm{x}) = \frac{1}{n} \sum_{i=1}^{n} v_{i} \nabla f_{i}(\bm{x}).
            \end{align}
        \end{assumption}

        \begin{remark}[Variance from of ES]
            The variance interpretation of the expected smoothness (ES) condition follows directly from the following decomposition. Let $g_{\bm{v}}(\bm{x}) = \nabla f_{\bm{v}}(\bm{x})$ with $\mathbb{E}_{\bm{v}} \left[g_{\bm{v}}(\bm{x}) \right] = \nabla f(\bm{x})$. Then,
            \begin{align*}
                &\mathbb{E}_{\bm{v}}  \left \Vert g_{\bm{v}}(\bm{x}) - \nabla f(\bm{x}) \right \Vert^{2}  \\
                &\quad = \mathbb{E}_{\bm{v}} \left[ \left \Vert g_{\bm{v}}(\bm{x}) \right \Vert^{2} - 2 \langle g_{\bm{v}}(\bm{x}), \nabla f(\bm{x}) \rangle \right] + \left \Vert \nabla f(\bm{x}) \right \Vert^{2} \\
                &\quad = \mathbb{E}_{\bm{v}}  \left \Vert g_{\bm{v}}(\bm{x}) \right \Vert^{2}  - 2 \langle \mathbb{E}_{\bm{v}} \left[g_{\bm{v}}(\bm{x}) \right], \nabla f(\bm{x}) \rangle + \left \Vert \nabla f(\bm{x}) \right \Vert^{2} \\
                &\quad = \mathbb{E}_{\bm{v}}  \left \Vert g_{\bm{v}}(\bm{x}) \right \Vert^{2} - 2 \left \Vert \nabla f(\bm{x}) \right \Vert^{2} + \left \Vert \nabla f(\bm{x}) \right \Vert^{2} \\
                &\quad = \mathbb{E}_{\bm{v}}  \left \Vert g_{\bm{v}}(\bm{x}) \right \Vert^{2} - \left \Vert \nabla f(\bm{x}) \right \Vert^{2}.
            \end{align*}
            Consequently, Applying the ES condition \eqref{eq:ES} yields
            \[
                \mathbb{E}_{\bm{v}}  \left \Vert g_{\bm{v}}(\bm{x}) - \nabla f(\bm{x}) \right \Vert^{2}  \leq 2A(f(\bm{x}) - f^{\star}) + (B - 1) \left \Vert \nabla f(\bm{x}) \right \Vert^{2} + C.
            \]
            This shows that the variance of the stochastic gradient is controlled by the suboptimality of the function, the norm of the gradient, and a constant term.
        \end{remark}

        \subsubsection*{On Increasing Batch Size and Learning Rate}
            Recent work by \cite{umeda2024accelerates} demonstrates that increasing both the batch size $\tau$ and the step size (called the learning rate) $\eta$ can accelerate convergence. Before presenting their finite-time bound, we recall that \cite{umeda2024accelerates} analyze SGD under the classical bounded variance assumption
            \[
                \mathbb{E} \left \Vert g(\bm{x}) - \nabla f(\bm{x}) \right \Vert^{2} \leq \sigma^{2},
            \]
            together with standard $L$-smoothness of the objective function. Under these assumptions, they prove a non-convex convergence guarantee for minibatch SGD, which we restate below. Here, we define the minibatch stochastic gradient by $g_{\bm{v}}(\bm{x}_{k}) \coloneqq \frac{1}{\tau_{k}} \sum_{i=1}^{\tau_{k}} \nabla f_{v_{k,i}}(\bm{x}_{k})$, where $\bm{v}_{k} \coloneqq \left( v_{k,1}, \cdots, v_{k,\tau_{k}} \right)^{\top}$ comprises $\tau_{k}$ independent and identically distributed variables and is independent of $\bm{x}_{k}$. The following result in \cite[Lemma 2.1]{umeda2024accelerates} provides a finite-time bound on the expected squared gradient norm generated by minibatch SGD defined by $\bm{x}_{k+1} = \bm{x}_{k} - \eta_{k} g_{\bm{v}_{k}}(\bm{x}_{k})$:
            \begin{equation}
                \min_{0 \leq k \leq K} \mathbb{E} \left \Vert \nabla f(\bm{x}_{k}) \right \Vert^{2} 
                \leq \frac{2(f(\bm{x}_{0}) - f^{\star})}{\left( 2 - L \overline{\eta} \right) \sum_{k=0}^{K-1} \eta_{k}} + \frac{L \sigma^{2} \sum_{k=0}^{K-1} \eta_{k}^{2} \tau_{k}^{-1}}{(2 - L \overline{\eta}) \sum_{k=0}^{K-1} \eta_{k}},\label{eq:umeda_lemma}
            \end{equation}
            where $\bar{\eta} \coloneqq \sup_k \eta_{k}$ denotes the maximum step size. This inequality \eqref{eq:umeda_lemma} shows that the variance term is scaled by the factor $\tau_{k}^{-1}$. Hence, enlarging the batch size $\tau_{k}$ reduces the variance contribution, allowing the step size $\eta$ to be increased proportionally while keeping the bound finite. This theoretical mechanism explains why jointly increasing $\tau$ and $\eta$ leads to accelerated convergence in practice, as confirmed by the analysis and experiments in \cite{umeda2024accelerates}.

            Comparing \eqref{eq:ES} with \eqref{eq:umeda_lemma}, we see that the variance term in (9) of \cite{umeda2024accelerates} corresponds to the constant $C$ in ES, scaled by $1/\tau_{k}$ due to minibatch averaging. That is, for the minibatch gradient defined via the sampling vector $\bm{v}_{k} = \left( v_{k,1}, \cdots, v_{k,\tau_{k}} \right)^{\top}$,
            \[
                g_{\bm{v}_{k}}(\bm{x}_{k}) \coloneqq \frac{1}{\tau_{k}} \sum_{i=1}^{\tau_{k}} g_{v_{k,i}} (\bm{x}_{k}), \quad \text{where} \quad g_{v_{k,i}} (\bm{x}) \coloneqq \nabla f_{v_{k,i}} (\bm{x})
            \]
            with $v_{k,i}$ denoting the index of the $i$-th selected sample in the minibatch at iteration $k$, we obtain
            \begin{equation*}
                \mathbb{E}_{\bm{v}_{k}} \left \Vert g_{\bm{v}_{k}} (\bm{x}_{k}) \right \Vert^{2} \leq \frac{2A}{\tau_{k}} (f(\bm{x}_{k}) - f^{\star}) + \left( 1 + \frac{B - 1}{\tau_{k}} \right) \left \Vert \nabla f(\bm{x}_{k}) \right \Vert^{2} + \frac{C}{\tau_{k}}. \label{eq:minibatch}
            \end{equation*}
            The detailed derivation is provided in the proof Lemma~\ref{lem:onestep} below. The variance reduction induced by mini-batch averaging can be made explicit as follows. Let
            \[
                \nabla f_{B_{k}}(\bm{x}_{k}) = \frac{1}{b_{k}} \sum_{i=1}^{b_{k}} \nabla f_{\xi_{k,i}}(\bm{x}_{k})
            \]
            be a mini-batch gradient. We analyze its variance conditional on the past iterates $\hat{\bm{\xi}}_{k-1}$.
        
    \subsection{Bounding the Mini-batch Variance}
    We have 
        \begin{align*}
            &\mathbb{E}_{\bm{\xi}_{k}}  \left \Vert \nabla f_{B_{k}}(\bm{x}_{k}) - \nabla f(\bm{x}_{k}) \right \Vert^{2}  \\
            &\quad = \mathbb{E}_{\xi_{k,i}} \left \Vert \frac{1}{b_{k}} \sum_{i=1}^{b_{k}} \nabla f_{\xi_{k,i}}(\bm{x}_{k}) - \nabla f(\bm{x}_k)\right \Vert^{2} \\
            &\quad = \mathbb{E}_{\xi_{k,i}}  \left \Vert \frac{1}{b_{k}} \sum_{i=1}^{b_{k}} \left( \nabla f_{\xi_{k,i}}(\bm{x}_{k}) - \nabla f(\bm{x}_k) \right) \right \Vert^{2} \\
            &\quad = \frac{1}{b_{k}^{2}}\ \mathbb{E}_{\xi_{k,i}}  \left \Vert \sum_{i=1}^{b_{k}} \left( \nabla f_{\xi_{k,i}}(\bm{x}_{k}) - \nabla f(\bm{x}_k) \right) \right \Vert^{2} \\
            &\quad = \frac{1}{b_{k}^{2}}\ \sum_{i=1}^{b_{k}}\  \mathbb{E}_{\xi_{k,i}} \left \Vert \nabla f_{\xi_{k,i}}(\bm{x}_{k}) - \nabla f(\bm{x}_k) \right \Vert^{2} \\
            &\quad = \frac{1}{b_{k}}\  \mathbb{E}_{\xi_{k,i}} \left \Vert \nabla f_{\xi_{k,i}}(\bm{x}_{k}) - \nabla f(\bm{x}_k) \right \Vert^{2} \\
            &\quad \leq \frac{1}{b_{k}}\ \left\{ 2A(f(\bm{x}_{k}) - f^{\star}) + (B - 1) \left \Vert \nabla f(\bm{x}_{k}) \right \Vert^{2} + C \right\}.
        \end{align*}

    \subsection{Bounding the Second Moment}
        We derive an upper bound on the conditional second moment of the stochastic gradient under the expected smoothness assumption. This bound expresses the squared norm of the mini-batch gradient in terms of the suboptimality and the norm of the full gradient, which is essential for analyzing the convergence behavior of stochastic gradient methods. We decompose the mini-batch gradient into the deviation from the full gradient plus the full gradient itself. This part of the proof further analyzes the expected squared norm of the mini-batch gradient. It uses variance decomposition and the result derived previously to bound $\mathbb{E}_{\bm{\xi}_{k}} [  \Vert \nabla f_{B_{k}}(\bm{x}_{k})  \Vert^{2} \mid \hat{\bm{\xi}}_{k-1} ]$.
        \begin{align*}
            &\mathbb{E}_{\bm{\xi}_{k}} \left[ \left \Vert \nabla f_{B_{k}}(\bm{x}_{k}) \right \Vert^{2} \mid \hat{\bm\xi}_{k-1} \right] \\
            &\quad= \mathbb{E}_{\bm\xi_{k}} \left[ \left \Vert (\nabla f_{B_{k}}(\bm{x}_{k}) - \nabla f(\bm{x}_{k})) + \nabla f(\bm{x}_{k}) \right \Vert^{2} \mid  \hat{\bm\xi}_{k-1} \right] \\
            &\quad= \mathbb{E}_{\bm\xi_{k}} \left[ \left \Vert \nabla f_{B_{k}}(\bm{x}_{k}) - \nabla f(\bm{x}_{k}) \right \Vert^{2} \mid \hat{\bm\xi}_{k-1} \right] \\
            &\qquad + 2 \left\langle \mathbb{E}_{\bm\xi_{k}} \left[ \nabla f_{B_{k}}(\bm{x}_{k}) - \nabla f(\bm{x}_{k}) \mid \hat{\bm\xi}_{k-1} \right], \nabla f(\bm{x}_{k}) \right\rangle + \mathbb{E}_{\bm\xi_{k}} \left[ \left \Vert \nabla f(\bm{x}_{k}) \right \Vert^{2} \mid \hat{\bm\xi}_{k-1} \right] \\
            &\quad= \mathbb{E}_{\bm\xi_{k}} \left[ \left \Vert \nabla f_{B_{k}}(\bm{x}_{k}) - \nabla f(\bm{x}_{k}) \right \Vert^{2} \mid \hat{\bm\xi}_{k-1} \right] \\
            &\qquad + 2 \langle \nabla f(\bm{x}_{k}) - \nabla f(\bm{x}_{k}), \nabla f(\bm{x}_{k}) \rangle + \mathbb{E}_{\bm\xi_{k}} \left[ \left \Vert \nabla f(\bm{x}_{k}) \right \Vert^{2} \mid \hat{\bm\xi}_{k-1} \right] \\
            &\quad= \mathbb{E}_{\bm\xi_{k}} \left[ \left \Vert \nabla f_{B_{k}}(\bm{x}_{k}) - \nabla f(\bm{x}_{k}) \right \Vert^{2} \mid \hat{\bm\xi}_{k-1} \right] + \mathbb{E}_{\bm\xi_{k}} \left[ \left \Vert \nabla f(\bm{x}_{k}) \right \Vert^{2} \mid \hat{\bm\xi}_{k-1} \right],
        \end{align*}
        where the middle term vanishes due to unbiasedness of the mini-batch gradient. Next, we apply the ES assumption to the variance term:
        \begin{align*}
            &\mathbb{E}_{\bm\xi_{k}} \left[ \left \Vert \nabla f_{B_{k}}(\bm{x}_{k}) \right \Vert^{2} \mid \hat{\bm\xi}_{k-1} \right] \\
            &\quad = \mathbb{E}_{\bm\xi_{k}} \left[ \left \Vert \nabla f_{B_{k}}(\bm{x}_{k}) - \nabla f(\bm{x}_{k}) \right \Vert^{2} \mid \hat{\bm\xi}_{k-1} \right] + \mathbb{E}_{\bm\xi_{k}} \left[ \left \Vert \nabla f(\bm{x}_{k}) \right \Vert^{2} \mid \hat{\bm\xi}_{k-1} \right] \\
            &\quad \leq \frac{1}{b_{k}} \left\{ 2A(f(\bm{x}_{k}) - f^{\star}) + (B - 1) \left \Vert \nabla f(\bm{x}_k) \right \Vert^{2} + C \right\} + \left \Vert \nabla f(\bm{x}_k) \right \Vert^{2} \\
            &\quad = \frac{1}{b_{k}} \left\{ 2A(f(\bm{x}_{k}) - f^{\star}) + (B - 1 + b_{k}) \left \Vert \nabla f(\bm{x}_k) \right \Vert^{2} + C \right\}.
        \end{align*}
        This completes the proof of the mini-batch variance bound. This bound will be used below in the derivation of the one-step descent inequality for SGD. \medskip

        \subsubsection*{Insight:}
            This establishes a direct correspondence between Assumption~\ref{as:es} and the bound in (9) of \cite{umeda2024accelerates}: increasing the batch size $\tau_{k}$ reduces the effective variance term $C/{\tau_{k}}$, allowing a proportionally larger learning rate $\eta_{k}$ while maintaining stability. Thus, ES provides a natural theoretical explanation for the accelerated convergence via joint scaling of the batch size and learning rate. \medskip

        \subsubsection*{Interpretation:}
            (A) The coefficient $A$ ties the variance to the \emph{suboptimality} $f(\bm{x}) - f^{\star}$.
            (B) The coefficient $B$ captures the correlation with the squared gradient norm.
            (C) The constant $C$ represents an irreducible ``baseline variance`` from data heterogeneity. \medskip

            This decomposition shows explicitly how the second moment of the stochastic gradient is related to its variance and the squared norm of the full gradient, providing the basis for interpreting the roles of the ES constants.
            
            As a direct consequence of Assumptions~\ref{as:L-smooth}--\ref{as:unbiased}, we establish the following useful proposition (Full derivations and justifications are provided in Appendix~\ref{app:proof_prop_es_valid}).

            \begin{proposition}[Validity of ES for common sampling] \label{prop:es-valid}
                For $n = 1$, all sampling schemes reduce to deterministic sampling, and the ES condition holds with $\left( A, B, C \right) = \left( 0, 1, 0 \right)$. For $n \geq 2$, the following sampling schemes admit the stated constants:

                Consider Problem \eqref{stochastic_minimization} and suppose that Assumption \ref{as:unbiased} holds and that each component $f_{i}$ is $L_{i}$-smooth and bounded below, which implies Assumptions \ref{as:L-smooth} and \ref{as:lower} hold. Let $\Delta^{\inf} = \frac{1}{n} \sum_{i=1}^{n} (f^{\star} - f_{i}^{\inf})$, where $f^{\inf}$ (resp. $f_{i}^{\inf}$) is the infimum of $f$ (resp. $f_{i}$) over $\mathbb{R}^{d}$. For the following sampling schemes, there exist finite constants $\left( A, B, C \right)$ such that Assumption~\ref{as:es} holds:

                \begin{enumerate}
                    \item \textbf{Sampling with replacement} (minibatch of size $\tau$ with probabilities $q_{i}$):
                        \begin{align*}
                            \begin{aligned}
                                A = \frac{1}{\tau n}\max_{1 \leq i \leq n} \frac{L_{i}}{q_{i}}, \text{ } B = 1 - \frac{1}{\tau},\text{ } C = 2 A \Delta^{\inf}.
                            \end{aligned}
                        \end{align*}
    
                    \item \textbf{Independent sampling without replacement} (each index included with probability $p_{i}$):
                        \begin{align*}
                            \begin{aligned}
                                A = \max_{1 \leq i \leq n} \frac{(1-p_{i})L_{i}}{p_{i} n},\text{ } B = 1,\text{ } C = 2 A \Delta^{\inf}.
                            \end{aligned}
                        \end{align*}
    
                    \item \textbf{\texorpdfstring{$\tau$}{tau}-Nice sampling without replacement} (uniform without replacement of size $\tau$, for $n\geq 2$):
                        \begin{align*}
                            \begin{aligned}
                                A = \frac{n-\tau}{\tau (n-1)}\max_{1 \leq i \leq n} L_{i},\text{ } B = \frac{n(\tau-1)}{\tau(n-1)},\text{ } C = 2 A \Delta^{\inf}.
                            \end{aligned}
                        \end{align*}
                \end{enumerate}
            \end{proposition}

    \paragraph{Example: Least-squares regression.}
        Next, we illustrate the explicit ES constants in Proposition~\ref{prop:es-valid} for the classical least-squares regression problem. Consider the finite-sum objective,
        \begin{equation}
            f(\bm{x}) = \frac{1}{2n} \sum_{i=1}^{n} \left( \bm{a}_{i}^{\top}\bm{x}-y_{i} \right)^{2},
            \label{eq:ls-objective}
        \end{equation}
        where $\bm{x}\in\mathbb{R}^{d}$, $\bm{a}_{i}\in\mathbb{R}^{d}$, and $y_{i}\in\mathbb{R}$. In the notation of Section~2, define the component functions by
        \begin{equation}
            f_{i}(\bm{x}) \coloneqq \frac{1}{2} \left( \bm{a}_{i}^{\top}\bm{x}-y_{i} \right)^{2},
            \qquad
            f(\bm{x}) = \frac{1}{n} \sum_{i=1}^{n}f_{i}(\bm{x}).
            \label{eq:ls-component}
        \end{equation}
        The gradient and Hessian of each component function are given by
        \begin{equation}
            \nabla f_{i}(\bm{x}) = \left( \bm{a}_{i}^{\top}\bm{x}-y_{i} \right)\bm{a}_{i},
            \qquad
            \nabla^{2}f_{i}(\bm{x}) = \bm{a}_{i}\bm{a}_{i}^{\top}.
            \label{eq:ls-gradient-hessian}
        \end{equation}
        Since $\bm{a}_{i}\bm{a}_{i}^{\top}$ is positive semidefinite and $\| \bm{a}_{i}\bm{a}_{i}^{\top} \| = \|\bm{a}_{i}\|^{2}$, where $\|\cdot \|$ denotes the Euclidean norm for vectors and the induced spectral norm for matrices, each $f_i$ is $L_i$-smooth with
        \begin{equation}
            L_i = \|\bm{a}_{i} \|^{2}.
            \label{eq:ls-Li}
        \end{equation}
        Moreover, the squared norm of the component gradient satisfies
        \begin{align}
            \left\|\nabla f_i(\bm{x}) \right\|^{2}
            &= \left\| \left( \bm{a}_{i}^{\top}\bm{x}-y_{i} \right)\bm{a}_{i} \right\|^{2} \notag \\
            &= \|\bm{a}_{i} \|^{2} \left( \bm{a}_{i}^{\top}\bm{x}-y_{i} \right)^{2} \notag \\
            &= 2 \|\bm{a}_{i} \|^{2}f_i(\bm{x}) = 2L_i f_i(\bm{x}).
            \label{eq:ls-gradient-bound-exact}
        \end{align}
        Thus, if
        \begin{equation}
            R \coloneqq \max_{1\leq i\leq n} \|\bm{a}_{i} \|,
            \label{eq:ls-R}
        \end{equation}
        then $L_i\leq R^2$ for every $i$, and consequently
        \begin{equation}
            \|\nabla f_i(\bm{x}) \|^{2} \leq 2R^{2}f_i(\bm{x}).
            \label{eq:ls-gradient-bound}
        \end{equation}
        We now consider the sampling-with-replacement scheme in Proposition~\ref{prop:es-valid} with a mini-batch size $\tau$ and uniform sampling probabilities,
        \begin{equation}
            q_i = \frac{1}{n}, \qquad i=1,\ldots,n.
            \label{eq:ls-uniform-sampling}
        \end{equation}
        By Proposition~\ref{prop:es-valid}, the corresponding ES constants are
        \begin{equation}
            A = \frac{1}{\tau n} \max_{1\leq i\leq n} \frac{L_i}{q_i},
            \qquad
            B = 1-\frac{1}{\tau},
            \qquad
            C = 2A\Delta_{\inf},
            \label{eq:ls-es-constants-general}
        \end{equation}
        where
        \begin{equation}
            \Delta_{\inf} \coloneqq \frac{1}{n} \sum_{i=1}^{n} \left( f^{\ast}-f_i^{\inf} \right)
        \end{equation}
        denotes the heterogeneity quantity introduced in Proposition~\ref{prop:es-valid}. Substituting $q_i=1/n$ and $L_i= \|\bm{a}_{i} \|^2$ into \eqref{eq:ls-es-constants-general} gives
        \begin{align}
            A &= \frac{1}{\tau n} \max_{1\leq i\leq n} \frac{L_i}{1/n} = \frac{1}{\tau} \max_{1\leq i\leq n}L_i = \frac{R^2}{\tau}, \label{eq:ls-A} \\
            B &= 1-\frac{1}{\tau}, \label{eq:ls-B} \\
            C &= \frac{2R^2}{\tau}\Delta_{\inf}. \label{eq:ls-C}
        \end{align}
        Therefore, for uniform sampling with replacement, the ES condition takes the explicit form,
        \begin{equation}
            \mathbb{E}_{\bm{v}}  \|\nabla f_{\bm{v}}(\bm{x}) \|^{2}  \leq \frac{2R^{2}}{\tau} \left( f(\bm{x})-f^{\ast} \right) + \left( 1-\frac{1}{\tau} \right) \|\nabla f(\bm{x}) \|^{2} + \frac{2R^{2}}{\tau}\Delta_{\inf}.
            \label{eq:ls-explicit-es}
        \end{equation}
        Here, $\nabla f_{\bm{v}}(\bm{x})$ denotes the stochastic gradient associated with the sampling vector $\bm{v}$ defined in Section~\ref{sec:ES}. Equation~\eqref{eq:ls-explicit-es} makes explicit how the geometry of the data and the mini-batch size enter the ES condition. In particular, the coefficient $A$ of the suboptimality term and the heterogeneity-dependent constant $C$ both scale as $1/\tau$:
        \begin{equation}
            A = \frac{R^2}{\tau}, \qquad C = \frac{2R^2}{\tau}\Delta_{\inf}.
        \end{equation}
        Thus, increasing the mini-batch size reduces these two contributions to the ES bound. At the same time, the coefficient of the squared full gradient is $B = 1 - 1/\tau$, so the dependence of the complete ES bound on $\tau$ is captured by all three constants $\left( A, B, C \right)$ rather than by a uniform $1/\tau$ scaling of the entire bound.
        
        The explicit constants in \eqref{eq:ls-A}--\eqref{eq:ls-C} can be substituted directly into the general convergence results in Section~3. In particular, the dependence of the resulting convergence guarantees on the mini-batch size $\tau$, the maximum component smoothness $R^2 = \max_{1\leq i\leq n}\|\bm{a}_i\|^2$, and the heterogeneity quantity $\Delta_{\inf}$ is thereby made explicit. This provides a concrete instance in which the abstract ES constants in Proposition~\ref{prop:es-valid} can be computed directly from the data matrix and the mini-batch sampling scheme.
    
        These explicit formulas for $\left( A, B, C \right)$ provide a quantitative characterization of the stochastic gradient structure under common sampling strategies. In conjunction with Theorem~\ref{thm:1}, they make explicit how the sampling strategy, minibatch size, and component-wise smoothness affect the convergence bound and the admissible step-size range. In particular, the explicit ES constants in Proposition~\ref{prop:es-valid} can be directly substituted into the general convergence bound of Theorem~\ref{thm:1}, yielding sampling-specific convergence guarantees. These explicit bounds are derived as a direct consequence of Corollary~\ref{cor:general_rate} and are summarized in Corollary~\ref{cor:sampling-specific-bounds} below. The variance decomposition with sampling weights leading to these sampling-dependent constants is developed in Section~\ref{sec:moment-verification}. Thus, Proposition~\ref{prop:es-valid} provides the explicit ES constants for each sampling scheme, while Corollary~\ref{cor:sampling-specific-bounds} makes their implications for the convergence guarantees explicit.

    \subsection{A Self-Contained Derivation via Jensen's Inequality} \label{sec:jensen-warmup}
    The bounds established in Section~\ref{sec:ES} assumed that the samples composing a mini-batch are independent and identically distributed. We now present a complementary, self-contained derivation that relaxes this assumption: each sample $\xi_{i}$ in the mini-batch is allowed to carry its own sampling weight $v_{\xi_{i}}$ and its own ES constants $\left( A_{i}, B_{i}, C_{i} \right)$, so that the resulting bound applies to \emph{heterogeneous}, not necessarily identically distributed, mini-batches. The derivation proceeds via Jensen's inequality rather than the conditional variance decomposition used above, and yields an explicit, fully additive convergence guarantee that we show, at the end of this subsection, recovers Theorem~\ref{thm:1} in the special case $B = 0$.
    
    Let $g_{\xi}(\bm{x}) := v_{\xi}(\bm{x}) \nabla f_\xi(\bm{x})$, where $v_{\xi}$ represents a random sampling weight and $f_\xi$ is $L_\xi$-smooth. We assume independence between $v_{\xi}$ and $\nabla f_{\xi}$, which allows separation of expectations. Then, by the standard inequality for smooth functions, we have
    \begin{align*}
        \mathbb{E}_{\xi} \left \Vert g_{\xi}(\bm{x}) \right \Vert^{2} 
        &= \mathbb{E}_{\xi}  \left \Vert v_{\xi} \nabla f_{\xi}(\bm{x}) \right \Vert^{2} \\
        &= \mathbb{E}_{\xi} \left[v_{\xi}^{2} \left \Vert \nabla f_{\xi}(\bm{x}) \right \Vert^{2} \right] \\
        &= \mathbb{E}_{\xi} \left[v_{\xi}^{2} \right] \cdot \mathbb{E}_{\xi} \left[ \left \Vert \nabla f_{\xi}(\bm{x}) \right \Vert^{2} \right] \\
        &\leq \mathbb{E}_{\xi} \left[v_{\xi}^{2} \right] \cdot \mathbb{E}_{\xi} \left[2L(f_{\xi}(\bm{x}) - f_{\xi}^{\inf}) \right] \\
        &= 2L\mathbb{E}_{\xi} \left[v_{\xi}^{2} \right] \cdot \mathbb{E}_{\xi} \left[f_{\xi}(\bm{x}) - f_{\xi}^{\inf} \right] \\
        &= 2L\mathbb{E}_{\xi} \left[v_{\xi}^{2} \right] \cdot \mathbb{E}_{\xi} \left[f_{\xi}(\bm{x}) - f^{\inf} + f^{\inf} - f_{\xi}^{\inf} \right] \\
        &= 2L\mathbb{E}_{\xi} \left[v_{\xi}^{2} \right] \cdot \left( (\mathbb{E}_{\xi} \left[f_{\xi}(\bm{x}) \right] - f^{\inf}) + \mathbb{E}_{\xi} \left[f^{\inf} - f_{\xi}^{\inf} \right] \right) \\
        &= 2L\mathbb{E}_{\xi} \left[v_{\xi}^{2} \right] \cdot (f(\bm{x}) - f^{\inf}) + 2L\mathbb{E}_{\xi} \left[v_{\xi}^{2} \right] \cdot \mathbb{E}_{\xi} \left[f^{\inf} - f_{\xi}^{\inf} \right] \\
        &= 2A(f(\bm{x}) - f^{\inf}) + C,
    \end{align*}
    where
    \[
        A = L \mathbb{E}_{\xi} \left[ v_{\xi}^{2} \right], \quad B = 0, \quad C = 2 A \cdot \mathbb{E}_{\xi} \left[ f^{\inf} - f_{\xi}^{\inf} \right].
    \]
    This calculation demonstrates the basic form of the ES property for a single-sample stochastic gradient, explicitly relating the second moment of the stochastic gradient to the suboptimality of the objective. Table~\ref{tab:es_constants_jensen} enumerates the resulting constants $A$ for the three sampling schemes of Proposition~\ref{prop:es-valid}, expressed in terms of the per-sample second moment $\mathbb{E}[v_\xi^2]$; the two presentations are equivalent, but Table~\ref{tab:es_constants_jensen} makes the multiplicative role of $\mathbb{E}[v_\xi^2]$ explicit.
    
    \begin{table}[htbp]
        \centering
        \setlength{\tabcolsep}{3pt}
        \begin{tabularx}{\textwidth}{@{} X l l @{}}
            \toprule
            \textbf{Subsampling Method} & \textbf{$A$} & \textbf{$\mathbb{E} \left[ v_{\xi}^{2} \right]$} \\
             \midrule
             (i) Sampling with replacement &
             $\begin{aligned}
                 A &= \max_{i} \dfrac{L_{i}}{\tau n q_{i}} \\
                 &= \max_{i} L_{i} \mathbb{E} \left[ v_{i}^{2} \right]
             \end{aligned}$ &
             $\mathbb{E} \left[ v_{i}^{2} \right] = \dfrac{1}{\tau n q_{i}}$ \\
             \addlinespace[8pt]
             (ii) Independent sampling without replacement &
             $\begin{aligned}
                 A &= \max_{i} \dfrac{(1 - p_{i}) L_{i}}{p_{i} n} \\
                 &= \max_{i} L_{i} \mathbb{E} \left[ v_{i}^{2} \right]
             \end{aligned}$ &
             $\mathbb{E} \left[v_{i}^{2} \right] = \dfrac{1 - p_{i}}{p_{i} n}$ \\
             \addlinespace[8pt]
             (iii) \texorpdfstring{$\tau$}{tau}-nice sampling without replacement &
             $\begin{aligned}
                 A &= \dfrac{n - \tau}{\tau(n - 1)} \max_{i} L_{i} \\
                 &= \max_{i} L_{i} \mathbb{E} \left[ v_{i}^{2} \right]
             \end{aligned}$ &
             $\mathbb{E} \left[ v_{i}^{2} \right] = \dfrac{n - \tau}{\tau(n - 1)}$ \\
             \bottomrule
        \end{tabularx}
        \caption{Expected Smoothness constants under different subsampling strategies, expressed via the per-sample second moment $\mathbb{E}[v_\xi^2]$.}
        \label{tab:es_constants_jensen}
    \end{table}
    
    Consider the finite-sum objective,
    \[
        f(\bm{x}) = \frac{1}{n} \sum_{i=1}^{n} f_{i}(\bm{x}),
    \]
    and define the mini-batch stochastic gradient estimator as
    \[
        \nabla f_{\bm{v}}(\bm{x}) = \frac{1}{m} \sum_{i=1}^{m} v_{\xi_{i}} \nabla f_{\xi_{i}}(\bm{x}), \quad \bm{v} = (v_{\xi_{1}}, \dots, v_{\xi_{m}})^\top, \quad \bm{\xi} = (\xi_{1}, \dots, \xi_{m})^\top,
    \]
    where, unlike the mini-batch estimator of Section~\ref{sec:ES}, the samples $\xi_{1}, \dots, \xi_{m}$ need not be identically distributed. By applying Jensen's inequality and the single-sample ES property derived above to each sample individually, we obtain
    \begin{align*}
        \mathbb{E}_{\bm{\xi}} \left \Vert \nabla f_{\bm{v}}(\bm{x}) \right \Vert^{2} &= \mathbb{E}_{\bm{\xi}} \left \Vert \frac{1}{m} \sum_{i=1}^{m} v_{\xi_{i}} \nabla f_{\xi_{i}}(\bm{x}) \right \Vert^{2} \\
        &\leq \frac{1}{m} \sum_{i=1}^{m} \mathbb{E}_{\xi_{i}} \left \Vert v_{\xi_{i}} \nabla f_{\xi_{i}}(\bm{x}) \right \Vert^{2} \\
        &\leq \frac{1}{m} \sum_{i=1}^{m} \left( 2A_{i}(f(\bm{x}) - f^{\inf}) + C_{i} \right) \\
        &= 2 \left( \frac{1}{m} \sum_{i=1}^{m} A_{i} \right) (f(\bm{x}) - f^{\inf}) + \frac{1}{m} \sum_{i=1}^{m} C_{i} \\
        &= \frac{2}{m} A_{m}(f(\bm{x}) - f^{\inf}) + \frac{1}{m} C_{m} \\
        &= \frac{1}{m} \left(2A_{m}(f(\bm{x}) - f^{\inf}) + C_{m} \right),
    \end{align*}
    where $A_{i} = L \mathbb{E}_{\xi_{i}} [v_{\xi_{i}}^{2} ]$, $C_{i} = A_{i} \mathbb{E}_{\xi_{i}} [f^{\inf} - f_{\xi_{i}}^{\inf} ]$, $A_{m} = \sum_{i=1}^{m} A_{i}$, $C_{m} = \sum_{i=1}^{m} C_{i}$. Note that, unlike the constants $\left( A, B, C \right)$ of Assumption~4, which are shared across all samples, $A_m$ and $C_m$ here are \emph{additive across heterogeneous samples}: this derivation therefore extends the mini-batch ES property of Section~\ref{sec:ES} to mini-batches drawn from non-identical distributions, showing how the variance bound scales with the batch size even in this more general setting.
    
    Applying the $L$-smoothness property and the SGD update $\bm{x}_{k+1} = \bm{x}_{k} - \eta_{k} \nabla f_{\bm{v}}(\bm{x}_{k})$, the one-step expected descent under this heterogeneous mini-batch model is
    \begin{align*}
        f(\bm{x}_{k+1}) &\leq f(\bm{x}_{k}) - \eta_{k} \langle \nabla f(\bm{x}_{k}), \nabla f_{\bm{v}}(\bm{x}_{k}) \rangle + \frac{L \eta_{k}^{2}}{2} \left \Vert \nabla f_{\bm{v}}(\bm{x}_{k}) \right \Vert^{2}, \\
        \mathbb{E}_{\bm{v}} \left[ f(\bm{x}_{k+1}) \right] &\leq f(\bm{x}_{k}) - \eta_{k} \left \Vert \nabla f(\bm{x}_{k}) \right \Vert^{2} + \frac{L \eta_{k}^{2}}{2m} \left(2A_{m}(f(\bm{x}_{k}) - f^{\inf}) + C_{m} \right) \\
        &= f(\bm{x}_{k}) - \eta_{k} \left \Vert \nabla f(\bm{x}_{k}) \right \Vert^{2} + \frac{L \eta_{k}^{2}}{m} A_{m}(f(\bm{x}_{k}) - f^{\inf}) + \frac{L \eta_{k}^{2}}{2m} C_{m}.
    \end{align*}
    Finally, summing over $K$ iterations, we obtain a telescopic sum and isolate the minimum expected squared gradient norm:
    \begin{align*}
        \sum_{k=0}^{K-1} \eta_{k} \mathbb{E} \left \Vert \nabla f(\bm{x}_{k}) \right \Vert^{2}  &\leq f(\bm{x}_{0}) - f^{\inf} + \sum_{k=0}^{K-1} \eta_{k}^{2} \left( \frac{L}{m} A_{m} \left(f(\bm{x}_{0}) - f^{\inf} \right) + \frac{L}{2m} C_{m} \right), \\
        \min_{0 \leq k \leq K - 1} \mathbb{E} \left \Vert \nabla f(\bm{x}_{k}) \right \Vert^{2} &\leq \frac{f(\bm{x}_{0}) - f^{\inf}}{\sum_{k=0}^{K-1} \eta_{k}} + \frac{\sum_{k=0}^{K-1} \eta_{k}^{2} \left( \frac{L}{m} A_{m}(f(\bm{x}_{0}) - f^{\inf}) + \frac{L}{2m} C_{m} \right)}{\sum_{k=0}^{K-1} \eta_{k}}.
    \end{align*}
    This inequality establishes a non-asymptotic convergence rate of SGD under the basic ES form ($B = 0$), explicitly showing how the step sizes, batch size, and the additive sampling-variance constants $A_{m}, C_{m}$ influence convergence, even when the mini-batch consists of heterogeneous samples. Setting $A_{i} \equiv A/m$, $C_{i} \equiv C/m$ for all $i$ (i.e., recovering the i.i.d.\ case of Section~\ref{sec:ES} with $B = 0$ and $\tau_{k} \equiv m$) reduces $A_{m}$ to $A$ and $C_{m}$ to $C$, and the bound above coincides with the $B = 0$ specialization of Theorem~\ref{thm:1}. This confirms that the present, Jensen-based derivation is consistent with, and complements, the general convergence guarantee established in Theorem~\ref{thm:1}.
    
    \subsection{Direct Verification via Sampling-Weight Moments} \label{sec:moment-verification}
    The constants $\left(A, B, C \right)$ for the three sampling schemes were already derived in the proof of Proposition~\ref{prop:es-valid} (Appendix~\ref{app:proof_prop_es_valid}) by applying the descent-lemma bound \eqref{eq:appC_descent_component} to a general point $\bm{x} \in \mathbb{R}^d$. We now give an independent, computational cross-check of these same constants, evaluated directly at the iterate $\bm{x}_{k}$ generated by SGD~\eqref{eq:moment-decomp-identity}. Rather than proceeding through the abstract smoothness bound~\eqref{eq:appC_descent_component}, the derivation below starts from the second moments and cross-moments of the sampling weights themselves, obtained directly from the distribution of each sampling scheme (e.g., the variance and covariance of a Binomial count for sampling with replacement). This provides a self-contained, moment-based verification that the sampling-dependent constants $\left(A, B, C \right)$ of Proposition~\ref{prop:es-valid} are correctly identified, complementing the general-$x$ proof in Appendix~\ref{app:proof_prop_es_valid} with a computation tied explicitly to the SGD trajectory.
    
    Let us consider a weighted stochastic gradient with sampling weights
    $\bm{v}_{k} = (v_{k, 1}, \dots, v_{k, n})^{\top}$:
    \[
        \nabla f_{\bm{v}_{k}}(\bm{x}_{k}) \coloneqq \frac{1}{n} \sum_{i=1}^{n} v_{k, i} \nabla f_{i}(\bm{x}_{k}),
    \]
    where $\mathbb{E} \left[ v_{k, i} \right] = 1$ for all $i$ and $v_{k, i}$ is independent of $\bm{x}_{k}$.
    \[
        \mathbb{E} \left[ \nabla f_{\bm{v}_{k}}(\bm{x}_{k}) \right] = \frac{1}{n} \sum_{i=1}^{n} \mathbb{E} \left[ v_{k, i} \right] \nabla f_{i}(\bm{x}_{k}) = \frac{1}{n} \sum_{i=1}^{n} \nabla f_{i}(\bm{x}_{k}) = \nabla f(\bm{x}_{k}).
    \]
    Start from the squared norm:
    \begin{align*}
        \left \Vert \nabla f_{\bm{v}_{k}}(\bm{x}_{k}) \right \Vert^{2} &= \left \Vert \frac{1}{n} \sum_{i=1}^{n} v_{k, i} \nabla f_{i}(\bm{x}_{k}) \right \Vert^{2} \\
        &= \frac{1}{n^{2}} \sum_{i=1}^{n} v_{k, i}^{2} \left \Vert \nabla f_{i}(\bm{x}_{k}) \right \Vert^{2} + \frac{1}{n^{2}} \sum_{\substack{i, j = 1 \\ i \neq j}}^{n} v_{k, i} v_{k, j} \langle \nabla f_{i}(\bm{x}_{k}), \nabla f_{j}(\bm{x}_{k}) \rangle \\
        &= \frac{1}{n^{2}} \sum_{i=1}^{n} v_{k, i}^{2} \left \Vert \nabla f_{i}(\bm{x}_{k}) \right \Vert^{2} + \frac{2}{n^{2}} \sum_{1 \leq i < j \leq n} v_{k, i} v_{k, j} \langle \nabla f_{i}(\bm{x}_{k}), \nabla f_{j}(\bm{x}_{k}) \rangle.
    \end{align*}
    Taking the expectation over $\bm{v}_{k}$ yields the following general identity, which we use repeatedly below for each sampling scheme:
    \begin{align}
    \begin{split}
    \label{eq:moment-decomp-identity}
        \mathbb{E} \left \Vert \nabla f_{\bm{v}_{k}}(\bm{x}_{k}) \right \Vert^{2} 
        &= \frac{1}{n^{2}} \sum_{i=1}^{n} \mathbb{E} \left[ v_{k, i}^{2} \right] \left \Vert \nabla f_{i}(\bm{x}_{k}) \right \Vert^{2}\\ 
        &\quad + \frac{2}{n^{2}} \sum_{1 \leq i < j \leq n} \mathbb{E} \left[ v_{k, i} v_{k, j} \right] \langle \nabla f_{i}(\bm{x}_{k}), \nabla f_{j}(\bm{x}_{k}) \rangle.
    \end{split}
    \end{align}

    \paragraph{(i) Sampling with Replacement.}
    Let $S_{i} \sim \text{Binomial}(\tau, q_{i})$ and $v_{i} \coloneqq S_{i} / (\tau q_{i})$. Then,
    \[
        \mathbb{E} \left[ v_{i}^{2} \right] = 1 - \frac{1}{\tau} + \frac{1}{\tau q_{i}}, \quad \mathbb{E} \left[ v_{i} v_{j} \right] = 1 - \frac{1}{\tau}, \quad i \neq j.
    \]
    Plug the above expressions into \eqref{eq:moment-decomp-identity}:
    \begin{align*}
        &\mathbb{E} \left \Vert \nabla f_{\bm{v}_{k}}(\bm{x}_{k}) \right \Vert^{2} \\
        &\quad = \frac{1}{n^{2}} \sum_{i=1}^{n} \left( 1 - \frac{1}{\tau} + \frac{1}{\tau q_{i}} \right) \left \Vert \nabla f_{i}(\bm{x}_{k}) \right \Vert^{2} + \frac{2}{n^{2}} \sum_{1 \leq i < j \leq n} \left( 1 - \frac{1}{\tau} \right) \langle \nabla f_{i}(\bm{x}_{k}), \nabla f_{j}(\bm{x}_{k}) \rangle \\
        &\quad = \left( 1 - \frac{1}{\tau} \right) \frac{1}{n^{2}} \left( \sum_{i=1}^{n} \left \Vert \nabla f_{i}(\bm{x}_{k}) \right \Vert^{2} + 2 \sum_{1 \leq i < j \leq n} \langle \nabla f_{i}(\bm{x}_{k}), \nabla f_{j}(\bm{x}_{k}) \rangle \right) \\
        &\qquad + \frac{1}{\tau n^{2}} \sum_{i=1}^{n} \frac{1}{q_{i}} \left \Vert \nabla f_{i}(\bm{x}_{k}) \right \Vert^{2} \\
        &\quad = \left( 1 - \frac{1}{\tau} \right) \left \Vert \nabla f(\bm{x}_{k}) \right \Vert^{2} + \frac{1}{\tau n^{2}} \sum_{i=1}^{n} \frac{1}{q_{i}} \left \Vert \nabla f_{i}(\bm{x}_{k}) \right \Vert^{2}.
    \end{align*}
    Use $L_{i}$-smoothness and $f_{i}^{\inf}$:
    \[
       \left \Vert \nabla f_{i}(\bm{x}_{k}) \right \Vert^{2} \leq 2 L_{i}(f_{i}(\bm{x}_{k}) - f_{i}^{\inf}).
    \]
    Relate the average component-wise sub-optimality to global sub-optimality:
    \begin{align*}
        \frac{1}{n} \sum_{i=1}^{n} (f_{i}(\bm{x}_{k}) - f_{i}^{\inf}) &= \left( \frac{1}{n} \sum_{i=1}^{n} f_{i}(\bm{x}_{k}) \right) - \left( \frac{1}{n} \sum_{i=1}^{n} f_{i}^{\inf} \right) \\
        &= f(\bm{x}_{k}) - \frac{1}{n} \sum_{i=1}^{n} \left( f^{\inf} - (f^{\inf} - f_{i}^{\inf}) \right) \\
        &= f(\bm{x}_{k}) -f^{\inf} + \frac{1}{n} \sum_{i=1}^{n} (f^{\inf} - f_{i}^{\inf}) \\
        &= f(\bm{x}_{k}) -f^{\inf} + \Delta^{\inf},
    \end{align*}
    where $\Delta^{\inf} = \frac{1}{n} \sum_{i=1}^{n} (f^{\inf} - f_{i}^{\inf})$. Bound the variance component:
    \begin{align*}
        \frac{1}{\tau n^{2}} \sum_{i=1}^{n} \frac{1}{q_{i}} \left \Vert \nabla f_{i}(\bm{x}_{k}) \right \Vert^{2} &\leq \frac{2}{\tau n^{2}} \sum_{i=1}^{n} \frac{L_{i}}{q_{i}} (f_{i}(\bm{x}_{k}) - f_{i}^{\inf}) \\
        &\leq \frac{2}{\tau n^{2}} \sum_{i=1}^{n} \left( n \cdot \max_{j} \frac{L_{j}}{n q_{j}} \right) (f_{i}(\bm{x}_{k}) - f_{i}^{\inf}) \\
        &= \frac{2}{\tau n} \left( \max_{j} \frac{L_{j}}{n q_{j}} \right) \sum_{i=1}^{n}(f_{i}(\bm{x}_{k}) - f_{i}^{\inf}) \\
        &= \frac{2}{\tau n} \left( \max_{j} \frac{L_{j}}{n q_{j}} \right) \left\{ n \cdot (f(\bm{x}_{k}) - f^{\inf} + \Delta^{\inf}) \right\} \\
        &= 2 \left( \max_{j} \frac{L_{j}}{\tau n q_{j}} \right) \left( f(\bm{x}_{k}) - f^{\inf} + \Delta^{\inf} \right).
    \end{align*}
    Combine and finalize the bound:
    \begin{align*}
        &\mathbb{E} \left \Vert \nabla f_{\bm{v}_{k}}(\bm{x}_{k}) \right \Vert^{2} \\
        &\quad = \left( 1 - \frac{1}{\tau} \right) \left \Vert \nabla f(\bm{x}_{k}) \right \Vert^{2} + \frac{1}{\tau n^{2}} \sum_{i=1}^{n} \frac{1}{q_{i}} \left \Vert \nabla f_{i}(\bm{x}_{k}) \right \Vert^{2} \\
        &\quad \leq 2 \left( \max_{j} \frac{L_{j}}{\tau n q_{j}} \right) \left( f(\bm{x}_{k}) - f^{\inf}  \right) + \left( 1 - \frac{1}{\tau} \right) \left \Vert \nabla f(\bm{x}_{k}) \right \Vert^{2} + 2 \left( \max_{j} \frac{L_{j}}{\tau n q_{j}} \right) \Delta^{\inf},
    \end{align*}
    giving
    \[
        A = \max_{i} \frac{L_{i}}{\tau n q_{i}}, \quad B = 1 - \frac{1}{\tau}, \quad C = 2 A \Delta^{\inf},
    \]
    which coincides with the constants obtained in Case~1 of Proposition~\ref{prop:es-valid}.
    
    \paragraph{(ii) Independent Sampling without Replacement.}
    Let $v_{i} = \mathbbm{1}_{(i \in S)} / p_{i}$. Then,
    \[
        \mathbb{E} \left[ v_{i}^{2} \right] = \frac{1}{p_{i}}, \quad \mathbb{E} \left[ v_{i} v_{j} \right] = 1, \quad i \neq j.
    \]
    Plug the above expressions into \eqref{eq:moment-decomp-identity}:
    \begin{align*}
        &\mathbb{E} \left \Vert \nabla f_{\bm{v}}(\bm{x}_{k}) \right \Vert^{2} \\
        &\quad = \frac{1}{n^{2}} \sum_{i=1}^{n} \frac{1}{p_{i}} \left \Vert \nabla f_{i}(\bm{x}_{k}) \right \Vert^{2} + \frac{2}{n^{2}} \sum_{1 \leq i < j \leq n} \langle \nabla f_{i}(\bm{x}_{k}), \nabla f_{j}(\bm{x}_{k}) \rangle \\
        &\quad = \left( \frac{1}{n^{2}} \sum_{i=1}^{n} \left \Vert \nabla f_{i}(\bm{x}_{k}) \right \Vert^{2} + \frac{2}{n^{2}}\sum_{1 \leq i < j \leq n} \langle \nabla f_{i}(\bm{x}_{k}), \nabla f_{j}(\bm{x}_{k}) \rangle \right) \\
        &\quad \qquad + \frac{1}{n^{2}} \sum_{i=1}^{n} \left( \frac{1}{p_{i}} - 1 \right) \left \Vert \nabla f_{i}(\bm{x}_{k}) \right \Vert^{2} \\
        &\quad = \left \Vert \nabla f(\bm{x}_{k}) \right \Vert^{2} + \frac{1}{n^{2}} \sum_{i=1}^{n} \left( \frac{1}{p_{i}} - 1 \right) \left \Vert \nabla f_{i}(\bm{x}_{k}) \right \Vert^{2}.
    \end{align*}
    Use $L_{i}$-smoothness and $f_{i}^{\inf}$:
    \[
       \left \Vert \nabla f_{i}(\bm{x}_{k}) \right \Vert^{2} \leq 2 L_{i}(f_{i}(\bm{x}_{k}) - f_{i}^{\inf}).
    \]
    Bound the variance component:
    \begin{align*}
        \frac{1}{n^{2}} \sum_{i=1}^{n} \left( \frac{1}{p_{i}} - 1 \right) \left \Vert \nabla f_{i}(\bm{x}_{k}) \right \Vert^{2}
        &= \frac{1}{n^{2}} \sum_{i=1}^{n} \frac{1 - p_{i}}{p_{i}} \left \Vert \nabla f_{i}(\bm{x}_{k}) \right \Vert^{2} \\
        &\leq \frac{1}{n^{2}} \sum_{i=1}^{n} \frac{1 - p_{i}}{p_{i}} 2 L_{i}(f_{i}(\bm{x}_{k}) - f_{i}^{\inf}) \\
        &\leq \frac{2}{n^{2}} \sum_{i=1}^{n} \left( n \cdot \max_{j} \frac{(1 - p_{j})L_{j}}{p_{j} n} \right) (f_{i}(\bm{x}_{k}) - f_{i}^{\inf}) \\
        &= \frac{2}{n} \left( \max_{j} \frac{(1 - p_{j})L_{j}}{p_{j} n} \right) \sum_{i=1}^{n} (f_{i}(\bm{x}_{k}) - f_{i}^{\inf}) \\
        &= \frac{2}{n} \left( \max_{j} \frac{(1 - p_{j})L_{j}}{p_{j} n} \right) \left\{ n \cdot (f(\bm{x}_{k}) - f^{\inf} + \Delta^{\inf}) \right\} \\
        &= 2 \left( \max_{j} \frac{(1 - p_{j}) L_{j}}{p_{j} n} \right) \left( f(\bm{x}_{k}) - f^{\inf} + \Delta^{\inf} \right).
    \end{align*}
    Combine and finalize the bound:
    \begin{align*}
        &\mathbb{E} \left \Vert \nabla f_{\bm{v}_{k}}(\bm{x}_{k}) \right \Vert^{2} \\
        &\quad = \left \Vert \nabla f(\bm{x}_{k}) \right \Vert^{2} + \frac{1}{n^{2}} \sum_{i=1}^{n} \left( \frac{1}{p_{i}} - 1 \right) \left \Vert \nabla f_{i}(\bm{x}_{k}) \right \Vert^{2} \\
        &\quad \leq \left \Vert \nabla f(\bm{x}_{k}) \right \Vert^{2} + 2 \left( \max_{j} \frac{(1 - p_{j}) L_{j}}{p_{j} n} \right) \left( f(\bm{x}_{k}) - f^{\inf} + \Delta^{\inf} \right) \\
        &\quad = 2 \left( \max_{j} \frac{(1 - p_{j}) L_{j}}{p_{j} n} \right) (f(\bm{x}_{k}) - f^{\inf}) + \left \Vert \nabla f(\bm{x}_{k}) \right \Vert^{2} + 2 \left( \max_{j} \frac{(1 - p_{j}) L_{j}}{p_{j} n} \right) \Delta^{\inf},
    \end{align*}
    giving
    \[
        A = \max_{i} \frac{(1 - p_{i}) L_{i}}{p_{i} n}, \quad B = 1, \quad C = 2 A \Delta^{\inf},
    \]
    which coincides with the constants obtained in Case~2 of Proposition~\ref{prop:es-valid}.
    
    \paragraph{(iii) $\tau$-nice Sampling without Replacement.}
    Let $v_{i} = \frac{n}{\tau}\ \mathbbm{1}_{(i \in S)}$, $p_{i} = \tau / n$. Then,
    \[
        \mathbb{E} \left[ v_{i}^{2} \right] = \frac{n}{\tau}, \quad \mathbb{E} \left[ v_{i} v_{j} \right] = \frac{n(\tau - 1)}{\tau(n - 1)}, \quad i \neq j.
    \]
    Plug into \eqref{eq:moment-decomp-identity}:
    \begin{align*}
        \mathbb{E} \left \Vert \nabla f_{\bm{v}}(\bm{x}_{k}) \right \Vert^{2} &= \frac{n}{\tau n^{2}} \sum_{i=1}^{n} \left \Vert \nabla f_{i}(\bm{x}_{k}) \right \Vert^{2} + \frac{2}{n^{2}} \sum_{1 \leq i < j \leq n} \frac{n(\tau - 1)}{\tau(n - 1)} \langle \nabla f_{i}(\bm{x}_{k}), \nabla f_{j}(\bm{x}_{k}) \rangle \\
        &= \frac{n(\tau - 1)}{\tau(n - 1)} \left \Vert \nabla f(\bm{x}_{k}) \right \Vert^{2} + \left( \frac{1}{n \tau} - \frac{n(\tau - 1)}{n^{2} \tau(n - 1)} \right) \sum_{i=1}^{n} \left \Vert \nabla f_{i}(\bm{x}_{k}) \right \Vert^{2} \\
        &= \frac{n(\tau - 1)}{\tau(n - 1)} \left \Vert \nabla f(\bm{x}_{k}) \right \Vert^{2} + \frac{n - \tau}{n \tau(n - 1)} \sum_{i=1}^{n} \left \Vert \nabla f_{i}(\bm{x}_{k}) \right \Vert^{2}.
    \end{align*}
    Use $L_{i}$-smoothness and $f_{i}^{\inf}$:
    \[
       \left \Vert \nabla f_{i}(\bm{x}_{k}) \right \Vert^{2} \leq 2 L_{i}(f_{i}(\bm{x}_{k}) - f_{i}^{\inf}).
    \]
    Bound the variance component:
    \begin{align*}
        \frac{n - \tau}{n\tau(n - 1)} \sum_{i=1}^{n} \left \Vert \nabla f_{i}(\bm{x}_{k}) \right \Vert^{2}
        &\leq \frac{2(n - \tau)}{n\tau(n - 1)} \sum_{i=1}^{n} L_{i}(f_{i}(\bm{x}_{k}) - f_{i}^{\inf}) \\
        &\leq \frac{2(n - \tau)}{n\tau(n - 1)} \left(\max_{j} L_{j} \right) \sum_{i=1}^{n} (f_{i}(\bm{x}_{k}) - f_{i}^{\inf}) \\
        &= \frac{2(n - \tau)}{n\tau(n - 1)} \left(\max_{j} L_{j} \right) \left\{ n \cdot (f(\bm{x}_{k}) - f^{\inf} + \Delta^{\inf}) \right\} \\
        &= 2 \frac{n - \tau}{\tau(n-1)} \left(\max_{j} L_{j} \right) \left( f(\bm{x}_{k}) - f^{\inf} + \Delta^{\inf} \right).
    \end{align*}
    Combine and finalize the bound:
    \begin{align*}
        &\mathbb{E} \left \Vert \nabla f_{\bm{v}_{k}}(\bm{x}_{k}) \right \Vert^{2} \\
        &\quad \leq \frac{n(\tau - 1)}{\tau(n-1)} \left \Vert \nabla f(\bm{x}_{k}) \right \Vert^{2} + 2 \frac{n - \tau}{\tau(n-1)} \left(\max_{j} L_{j} \right) \left( f(\bm{x}_{k}) - f^{\inf} + \Delta^{\inf} \right) \\
        &\quad = 2 \frac{n - \tau}{\tau(n-1)} \left( \max_{j} L_{j} \right) \left( f(\bm{x}_{k}) - f^{\inf} \right) \\
        &\qquad + \frac{n(\tau - 1)}{\tau(n-1)} \left \Vert \nabla f(\bm{x}_{k}) \right \Vert^{2} + 2 \frac{n - \tau}{\tau(n-1)} \left( \max_{j} L_{j} \right) \Delta^{\inf},
    \end{align*}
    giving
    \[
        A = \frac{n - \tau}{\tau(n-1)}\max_{i} L_{i}, \quad B = \frac{n(\tau - 1)}{\tau(n-1)}, \quad C = 2 A \Delta^{\inf},
    \]
    which coincides with the constants obtained in Case~3 of Proposition~\ref{prop:es-valid}.
    
    In all three cases, the constants $\left( A, B, C \right)$ obtained directly from the sampling-weight moments coincide exactly with those established in Proposition~\ref{prop:es-valid} via the general-$x$ argument in Appendix~\ref{app:proof_prop_es_valid}, confirming the consistency of the two derivations.

\section{Main Result}
    Consider solving Problem \eqref{stochastic_minimization} by using SGD defined for all $k$ by
    \begin{align}
        \bm{x}_{k+1} = \bm{x}_{k} - \eta_{k} \nabla f_{\bm{v}_{k}} (\bm{x}_{k}), \label{eq:es_sgd}
    \end{align}
    where $\bm{x}_{0} \in \mathbb{R}^{d}$, $\eta_{k} > 0$ is the step size, $v_{k}$ denotes the random sampling vector at iteration $k$, and $\nabla f_{\bm{v}}$ is defined as in \eqref{f_v_grad}. We first establish the one-step descent inequality that forms the basis of our analysis. It quantifies, in conditional expectation, the decrease in the objective value achieved by a single SGD step under the ES condition.

    \emph{Notation note.} \\
        Starting with this lemma and again in Theorem~\ref{thm:1}, $\tau_{k}$ denotes the mini-batch size at iteration $k$, which is allowed to depend on $k$. The same concept is denoted by $m$ in Section~\ref{sec:jensen-warmup}, where a generic \emph{fixed} mini-batch size is used in a general derivation. Thus, $\tau_{k}$ and $m$ both represent the number of samples in a mini-batch, with the notation chosen according to the context of each derivation. \\
    
    \begin{lemma}[One-step descent under ES] \label{lem:onestep}
        Let Assumptions~\ref{as:L-smooth}--\ref{as:es} hold and let $(\bm{x}_{k})$ be generated by SGD \eqref{eq:es_sgd} with $\eta_{k} > 0$. Then, for all $k \geq 0$,
        \begin{equation} \label{eq:onestep_descent}
            \begin{aligned}
                \mathbb{E} \left[ f(\bm{x}_{k+1}) \mid \bm{x}_{k} \right] 
                &\leq f(\bm{x}_{k}) - \eta_{k} \left( 1 - \frac{LB\eta_{k}}{2} \right) \left\| \nabla f(\bm{x}_{k}) \right\|^2 \\
                &\quad + LA\eta_{k}^2 \left( f(\bm{x}_{k}) - f^{\star} \right) + \frac{L\eta_{k}^2}{2} C.
            \end{aligned}
        \end{equation}
        In particular, $\tau_{k} \equiv 1$ gives
        \begin{align}
        \label{eq:onestep_descent_single}
        \begin{split}
            \mathbb{E} \left[ f(\bm{x}_{k+1}) \mid \bm{x}_{k} \right] 
            &\leq f(\bm{x}_{k}) - \eta_{k} \left( 1 - \frac{LB\eta_{k}}{2} \right) \left\| \nabla f(\bm{x}_{k}) \right\|^2 \\ 
            &\quad + LA\eta_{k}^2 \left( f(\bm{x}_{k}) - f^{\star} \right) + \frac{L\eta_{k}^2}{2} C,
        \end{split}
        \end{align}
        which is the form used in Theorem~\ref{thm:1}; setting $A=0$, $B=1$ and $C=\sigma^2$ (the bounded-variance case \eqref{eq:bv}), together with the telescoping argument in the proof of Theorem~\ref{thm:1}, recovers the finite-time bound \eqref{eq:umeda_lemma} of \cite{umeda2024accelerates}.
    \end{lemma}
            
    \begin{proof}
        Write $\bm{g}_{k} \coloneqq g_{\bm{v}_{k}}(\bm{x}_{k})$, so that the SGD update reads $\bm{x}_{k+1} - \bm{x}_{k} = -\eta_{k} \bm{g}_{k}$. Applying the descent lemma \eqref{eq:L-smooth-descent} of Assumption~\ref{as:L-smooth} with $\bm{y} = \bm{x}_{k+1}$ and $\bm{x} = \bm{x}_{k}$ gives
        \begin{align*}
            f(\bm{x}_{k+1}) &\leq f(\bm{x}_{k}) + \left\langle \nabla f(\bm{x}_{k}), \bm{x}_{k+1} - \bm{x}_{k} \right\rangle + \frac{L}{2} \left\| \bm{x}_{k+1} - \bm{x}_{k} \right\|^2 \\
            &= f(\bm{x}_{k}) - \eta_{k} \left\langle \nabla f(\bm{x}_{k}), \bm{g}_{k} \right\rangle + \frac{L\eta_{k}^2}{2} \left\| \bm{g}_{k} \right\|^2.
        \end{align*}
        From Assumption~\ref{as:unbiased} and the i.i.d.\ property of $v_{k,1},\dots,v_{k,\tau_k}$, $\mathbb{E} [ \bm{g}_{k} \mid \bm{x}_{k} ] = \nabla f(\bm{x}_{k})$ for every $\tau_k$, the cross term satisfies
        \[
            \mathbb{E} \left[ \left\langle \nabla f(\bm{x}_{k}), \bm{g}_{k} \right\rangle \mid \bm{x}_{k} \right] = \left \langle \nabla f(\bm{x}_{k}), \mathbb{E} \left[ \bm{g}_{k} \mid \bm{x}_{k} \right] \right \rangle = \left \Vert \nabla f(\bm{x}_{k}) \right \Vert^{2}.
        \]
        Taking the conditional expectation gives
        \begin{equation} \label{eq:onestep_mid}
            \mathbb{E} \left[ f(\bm{x}_{k+1}) \mid \bm{x}_{k} \right] \leq f(\bm{x}_{k}) - \eta_{k} \left \Vert \nabla f(\bm{x}_{k}) \right \Vert^{2} + \frac{L \eta_{k}^{2}}{2} \mathbb{E} \left[ \left \Vert \bm{g}_{k} \right \Vert^{2} \mid \bm{x}_{k} \right].
        \end{equation}
        It remains to bound the second moment of $\bm{g}_k$ for general $\tau_k$. By the variance-decomposition identity of Proposition~\ref{prop:var-decomp}, the i.i.d.\ property of $v_{k,1},\dots,v_{k,\tau_k}$, and the ES condition \eqref{eq:ES} of Assumption~\ref{as:es} applied to a single sample, we have
        \begin{align*}
            \mathbb{E} \left[ \left\| \bm{g}_{k} - \nabla f(\bm{x}_{k}) \right\|^{2} \mid \bm{x}_{k} \right]
            &= \frac{1}{\tau_{k}} \, \mathbb{E} \left[ \left\| \nabla f_{v_{k,1}}(\bm{x}_{k}) - \nabla f(\bm{x}_{k}) \right\|^{2} \mid \bm{x}_{k} \right] \\
            &\leq \frac{1}{\tau_{k}} \Big( 2A \left( f(\bm{x}_{k}) - f^{\star} \right) + (B-1) \left\| \nabla f(\bm{x}_{k}) \right\|^{2} + C \Big),
        \end{align*}
        so that, Proposition~\ref{prop:var-decomp} gives,
        \begin{equation} \label{eq:onestep_secondmoment}
            \mathbb{E} \left[ \left \Vert \bm{g}_{k} \right \Vert^{2} \mid \bm{x}_{k} \right] 
            \leq \frac{2A}{\tau_{k}} \left( f(\bm{x}_{k}) - f^{\star} \right) + \left( 1 + \frac{B-1}{\tau_{k}} \right) \left \Vert \nabla f(\bm{x}_{k}) \right \Vert^{2} + \frac{C}{\tau_{k}}.
        \end{equation}
        (Note that \eqref{eq:onestep_secondmoment} reduces to the ES condition \eqref{eq:ES} itself when $\tau_{k}=1$.) Substituting \eqref{eq:onestep_secondmoment} into \eqref{eq:onestep_mid} and collecting the terms proportional to $\left \Vert \nabla f(\bm{x}_{k}) \right \Vert^{2}$, namely
        \[
            -\eta_{k} + \frac{L\eta_{k}^{2}}{2}\left(1+\frac{B-1}{\tau_{k}}\right) = -\eta_{k} \left( 1 - \frac{L\eta_{k}(\tau_{k}+B-1)}{2\tau_{k}} \right),
        \]
        gives \eqref{eq:onestep_descent}. Setting $\tau_{k}=1$ gives \eqref{eq:onestep_descent_single}; the last claim follows by substituting $A=0$, $B=1$, $C=\sigma^{2}$ and repeating the telescoping and normalization steps used in the proof of Theorem~\ref{thm:1}.
        \qed
    \end{proof}

    \begin{theorem}[Convergence guarantee of SGD under ES] \label{thm:1}
        Let Assumptions~\ref{as:L-smooth}--\ref{as:es} hold. Suppose that the step sizes $(\eta_{k})_{k\geq 0}$ satisfy
        \[
            \eta_{\max} \coloneqq \sup_{k\geq 0}\eta_k < \infty, \qquad LB\eta_{\max}<2.
        \]
        Then, for all $K \geq 1$, SGD \eqref{eq:es_sgd} satisfies that 

        \begin{equation} \label{eq:min-grad}
            \begin{split}
                \min_{0 \leq k \leq K} \mathbb{E} \left \Vert \nabla f(\bm{x}_{k}) \right \Vert^{2} &\leq \frac{2(f(\bm{x}_{0}) - f^{\star})}{(2-LB\eta_{\max})\sum_{k=0}^{K-1}\eta_{k}} \\ 
                &\quad + \frac{2LA}{2-LB\eta_{\max}} \cdot \frac{\sum_{k=0}^{K-1} \eta_{k}^{2}\mathbb{E} \left[f_{k} - f^{\star} \right]}{\sum_{k=0}^{K-1} \eta_{k}} \\
                &\quad + \frac{LC}{2-LB\eta_{\max}} \cdot \frac{\sum_{k=0}^{K-1} \eta_{k}^{2}} {\sum_{k=0}^{K-1} \eta_{k}},
            \end{split}
        \end{equation}
        where $f_{k} = f(\bm{x}_{k})$.
    \end{theorem}
    
    \begin{proof}
        By Lemma~\ref{lem:onestep}, for every $k \geq 0$ we have the one-step expected descent conditional on $\bm{x}_{k}$:
        \begin{align*}
            \mathbb{E} \left[f(\bm{x}_{k+1}) \mid \bm{x}_{k} \right] &\leq f(\bm{x}_{k}) - \eta_{k}
            \left( 1 - \frac{L B}{2} \eta_{k} \right) \left \Vert \nabla f(\bm{x}_{k}) \right \Vert^{2} \\
            &\quad+ LA \eta_{k}^{2} \left( f(\bm{x}_{k}) - f^{\star} \right) + \frac{L \eta_{k}^{2}}{2}C .
        \end{align*}
        Taking the total expectation on both sides, by using the tower property of conditional expectation, yields
        \begin{align*}
            \mathbb{E} \left[f(\bm{x}_{k+1}) \right] &\leq \mathbb{E} \left[f(\bm{x}_{k}) \right] - \eta_{k} \left( 1 - \frac{L B}{2} \eta_{k} \right) \mathbb{E} \left \Vert \nabla f(\bm{x}_{k}) \right \Vert^{2} \\
            &\quad+ LA \eta_{k}^{2} \mathbb{E} \left[f(\bm{x}_{k}) - f^{\star} \right] + \frac{L\eta_{k}^{2}}{2}C .
        \end{align*}
        Rearranging to isolate the expected squared gradient norm gives
        \begin{align*}
            \eta_{k}\left(1 - \frac{L B}{2} \eta_{k} \right) \mathbb{E} \left \Vert \nabla f(\bm{x}_{k}) \right \Vert^{2}  &\leq \mathbb{E}[f(\bm{x}_{k})] - \mathbb{E} \left[f(\bm{x}_{k+1}) \right] \\
            &\quad+ LA \eta_{k}^{2} \mathbb{E} \left[f(\bm{x}_{k}) - f^{\star} \right] + \frac{L\eta_{k}^{2}}{2}C .
        \end{align*}
        We sum the above inequality over $k = 0, \dots, K-1$. The first term on the right-hand side telescopes, and Assumption~\ref{as:lower} gives $\mathbb{E}[f(\bm{x}_{K})] \geq f^{\star}$, so that
        \[
            \sum_{k=0}^{K-1} \left( \mathbb{E} \left[f(\bm{x}_{k}) \right] - \mathbb{E} \left[f(\bm{x}_{k+1}) \right] \right) = f(\bm{x}_{0}) - \mathbb{E} \left[f(\bm{x}_{K}) \right] \leq f(\bm{x}_{0}) - f^{\star}.
        \]
        Hence,
        \begin{multline} \label{eq:summed}
            \sum_{k=0}^{K-1} \eta_{k} \left( 1 - \frac{L B}{2} \eta_{k} \right) \mathbb{E}  \left \Vert \nabla f(\bm{x}_{k}) \right \Vert^{2}  \\
            \leq f(\bm{x}_{0}) - f^{\star} + LA \sum_{k=0}^{K-1} \eta_{k}^{2} \mathbb{E} \left[f(\bm{x}_{k})-f^{\star} \right] + \frac{L}{2}C \sum_{k=0}^{K-1} \eta_{k}^{2} .
        \end{multline}
        From the step-size condition in Theorem~\ref{thm:1}, we have
        \[
            1-\frac{LB}{2}\eta_{\max}>0.
        \]
        Since $\eta_{k} \leq \eta_{\max}$ for every $k \geq 0$, it follows that
        \begin{equation} \label{eq:positive-factor}
            1-\frac{LB}{2}\eta_{k} \geq 1-\frac{LB}{2}\eta_{\max} > 0.
        \end{equation}
        In particular, the quantity $2 - LB\eta_{\max}$ appearing in \eqref{eq:min-grad} is strictly positive, so every denominator below is well defined. Using \eqref{eq:positive-factor} and the nonnegativity of $\mathbb{E}[\Vert \nabla f(\bm{x}_{k}) \Vert^{2}]$, we may lower bound the left-hand side of \eqref{eq:summed}:
        \[
            \sum_{k=0}^{K-1} \eta_{k} \left( 1 - \frac{LB}{2} \eta_{k} \right) \mathbb{E} \left \Vert \nabla f(\bm{x}_{k}) \right \Vert^{2}  \geq \left( 1 - \frac{LB}{2} \eta_{\max} \right) \sum_{k=0}^{K-1}\eta_{k} \mathbb{E}  \left \Vert \nabla f(\bm{x}_{k}) \right \Vert^{2}.
        \]
        Combining the two displayed inequalities gives
        \begin{multline*}
            \left( 1 - \frac{LB}{2} \eta_{\max} \right) \sum_{k=0}^{K-1}\eta_{k} \mathbb{E}  \left \Vert \nabla f(\bm{x}_{k}) \right \Vert^{2} \\
            \leq f(\bm{x}_{0}) - f^{\star} + LA\sum_{k=0}^{K-1} \eta_{k}^{2} \mathbb{E} \left[f(\bm{x}_{k}) - f^{\star} \right] + \frac{L}{2}C \sum_{k=0}^{K-1} \eta_{k}^{2}.
        \end{multline*}   
        In particular, when $B=0$, this inequality becomes
        \begin{multline*}
            \sum_{k=0}^{K-1}\eta_{k} \mathbb{E} \left \Vert \nabla f(\bm{x}_{k}) \right \Vert^{2}
            \leq f(\bm{x}_{0}) - f^{\star} + LA\sum_{k=0}^{K-1} \eta_{k}^{2} \mathbb{E} \left[f(\bm{x}_{k}) - f^{\star} \right] + \frac{L}{2}C \sum_{k=0}^{K-1} \eta_{k}^{2}.
        \end{multline*}
        The subsequent weighted-averaging argument then gives the corresponding bound, which coincides with \eqref{eq:min-grad} after setting $B=0$.
    
        Since the minimum over $k = 0, \dots, K-1$ never exceeds the corresponding weighted average with the positive weights $\eta_{k}$,
        \[
            \min_{0 \leq k \leq K-1} \mathbb{E} \left \Vert \nabla f(\bm{x}_{k}) \right \Vert^{2} \leq \frac{\sum_{k=0}^{K-1} \eta_{k} \mathbb{E} \left \Vert \nabla f(\bm{x}_{k}) \right \Vert^{2} } {\sum_{k=0}^{K-1}\eta_{k}} ,
        \]
        and since taking the minimum over the larger index set $\{0, \dots, K\}$ can only decrease its value,
        \[
            \min_{0 \leq k \leq K} \mathbb{E} \left \Vert \nabla f(\bm{x}_{k}) \right \Vert^{2}  \leq \min_{0 \leq k \leq K-1} \mathbb{E} \left \Vert \nabla f(\bm{x}_{k}) \right \Vert^{2},
        \]
        so the same upper bound holds for $\min_{0 \leq k \leq K} \mathbb{E} \Vert \nabla f(\bm{x}_{k}) \Vert^{2}$. Dividing the previous inequality by the positive quantity $\left( 1 - \frac{LB}{2} \eta_{\max} \right) \sum_{k=0}^{K-1} \eta_{k}$ therefore yields
        \begin{align*}
            \min_{0 \leq k \leq K} \mathbb{E} \left \Vert \nabla f(\bm{x}_{k}) \right \Vert^{2} &\leq \frac{f(\bm{x}_{0}) - f^{\star}} {\left(1 - \frac{LB}{2} \eta_{\max}\right) \sum_{k=0}^{K-1}\eta_{k}} \\
            &\quad+ \frac{LA\sum_{k=0}^{K-1} \eta_{k}^{2} \mathbb{E} \left[f(\bm{x}_{k}) - f^{\star} \right]} {\left( 1 - \frac{LB}{2} \eta_{\max} \right) \sum_{k=0}^{K-1} \eta_{k}} \\
            &\quad+ \frac{\frac{L}{2}C \sum_{k=0}^{K-1} \eta_{k}^{2}} {\left( 1 - \frac{LB}{2} \eta_{\max} \right) \sum_{k=0}^{K-1} \eta_{k}} .
        \end{align*}
        Finally, from the algebraic identity,
        \[
            \frac{1}{1 - \frac{LB}{2} \eta_{\max}} = \frac{2}{2 - LB \eta_{\max}},
        \]
        which is well defined by \eqref{eq:positive-factor}, the displayed bound above is equivalent to \eqref{eq:min-grad}. This completes the proof.
        \qed
    \end{proof}
    
    The convergence guarantee in Theorem~\ref{thm:1} is supported by a sequence of detailed analyses. The variance decomposition above establishes the variance decomposition underlying the ES condition. Section~\ref{sec:jensen-warmup} further develops the convergence analysis under the ES condition and makes explicit how the stochastic-gradient second moment and the ES constants $\left( A, B, C \right)$ enter the convergence bound. These analyses provide detailed derivations underlying the convergence guarantee stated in Theorem~\ref{thm:1}.

    Unlike the analysis in \cite{khaled2020better}, where the convergence rates for constant step sizes explicitly depend on the total number of iterations $K$, our results provide fully explicit, sampling-specific convergence bounds directly in terms of the individual smoothness constants $\{ L_{i} \}$, sampling probabilities, and the data heterogeneity measure $\Delta^{\inf}$. This allows the use of flexible step-size schedules such as harmonic decay, polynomial decay, or cosine annealing, without requiring prior knowledge of $K$ and thereby improves on the both theoretical guarantees and practical applicability.

    \begin{corollary}[A general convergence bound with bounded function gaps] \label{cor:general_rate} Under the assumptions of Theorem~\ref{thm:1}, define
        \begin{equation}
            D_k \coloneqq \mathbb{E} \left[ f(\bm{x}_k) - f^{\star} \right], \qquad M_K \coloneqq \max_{0 \leq k \leq K-1} D_k,
            \label{eq:def_D_MK}
        \end{equation}
        and
        \begin{equation}
            S_1(K) \coloneqq \sum_{k=0}^{K-1} \eta_k, \qquad S_2(K) \coloneqq \sum_{k=0}^{K-1} \eta_k^2.
            \label{eq:def_S1_S2}
        \end{equation}
        Then, for all $K \geq 1$,
        \begin{equation}
            \begin{aligned}
                \min_{0 \leq k \leq K-1} \mathbb{E} \left\| \nabla f(\bm{x}_k) \right\|^2 
                &\leq \frac{2\left( f(\bm{x}_0) - f^{\star} \right)}{\left( 2 - LB\eta_{\max} \right) S_1(K)} \\
                &\quad + \frac{2LA M_K + LC}{2 - LB\eta_{\max}} \frac{S_2(K)}{S_1(K)}.
            \end{aligned}
            \label{eq:general-rate-MK}
        \end{equation}
    \end{corollary}

    \begin{remark}[Implications for specific step-size sequences] \label{rem:specific_stepsizes}
        The general bound \eqref{eq:general-rate-MK} shows that the convergence behavior is determined by the two quantities.
        \[
            S_1(K) = \sum_{k=0}^{K-1} \eta_k, \qquad \frac{S_2(K)}{S_1(K)} = \frac{\sum_{k=0}^{K-1} \eta_k^2}{\sum_{k=0}^{K-1} \eta_k},
        \]
        together with the behavior of $M_K$. For step-size sequences satisfying $\sum_{k=0}^{\infty} \eta_k^2 < \infty$, Proposition~\ref{prop:function-gap-bounded} gives a uniform bound $M_K \leq M$, and hence, \eqref{eq:general-rate-MK} yields a closed convergence rate.
    
        In particular, for the harmonic step size,
        \[
            \eta_k = \frac{\eta}{k+1},
        \]
        we have
        \[
            S_1(K) = \eta H_K, \qquad S_2(K) = \eta^2 H_K^{(2)},
        \]
        where $H_K = \sum_{k=1}^K k^{-1}$ and $H_K^{(2)} = \sum_{k=1}^K k^{-2}$. Since
        \[
            H_K = \log K + O(1), \qquad H_K^{(2)} \leq \frac{\pi^2}{6},
        \]
        it follows that
        \[
            \min_{0 \leq k \leq K-1} \mathbb{E} \left\| \nabla f(\bm{x}_k) \right\|^2 = O\left( \frac{1}{\log K} \right).
        \]
    
        For the polynomial step size,
        \[
            \eta_k = \eta(k+1)^{-\alpha}, \qquad \frac{1}{2} < \alpha < 1,
        \]
        we have
        \[
            S_1(K) = \Theta(K^{1-\alpha}), \qquad S_2(K) = O(1),
        \]
        and Proposition~\ref{prop:function-gap-bounded} applies. Hence,
        \[
            \min_{0 \leq k \leq K-1} \mathbb{E} \left\| \nabla f(\bm{x}_k) \right\|^2 = O(K^{\alpha-1}).
        \]
    
        For the constant step size $\eta_k = \eta$ and the cosine step size,
        \[
            \eta_k = \frac{\eta}{2} \left( 1 + \cos \frac{\pi k}{K} \right),
        \]
        the quantities $S_2(K)$ are not uniformly bounded as $K$ increases. Consequently, the uniform bound on $M_K$ does not follow from Proposition~\ref{prop:function-gap-bounded} alone. In these cases, \eqref{eq:general-rate-MK} gives the corresponding bounds with $M_K$ left explicit. In particular, for the constant step size, we have
        \[
            S_1(K) = K\eta, \qquad S_2(K) = K\eta^2,
        \]
        and therefore
        \[
            \min_{0 \leq k \leq K-1} \mathbb{E} \left\| \nabla f(\bm{x}_k) \right\|^2 \leq \frac{2(f(\bm{x}_0) - f^\star)}{(2 - LB\eta_{\max})K\eta} + \frac{2LAM_K + LC}{2 - LB\eta_{\max}} \eta.
        \]
        For the cosine step size,
        \[
            S_1(K) = \frac{\eta}{2}(K+1), \qquad S_2(K) = \frac{\eta^2}{8}(3K+4),
        \]
        and hence,
        \[
            \min_{0 \leq k \leq K-1} \mathbb{E} \left\| \nabla f(\bm{x}_k) \right\|^2 \leq \frac{4(f(\bm{x}_0) - f^\star)}{(2 - LB\eta_{\max})\eta(K+1)} + \frac{2LAM_K + LC}{2 - LB\eta_{\max}} \frac{\eta(3K+4)}{4(K+1)}.
        \]
    \end{remark}

    \begin{remark}[Extending the polynomial-decay rate beyond the uniformly bounded regime] \label{rem:poly-below-half}
        Remark~\ref{rem:specific_stepsizes} treats the polynomial step size $\eta_k = \eta(k+1)^{-\alpha}$ only for $\frac{1}{2} < \alpha < 1$, since $\sum_{k=0}^{\infty} \eta_k^2 < \infty$ holds only in this range, allowing Proposition~\ref{prop:function-gap-bounded} to furnish the uniform bound $M_K \leq M$. For $0 < \alpha \leq \frac{1}{2}$, $\sum_{k=0}^{\infty} \eta_k^2$ diverges, so Proposition~\ref{prop:function-gap-bounded} no longer applies and $M_K$ must be kept explicit in \eqref{eq:general-rate-MK}, exactly as for the constant and cosine step sizes above. We now derive the resulting rate for this remaining range using the same Riemann-sum--to--integral approximation,
        \[
            \sum_{k=0}^{K-1} \eta_k^p \approx \eta^p \int_1^K t^{-p\alpha} \, dt, \qquad p \in \{1, 2\},
        \]
        which we also use below as an independent asymptotic cross-check of the exact finite-$K$ expressions for $S_1(K)$ and $S_2(K)$ obtained in Remark~\ref{rem:specific_stepsizes} for the harmonic, constant, and cosine schedules.
    
        For $0 < \alpha < 1$,
        \[
            S_1(K) \approx \frac{\eta}{1-\alpha}\left(K^{1-\alpha} - 1\right) = \Theta(K^{1-\alpha}),
        \]
        consistent with the exact order already used in Remark~\ref{rem:specific_stepsizes}. For $S_2(K)$, the same approximation gives
        \[
            S_2(K) \approx \eta^2 \int_1^K t^{-2\alpha}\,dt =
            \begin{cases}
                \Theta(K^{1-2\alpha}), & 0 < \alpha < \tfrac{1}{2}, \\[4pt]
                \Theta(\log K), & \alpha = \tfrac{1}{2}, \\[4pt]
                O(1), & \tfrac{1}{2} < \alpha < 1,
            \end{cases}
        \]
        where the last case recovers $S_2(K) = O(1)$ used in Remark~\ref{rem:specific_stepsizes}. For $0 < \alpha < \tfrac{1}{2}$,
        \[
            \frac{S_2(K)}{S_1(K)} = \Theta\!\left(K^{-\alpha}\right).
        \]
        Substituting this into \eqref{eq:general-rate-MK} yields
        \[
            \min_{0 \leq k \leq K-1} \mathbb{E} \left\| \nabla f(\bm{x}_k) \right\|^2
            \leq \frac{2(f(\bm{x}_0) - f^\star)}{(2 - LB\eta_{\max})} \, O\!\left(K^{\alpha - 1}\right)
            + \frac{2LA M_K + LC}{2 - LB\eta_{\max}} \, O\!\left(K^{-\alpha}\right),
        \]
        i.e., $\min_{0 \leq k \leq K-1} \mathbb{E} \|\nabla f(\bm{x}_k)\|^2 = O(K^{\alpha-1}) + O(M_K K^{-\alpha})$, with $M_K$ left explicit since it need not be uniformly bounded in this range. For the boundary case $\alpha = \tfrac{1}{2}$,
        \[
            \frac{S_2(K)}{S_1(K)} = \Theta\!\left(\frac{\log K}{\sqrt{K}}\right),
        \]
        so that
        \[
            \min_{0 \leq k \leq K-1} \mathbb{E} \left\| \nabla f(\bm{x}_k) \right\|^2 = O\!\left(K^{-1/2}\right) + O\!\left(M_K \frac{\log K}{\sqrt{K}}\right).
        \]
        Together with Remark~\ref{rem:specific_stepsizes}, this covers the polynomial step size over the entire range $0 < \alpha < 1$: for $\alpha > \tfrac{1}{2}$ the rate is the closed $O(K^{\alpha-1})$ bound obtained from the uniformly bounded $M_K \leq M$, while for $\alpha \leq \tfrac{1}{2}$ the same integral-approximation technique yields the rate above with $M_K$ left explicit, reflecting the loss of the uniform bound on the expected function gaps.
    
        As a further check, applying the same integral approximation to the harmonic step size ($\alpha = 1$ is excluded from the polynomial family but is analogous) and to the cosine step size reproduces, up to the leading order, the exact expressions $S_1(K) = \eta H_K$, $S_2(K) = \eta^2 H_K^{(2)}$ and $S_1(K) = \frac{\eta}{2}(K+1)$, $S_2(K) = \frac{\eta^2}{8}(3K+4)$ used in Remark~\ref{rem:specific_stepsizes}, confirming the consistency of the continuous-approximation route with the exact discrete sums computed there.
    \end{remark}

    \begin{proposition}[Uniform boundedness of the expected function gaps] \label{prop:function-gap-bounded} Under the assumptions of Theorem~\ref{thm:1}, suppose that
        \begin{equation}
            \sum_{k=0}^{\infty} \eta_k^2 < \infty.
            \label{eq:square-summable-stepsize}
        \end{equation}
        Then, the sequence $\left\{ \mathbb{E} \left[ f(\bm{x}_k) - f^{\star} \right] \right\}_{k \geq 0}$ is uniformly bounded. In particular, define
        \begin{equation}
            Q \coloneqq \sum_{k=0}^{\infty} \eta_k^2 < \infty
            \label{eq:def_Q}
        \end{equation}
        and
        \begin{equation}
            M \coloneqq \exp(LAQ) \left( f(\bm{x}_0) - f^{\star} + \frac{LC}{2}Q \right).
            \label{eq:def_M}
        \end{equation}
        Then,
        \begin{equation}
            \mathbb{E} \left[ f(\bm{x}_k) - f^{\star} \right] \leq M \qquad \text{for all } k \geq 0.
            \label{eq:uniform_function_gap}
        \end{equation}
        Consequently,
        \begin{equation}
            M_{K} \leq M \qquad \text{for all } K \geq 1.
            \label{eq:MK_leq_M}
        \end{equation}
    \end{proposition}

    \begin{corollary}[Sampling-specific convergence bounds] \label{cor:sampling-specific-bounds}
        Under the assumptions of Theorem~\ref{thm:1}, let
        \[
            D_k \coloneqq \mathbb{E}[f(\bm{x}_k)-f^\star], \qquad M_K \coloneqq \max_{0\leq k\leq K-1} D_k,
        \]
        and let
        \[
            \eta_{\max} \coloneqq \sup_{k\geq 0}\eta_k < \frac{2}{LB},
        \]
        when $B>0$. Then, the general convergence bound in Corollary~\ref{cor:general_rate} becomes sampling-specific by substituting the constants of Proposition~\ref{prop:es-valid} as follows.
    
        \begin{equation}
            \min_{0\leq k\leq K-1} \mathbb{E}\left\|\nabla f(\bm{x}_k)\right\|^2 \leq \frac{2(f(\bm{x}_0)-f^\star)}{(2-LB_s\eta_{\max})S_1(K)} + \frac{2LA_sM_K+LC_s}{2-LB_s\eta_{\max}} \frac{S_2(K)}{S_1(K)}.
            \label{eq:sampling-specific-general}
        \end{equation}
    
        \begin{enumerate}
            \item \textbf{Sampling with replacement.}
            \label{part:sampling-with-replacement}
            
            For sampling with replacement,
            \[
                A_{\mathrm{WR}} = \frac{1}{\tau n} \max_{1\leq i\leq n}\frac{L_i}{q_i}, \qquad B_{\mathrm{WR}} = 1-\frac{1}{\tau}, \qquad C_{\mathrm{WR}} = 2A_{\mathrm{WR}}\Delta^{\inf}.
            \]
            Consequently,
            \begin{equation}
                \begin{aligned}
                    \min_{0\leq k\leq K-1} \mathbb{E}\left\|\nabla f(\bm{x}_k)\right\|^2 
                    \leq{}& \frac{2(f(\bm{x}_0)-f^\star)}{\left(2-LB_{\mathrm{WR}}\eta_{\max}\right)S_1(K)} \\
                    &+ \frac{2LA_{\mathrm{WR}}M_K+LC_{\mathrm{WR}}}{2-LB_{\mathrm{WR}}\eta_{\max}} \frac{S_2(K)}{S_1(K)}.
                \end{aligned}
                \label{eq:wr-specific-bound}
            \end{equation}
            Equivalently,
            \begin{equation}
                \begin{aligned}
                    \min_{0\leq k\leq K-1} \mathbb{E}\left\|\nabla f(\bm{x}_k)\right\|^2 
                    \leq{}& \frac{2(f(\bm{x}_0)-f^\star)}{\left(2-\frac{L(\tau-1)}{\tau}\eta_{\max}\right)S_1(K)} \\
                    &+ \frac{2LA_{\mathrm{WR}}\left(M_K+\Delta^{\inf}\right)}{2-\frac{L(\tau-1)}{\tau}\eta_{\max}} \frac{S_2(K)}{S_1(K)},
                \end{aligned}
                \label{eq:wr-specific-bound-explicit}
            \end{equation}
            where
            \[
                A_{\mathrm{WR}} = \frac{1}{\tau n} \max_{1\leq i\leq n}\frac{L_i}{q_i}.
            \]
    
            \item \textbf{Independent sampling without replacement.}
            \label{part:independent-sampling-without-replacement}
            
            For independent sampling without replacement,
            \[
                A_{\mathrm{WOR}} = \frac{1}{n} \max_{1\leq i\leq n} \frac{(1-p_i)L_i}{p_i}, \qquad B_{\mathrm{WOR}}=1, \qquad C_{\mathrm{WOR}} = 2A_{\mathrm{WOR}}\Delta^{\inf}.
            \]
            Hence,
            \begin{equation}
                \begin{aligned}
                    \min_{0\leq k\leq K-1} \mathbb{E}\left\|\nabla f(\bm{x}_k)\right\|^2 
                    \leq{}& \frac{2(f(\bm{x}_0)-f^\star)}{\left(2-L\eta_{\max}\right)S_1(K)} \\
                    &+ \frac{2LA_{\mathrm{WOR}}M_K+LC_{\mathrm{WOR}}}{2-L\eta_{\max}} \frac{S_2(K)}{S_1(K)}.
                \end{aligned}
                \label{eq:wor-specific-bound}
            \end{equation}
            In particular,
            \[
                2LA_{\mathrm{WOR}}M_K+LC_{\mathrm{WOR}} = 2LA_{\mathrm{WOR}}\left(M_K+\Delta^{\inf}\right).
            \]
    
            \item \textbf{$\tau$-nice sampling without replacement.}
            \label{part:tau-nice-sampling-without-replacement}
            
            For $\tau$-nice sampling without replacement,
            \[
                A_{\tau\text{-Nice}} = \frac{n-\tau}{\tau(n-1)} \max_{1\leq i\leq n}L_i, \qquad B_{\tau\text{-Nice}} = \frac{n(\tau-1)}{\tau(n-1)},
            \]
            and
            \[
                C_{\tau\text{-Nice}} = 2A_{\tau\text{-Nice}}\Delta^{\inf}.
            \]
            Therefore,
            \begin{equation}
                \begin{aligned}
                    \min_{0\leq k\leq K-1} \mathbb{E}\left\|\nabla f(\bm{x}_k)\right\|^2 
                    \leq{}& \frac{2(f(\bm{x}_0)-f^\star)}{\left(2-\frac{Ln(\tau-1)}{\tau(n-1)}\eta_{\max}\right)S_1(K)} \\
                    &+ \frac{2LA_{\tau\text{-Nice}}M_K+LC_{\tau\text{-Nice}}}{2-\frac{Ln(\tau-1)}{\tau(n-1)}\eta_{\max}} \frac{S_2(K)}{S_1(K)}.
                \end{aligned}
                \label{eq:nice-specific-bound}
            \end{equation}
            Equivalently,
            \[
                2LA_{\tau\text{-Nice}}M_K+LC_{\tau\text{-Nice}} = 2LA_{\tau\text{-Nice}}\left(M_K+\Delta^{\inf}\right).
            \]
        \end{enumerate}
    
        Thus, the explicit ES constants of Proposition~\ref{prop:es-valid} directly determine the sampling-dependent convergence upper bounds. In particular, the sampling strategy and the mini-batch size enter the convergence guarantee through $A$, $B$, and $C$. The asymptotic order with respect to $K$ for each step-size schedule remains the same as in Corollary~\ref{cor:general_rate}, while the corresponding constants and the admissible step-size range are sampling-dependent.
    
        For the constant step-size schedule,
        \[
            \eta_k\equiv\eta, \qquad \eta<\frac{2}{LB},
        \]
        when $B>0$, one has
        \[
            S_1(K)=K\eta, \qquad S_2(K)=K\eta^2,
        \]
        and the general sampling-specific bound \eqref{eq:sampling-specific-general} reduces to
        \[
            \min_{0\leq k\leq K-1} \mathbb{E}\left\|\nabla f(\bm{x}_k)\right\|^2 \leq \frac{2(f(\bm{x}_0)-f^\star)}{(2-LB\eta)K\eta} + \frac{2LAM_K+LC}{2-LB\eta}\eta,
        \]
        with $\left( A, B, C \right)$ replaced by the corresponding sampling-specific expressions in parts~\ref{part:sampling-with-replacement}--\ref{part:tau-nice-sampling-without-replacement}.
    
        For the harmonic step-size schedule,
        \[
            \eta_k=\frac{\eta}{k+1},
        \]
        one has
        \[
            S_1(K)=\eta H_K, \qquad S_2(K)=\eta^2H_K^{(2)},
        \]
        and, whenever Proposition~\ref{prop:function-gap-bounded} applies, the corresponding explicit uniform bound on $M_K$ gives
        \[
            \min_{0\leq k\leq K-1} \mathbb{E}\left\|\nabla f(\bm{x}_k)\right\|^2 = O\!\left(\frac{1}{\log K}\right),
        \]
        with the multiplicative constant determined explicitly by the sampling-specific $\left( A, B, C \right)$ in parts~\ref{part:sampling-with-replacement}--\ref{part:tau-nice-sampling-without-replacement}.
    
        For the polynomial-decay schedule,
        \[
            \eta_k=\frac{\eta}{(k+1)^\alpha}, \qquad \frac12<\alpha<1,
        \]
        the same substitution gives
        \[
            S_1(K) = \eta\sum_{k=0}^{K-1}\frac{1}{(k+1)^\alpha}, \qquad S_2(K) = \eta^2\sum_{k=0}^{K-1}\frac{1}{(k+1)^{2\alpha}},
        \]
        and hence
        \[
            \min_{0\leq k\leq K-1} \mathbb{E}\left\|\nabla f(\bm{x}_k)\right\|^2 = O\!\left(K^{\alpha-1}\right),
        \]
        with the sampling-dependent prefactors determined by $\left( A, B, C \right)$ in parts~\ref{part:sampling-with-replacement}--\ref{part:tau-nice-sampling-without-replacement}. The extended range $0<\alpha\leq1/2$ is treated separately in Remark~\ref{rem:poly-below-half}.
    
        Finally, for the cosine-annealing schedule,
        \[
            \eta_k = \frac{\eta}{2} \left( 1+\cos\left(\frac{\pi k}{K}\right) \right),
        \]
        the corresponding substitution into \eqref{eq:sampling-specific-general} yields the sampling-specific upper bound for this schedule, with the sampling strategy entering through $\left( A, B, C \right)$.
    \end{corollary}

\section{Empirical Validation} \label{sec:empirical}
    The source code for this study is publicly available at an anonymized GitHub repository \url{https://github.com/iiduka-researches/sgd_es_jota2026.git}. In this section, we empirically validate our theoretical framework by training a ResNet-18 model on the CIFAR-100 dataset using SGD \cite{krizhevsky2009learning,he2016deep}. The results provide three key insights. First, we observe a macroscopic equivalence in performance across different mini-batch sampling strategies, which motivates a deeper analysis. Second, we rigorously confirm that the ES condition accurately models the empirical gradient variance in this practical deep-learning setting. Third, and most critically, we leverage the validated ES model to uncover a non-trivial, two-phase dynamic in the structure of stochastic gradient noise, which provides a microscopic explanation for the macroscopic observations.

    \subsection{Experimental Setup and Estimation Methodology}
        We evaluated our theoretical framework by training a ResNet-18 architecture on the CIFAR-100 dataset for 200 epochs. The network was optimized using standard stochastic gradient descent (SGD) without momentum, employing a mini-batch size of $128$ and a learning rate of $\eta = 0.1$. The step-size condition in Theorem \ref{thm:1}, 
        \[
            \eta_{k} \in \left( 0, \frac{2}{LB} \right),
        \]
        is a theoretical sufficient condition for convergence. In practical deep-learning settings, accurately estimating the global Lipschitz constant $L$ is difficult. Although our analysis provides explicit expressions for the ES constant $B$ under several common sampling schemes, the global theoretical upper bound $2/(LB)$ is generally difficult to evaluate directly for a deep neural network. Therefore, the learning rate $\eta = 0.1$ used in our experiments is not intended as a numerical verification of the theoretical step-size condition. Instead, it is adopted as a fixed practical learning rate for investigating the empirical behavior of the ES constants and the gradient noise structure. In particular, we estimate local ES coefficients $( \hat{A}_{k}, \hat{B}_{k}, \hat{C}_{k} )$ along the optimization trajectory to characterize the evolving stochastic gradient structure. The theoretical bound $2/(LB)$ can thus be interpreted as a convergence-guaranteeing guideline relating the admissible step size to the global smoothness of the objective and the stochastic gradient structure induced by the sampling scheme. Standard data augmentation techniques were applied, including random cropping ($32 \times 32$ with a padding of $4$), horizontal flipping, and random rotation ($15^{\circ}$). To accurately compute the full gradient $\nabla f(\bm{x})$ and the objective value $f(\bm{x})$, we construct a separate evaluation dataloader without data augmentation, using only normalization. This ensures that the gradient corresponds to the empirical risk rather than an augmented objective. The full gradient is obtained by accumulating gradients over the entire dataset using a sum-reduction loss, followed by normalization by the total number of samples. To investigate the evolution of the Expected Smoothness coefficients presented in Figures \ref{fig:es_validation} and \ref{fig:noise_A}--\ref{fig:noise_C}, we employed a dynamic estimation procedure. Although the theoretical constants $(A, B, C)$ governing the global function landscape are fixed, the \emph{local} geometry traversed by the optimizer evolves continuously. To accurately capture this trajectory-dependent behavior, we computed empirical local estimates $(\hat{A}_{k}, \hat{B}_{k}, \hat{C}_{k})$ dynamically at the end of each epoch. Specifically, we maintained a rolling buffer of the $W=50$ most recent epochs. At the end of each epoch, we recorded the suboptimality gap $f(\bm{x}) - \hat{f}^{\star}$---where the true minimum $f^{\star}$ is approximated by the lowest training loss observed up to that point $\hat{f}^{\star}$---alongside the squared full gradient norm $\left \Vert \nabla f(\bm{x}) \right \Vert^{2}$.
        We also estimate the second moment of stochastic gradients, 
        \[
            \mathbb{E} \left \Vert g_{\text{batch}} (\bm{x}) \right \Vert^{2},
        \]
        instead of the variance $\mathbb{E}  \Vert g_{\text{batch}}(\bm{x}) - \nabla f(\bm{x}) \Vert^{2}$. This is justified by the identity (see Proposition \ref{prop:var-decomp})
        \[
            \mathbb{E} \left \Vert g(\bm{x}) \right \Vert^{2} = \mathbb{E} \left \Vert g(\bm{x}) - \nabla f(\bm{x}) \right \Vert^{2} + \left \Vert \nabla f(\bm{x}) \right \Vert^{2},
        \]
        which allows the ES condition to be equivalently expressed in terms of second moments. At each epoch, this quantity is estimated via Monte Carlo sampling. Specifically, we draw multiple mini-batches (100 in our implementation), compute the squared norm of the corresponding stochastic gradients, and aggregate them using the median rather than the mean to improve robustness against heavy-tailed fluctuations. We adopt a generalized sampling framework using importance weights. Given per-sample losses $\ell_i$, the mini-batch loss is computed as
        \[
            L = \frac{1}{b} \sum_{i=1}^{b} v_{i} \ell_{i},
        \]
        where $v_{i}$ are sampling-dependent weights. This formulation unifies standard mini-batch sampling, sampling with replacement, independent Bernoulli sampling, and $\tau$-nice sampling. To avoid numerical instability, we discard degenerate samples with near-zero values of the suboptimality gap, gradient norm, or second moment. Using this buffer, we formulated a Non-Negative Least Squares (NNLS) problem to determine the optimal coefficients $(\hat{A}_{k}, \hat{B}_{k}, \hat{C}_{k})$ that minimize the residual of the ES inequality. To improve numerical stability, we normalize the regression variables before solving the NNLS problem. Specifically, both $(f(\bm{x}) - f^{\star})$ and $\left \Vert \nabla f(\bm{x}) \right \Vert^{2}$ are rescaled by their mean values within the buffer. After solving the normalized problem, the coefficients are rescaled back to their original units. This epoch-wise, data-driven methodology provides a rigorous means to dissect the dominant noise sources at various phases of the optimization process.

    \subsection{Macroscopic Equivalence of Sampling Strategies}
        We begin by examining the macroscopic convergence behavior, focusing on the primary metrics of training loss and test accuracy. As illustrated in Figure \ref{fig:performance_updated}, all four sampling strategies---Independent, Normal, Replacement, and \texorpdfstring{$\tau$}{tau}-Nice---produce remarkably similar learning curves. Despite the theoretical differences in their ES constants (as derived in Proposition \ref{prop:es-valid}), which imply distinct underlying variance structures, these differences do not manifest as significant variations in convergence speed or final model performance. This counter-intuitive outcome strongly suggests the presence of a dominant, underlying dynamic that governs the optimization trajectory, overshadowing the subtler effects of the sampling methods. At first sight, one would expect different sampling distributions to induce different gradient-variance structures, and therefore different optimization trajectories. However, under the ES framework, the second-moment expressions of all four sampling rules share the same decomposable structure each can be written in the form $A \left \Vert \nabla f(\bm{x}) \right \Vert^{2} + B (f(\bm{x}) - f^{\star}) + C$ with different constants. As a result, once these constants are absorbed into the learning-rate scaling, the macroscopic behavior becomes governed by the same underlying dynamic, making the strategies appear nearly equivalent in practice. This observation necessitates a microscopic analysis of the gradient noise to understand the source of this equivalence.

        \begin{figure}
            \centering
            \begin{minipage}{0.3\linewidth}
                \centering
                \includegraphics[width=\linewidth]{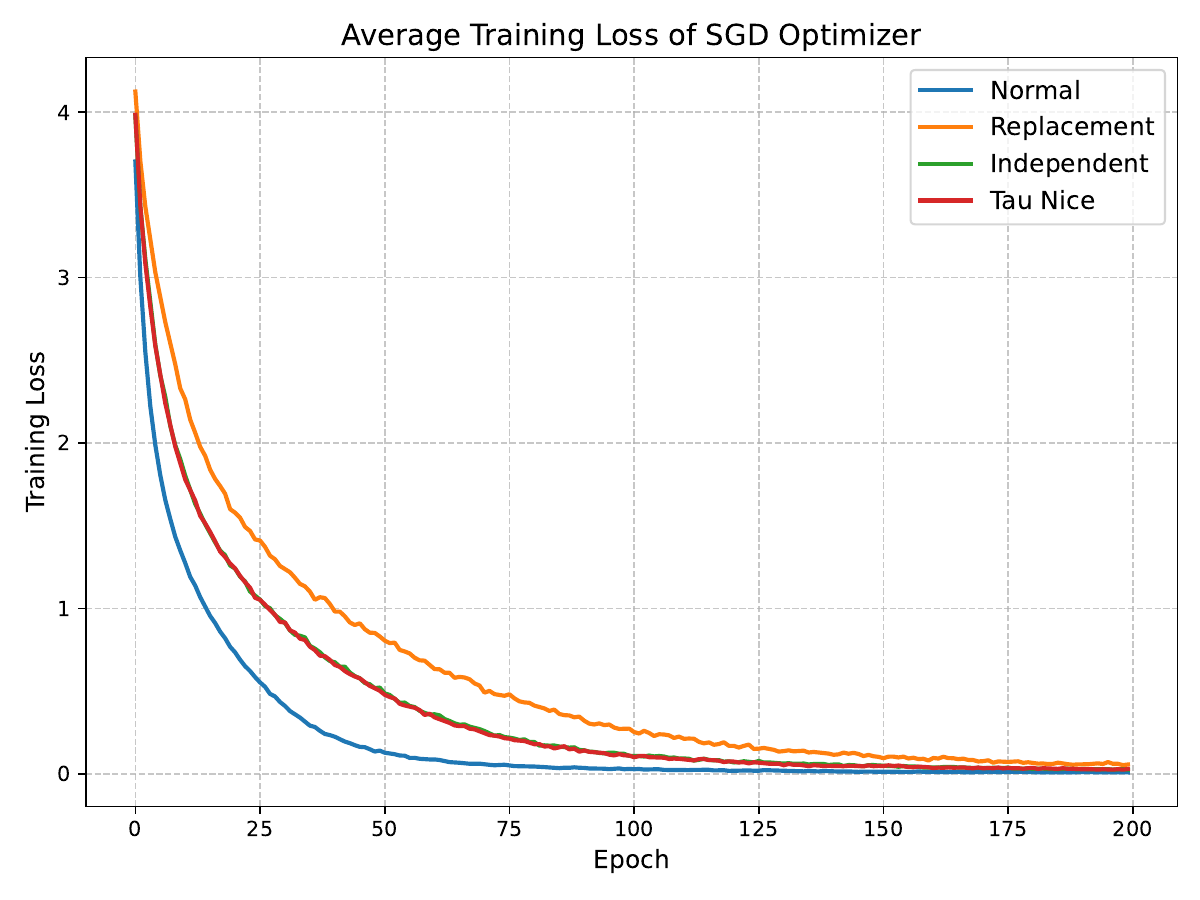} \vspace{1mm} (a) Training loss \label{fig:training_loss}
            \end{minipage}
            \hfill
            \begin{minipage}{0.3\linewidth}
                \centering
                \includegraphics[width=\linewidth]{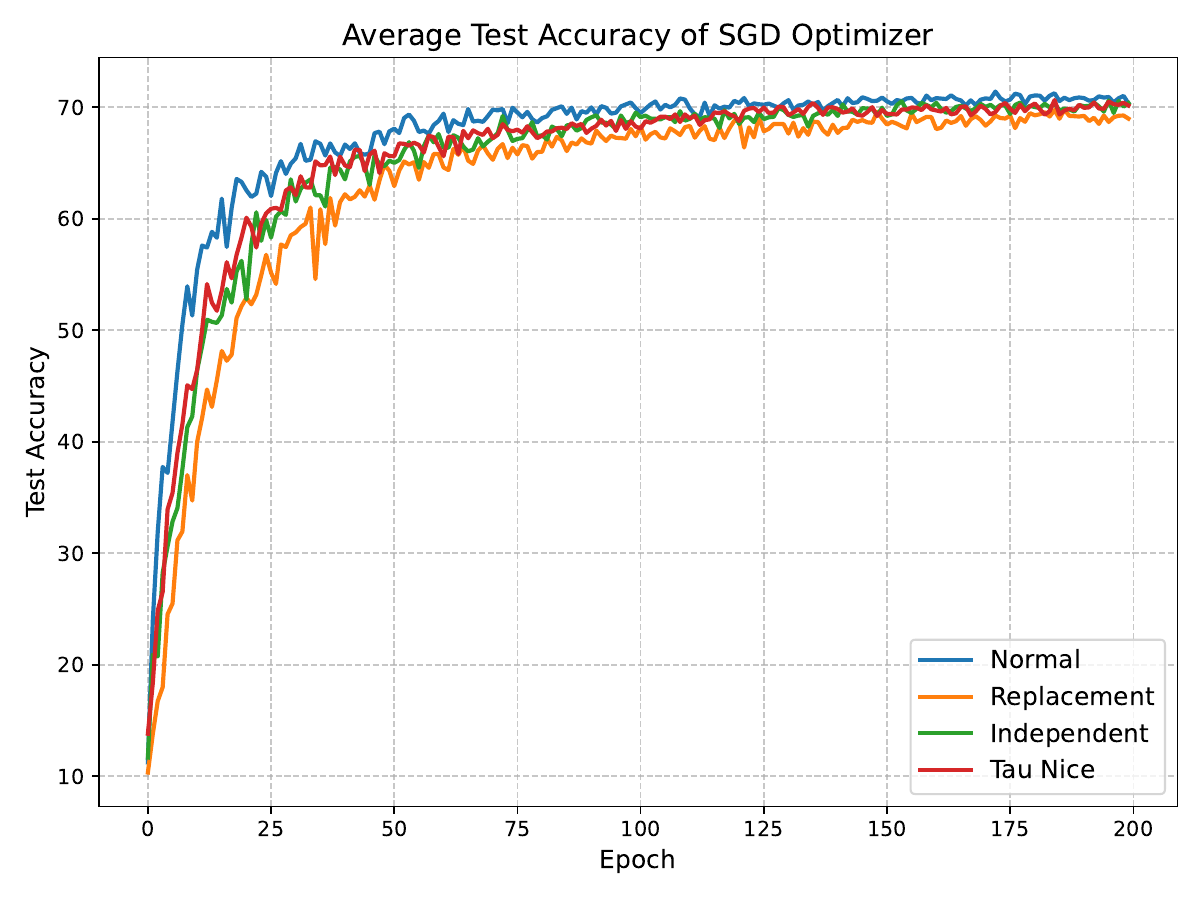} \vspace{1mm} (b) Test accuracy \label{fig:test_accuracy}
            \end{minipage}
            \hfill
            \begin{minipage}{0.3\linewidth}
                \centering
                \includegraphics[width=\linewidth]{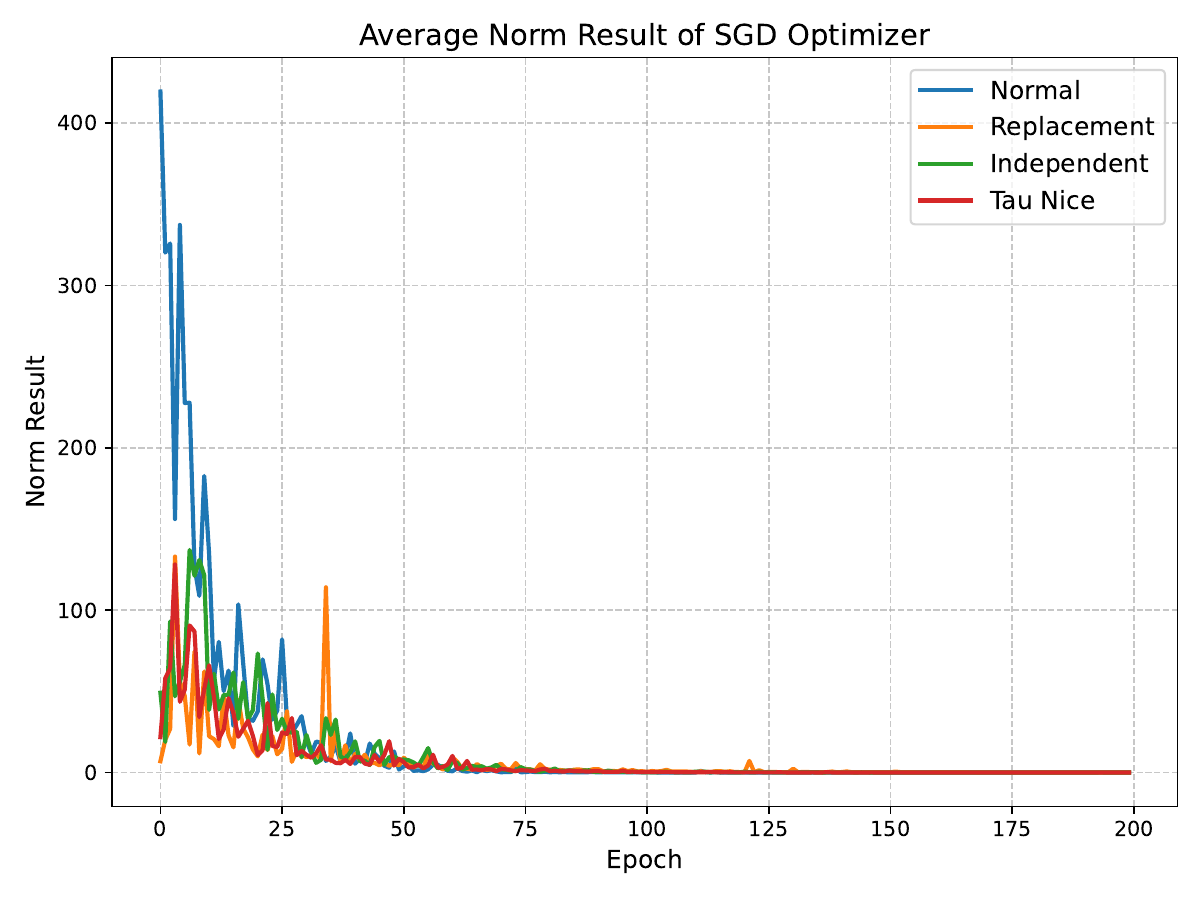} 
                \vspace{1mm} (c) Norm \label{fig:norm}
            \end{minipage}
            \caption{Training dynamics for SGD. All four sampling methods yield nearly identical loss, accuracy, and gradient norm curves. This macroscopic equivalence motivates a deeper investigation into the microscopic noise structure.} \label{fig:performance_updated}
        \end{figure}

    \subsection{Empirical Verification of the Expected Smoothness Model}
        Before utilizing the ES framework to dissect the noise structure, we first establish its empirical validity. The core of our analysis relies on the assumption that the ES condition can accurately model the second moment of stochastic gradients. Figure \ref{fig:es_validation} provides compelling evidence for this assumption. For all four sampling strategies, the predicted variance derived from our fitted ES model (RHS, the red dashed line) almost perfectly tracks the empirically measured second moment (LHS, the blue solid line). Furthermore, the residuals of this fit are centered around zero (see Figure \ref{fig:es_residuals}), confirming that there is no systematic bias in our model. This high-fidelity fit grants us confidence in using the ES coefficients $\left( A, B, C \right)$ as reliable proxies for the underlying properties of the stochastic noise.

        \begin{figure}[htbp]
                \centering
                \includegraphics[width=\linewidth]{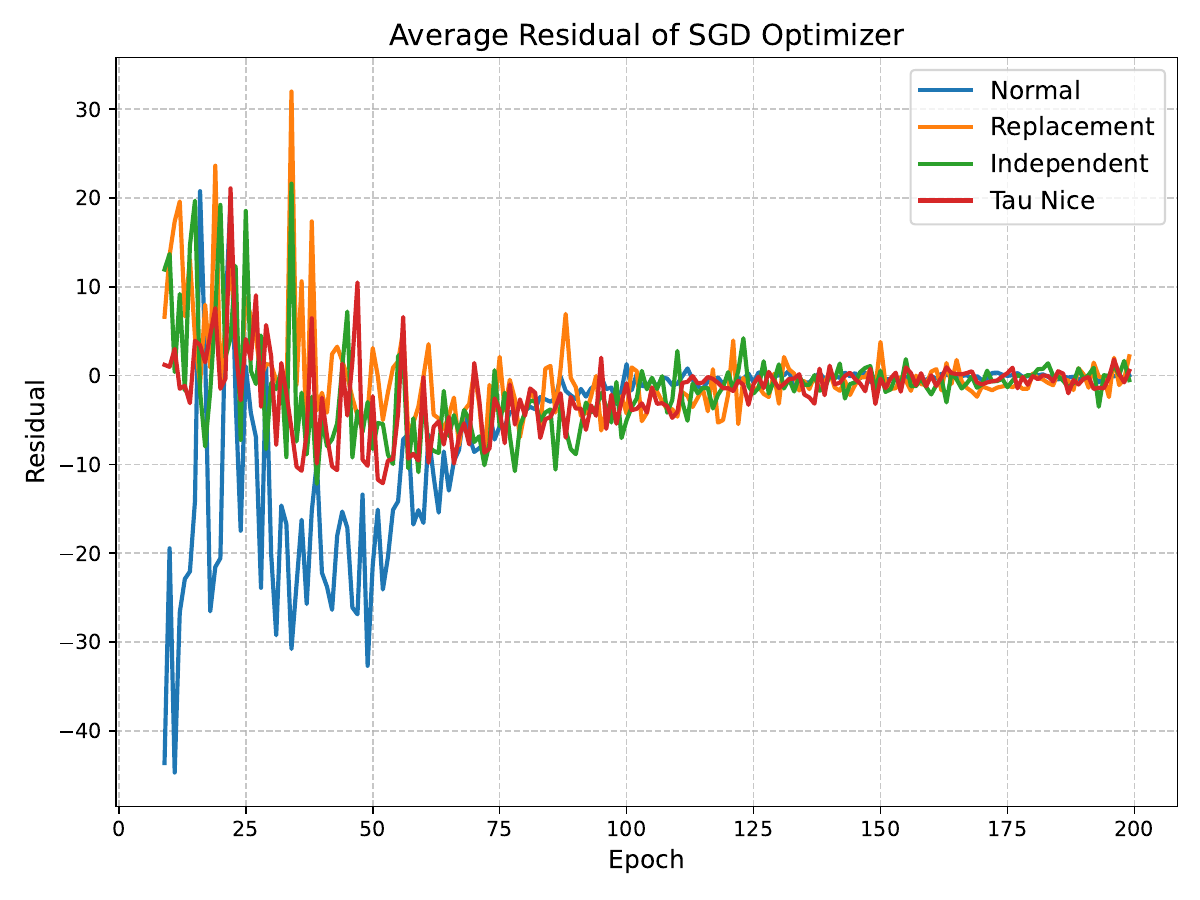}
                
                \caption{Residuals of the ES model fit. The residuals are centered at zero, indicating an unbiased and accurate model fit. This validates the use of the ES framework for our subsequent analysis.}
                \label{fig:es_residuals}
            \end{figure}

            \begin{figure}[htbp]
                \centering
                \begin{minipage}{0.48\linewidth}
                    \centering
                    \includegraphics[width=1.0\textwidth]{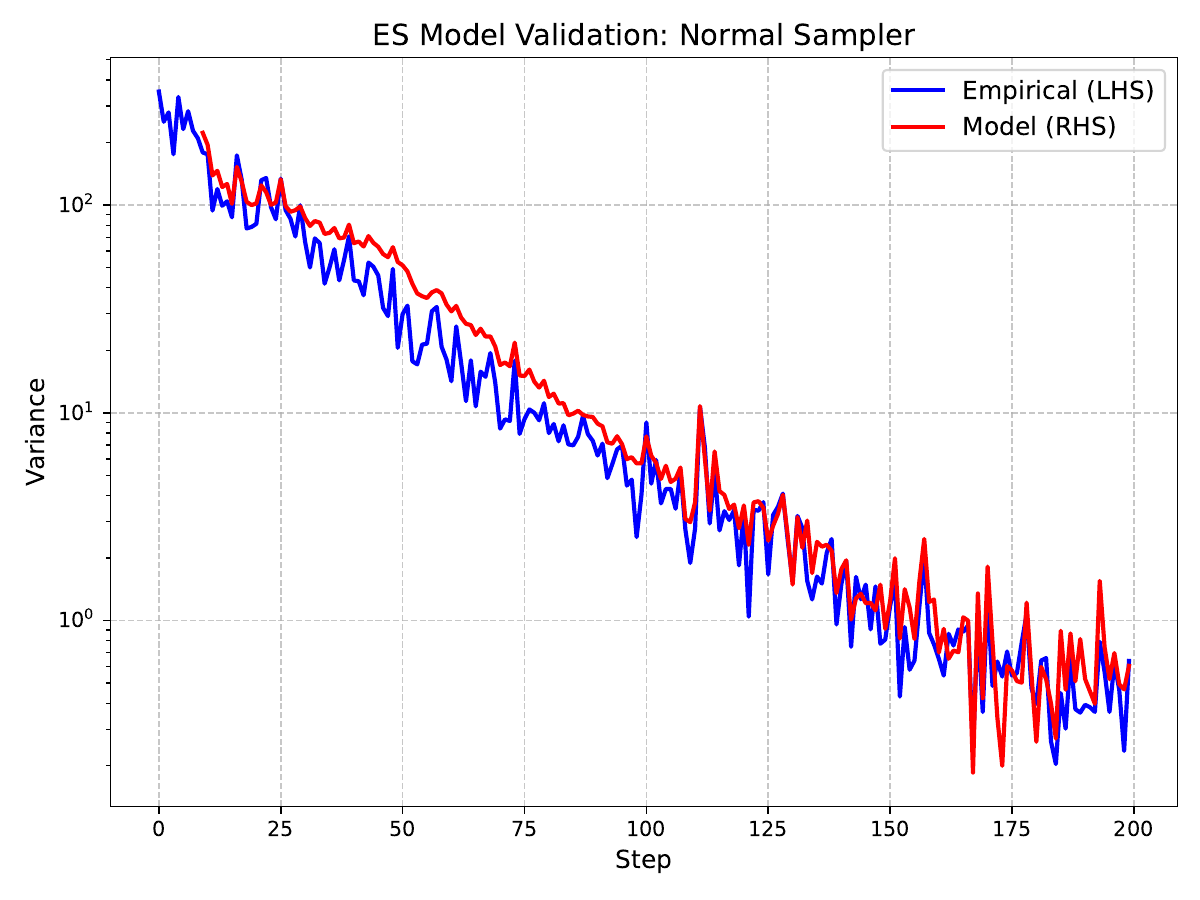}
                    \vspace{1mm} (a) normal sampler
                \end{minipage}
                \hfill
                \begin{minipage}{0.48\linewidth}
                    \centering
                    \includegraphics[width=1.0\textwidth]{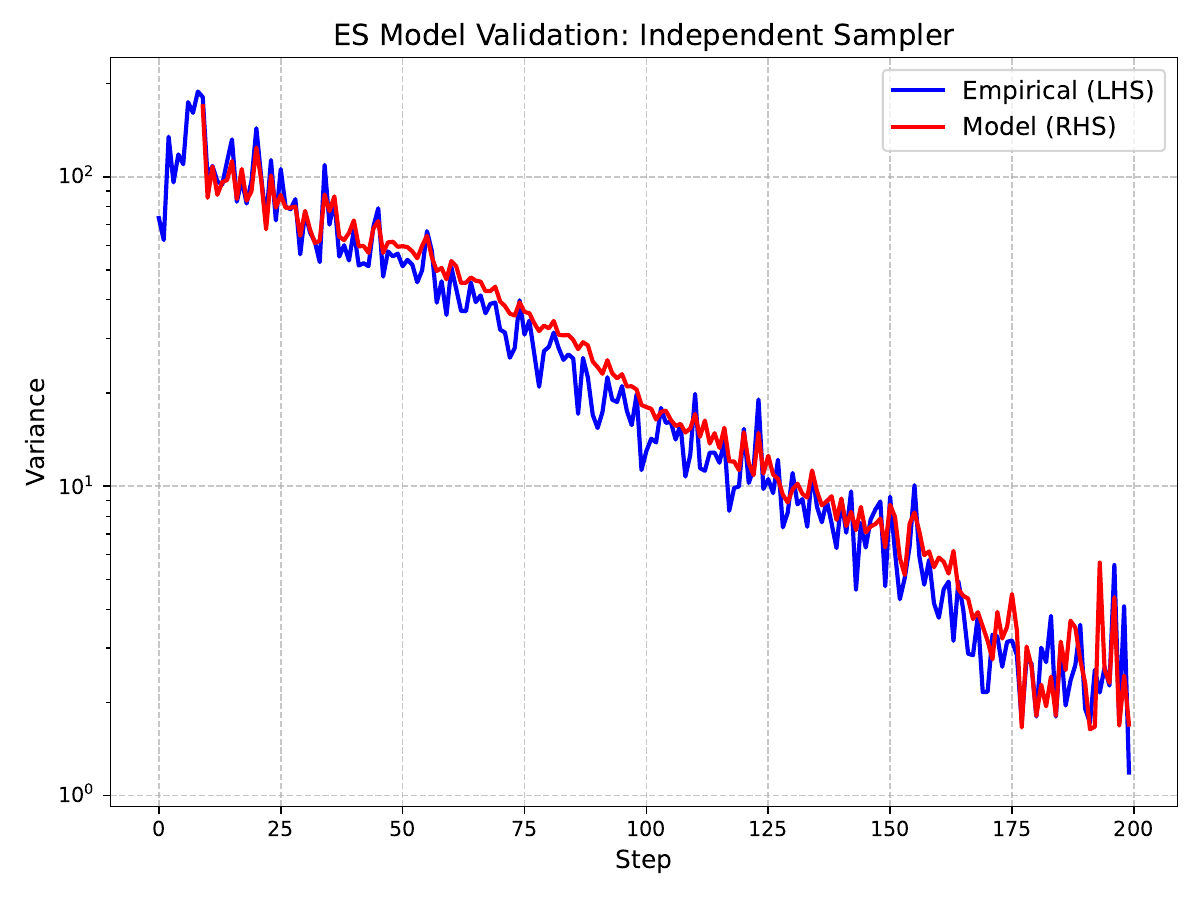}
                    \vspace{1mm} (b) independent sampler
                \end{minipage}
                \hfill
                \begin{minipage}{0.48\linewidth}
                    \centering
                    \includegraphics[width=1.0\textwidth]{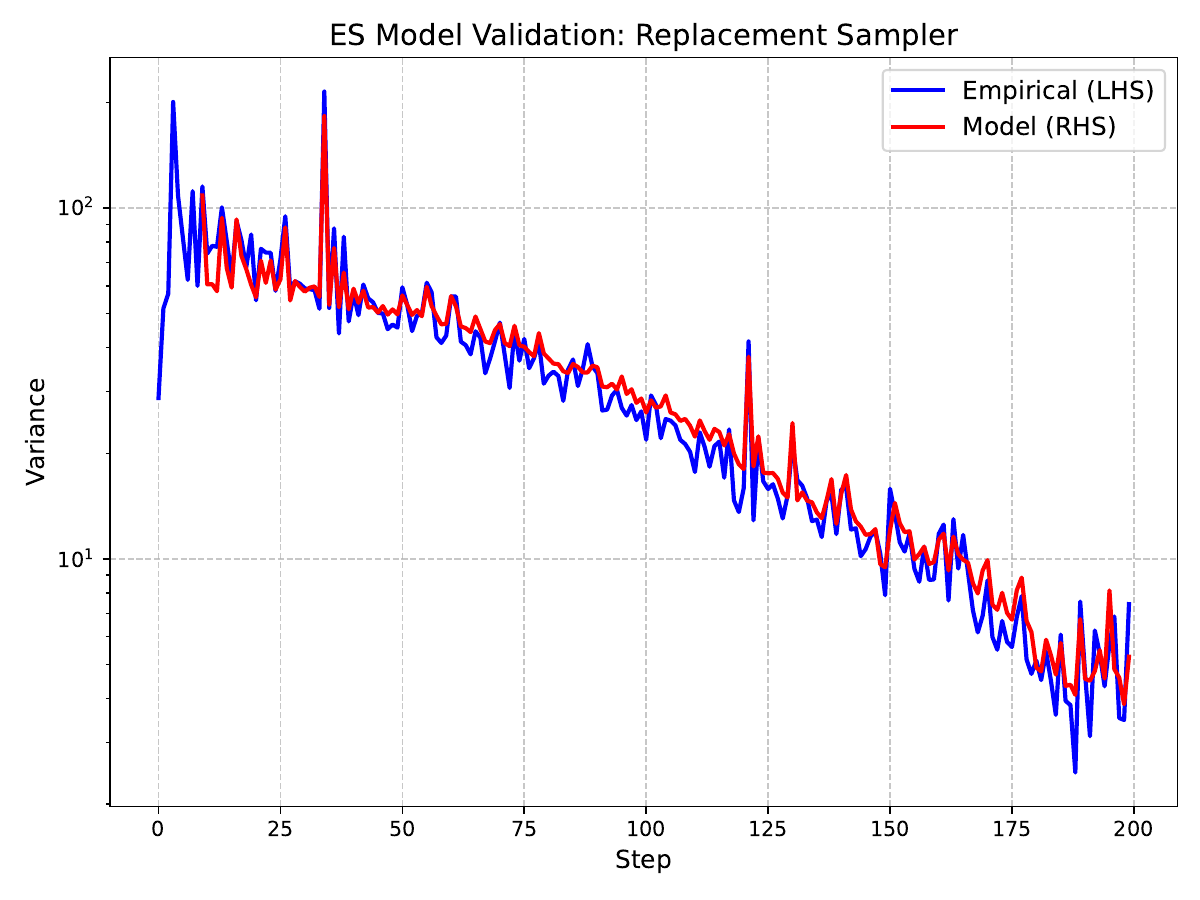}
                    \vspace{1mm} (c) replacement sampler
                \end{minipage}
                \hfill
                \begin{minipage}{0.48\linewidth}
                    \centering
                    \includegraphics[width=1.0\textwidth]{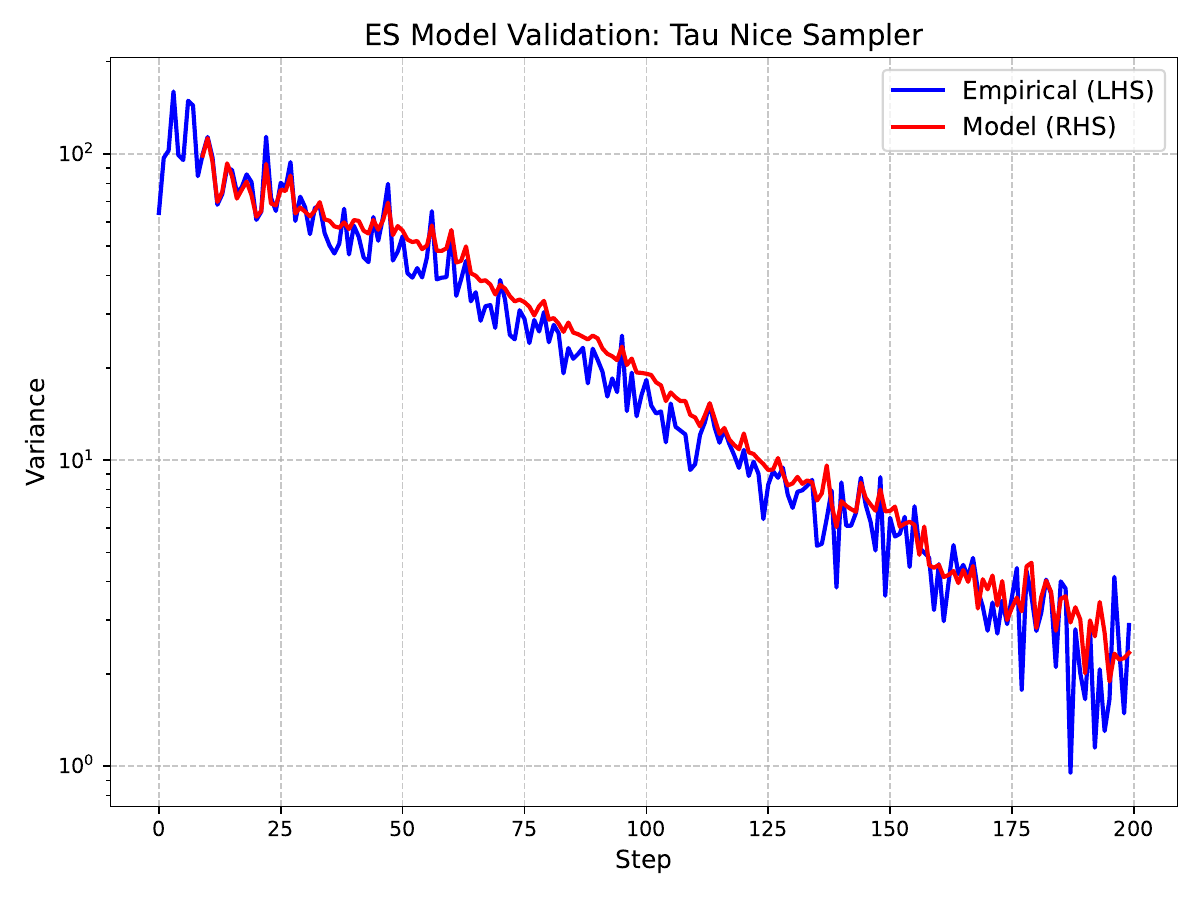}
                    \vspace{1mm} (d) tau-nice sampler
                \end{minipage} 
                \caption{Validation of the ES condition for different sampling strategies: (a) normal, (b) independent, (c) replacement and (d) $\tau$-nice sampler. In each plot, the ES model's predicted second moment (red dashed line) accurately captures the empirically measured variance (blue solid line).} \label{fig:es_validation}
            \end{figure}

    \subsection{The Two-Phase Microscopic Structure of Stochastic Noise}
        With the validity of the ES model established, we now analyze the evolution of the ES coefficients to reveal the changing nature of the gradient noise. As shown in Figure \ref{fig:noise_A}, \ref{fig:noise_B}, and \ref{fig:noise_C} the training process under SGD can be clearly divided into two distinct phases, characterized by a fundamental shift in the dominant source of stochastic variance.

        \subsubsection*{Phase 1: Suboptimality-Dominated Regime.}
            In the initial stages of training, the model parameters are far from any optimal configuration, leading to a large suboptimality gap $\left( f(\bm{x}_{k}) - f^{\star} \right)$. During this exploratory phase, the second moment of stochastic gradients is primarily driven by the terms associated with coefficients $A$ and $C$. As can be seen in Figures \ref{fig:noise_A} and \ref{fig:noise_C}, both coefficients are initially large and then rapidly decay. This indicates that the stochastic gradient noise is dominated by global suboptimality and a large baseline variance induced by random initialization. In this regime, the optimizer takes relatively large and exploratory steps to quickly reduce the overall loss.

        \subsubsection*{Phase 2: Geometry-Dominated Regime.}
            As the model approaches a low-loss region, the suboptimality gap shrinks, causing the influence of the $A$ and $C$ terms on the second moment to diminish significantly. Concurrently, a structural shift occurs: the geometry-related coefficient $B$, which is smaller in the early phase, increases and stabilizes at a relatively high value (Figure \ref{fig:noise_B}). This marks the transition to a refinement phase where the stochastic gradient noise is predominantly governed by the geometry-dependent term $\left \Vert \nabla f (\bm{x}_{k}) \right \Vert^{2}$. In this regime, the noise is no longer dictated by the distance to the solution, but rather by the local geometry of the loss landscape. The optimizer transitions from global exploration to local refinement within a basin of attraction. The transition between these two regimes is gradual rather than abrupt, and its exact location depends on the optimization trajectory. Nevertheless, this two-phase dynamic provides a compelling explanation for the macroscopic equivalence observed in Figure \ref{fig:performance_updated}. While different sampling strategies induce different theoretical constants, their impact is secondary compared to the dominant effect of the optimization phase, which governs the overall structure of stochastic gradient noise. This transition can be further interpreted as the stage where the geometry-dependent contribution to the second moment begins to dominate over the suboptimality-dependent components, providing a quantitative perspective on the shift in noise structure.
            
        \begin{figure}[htbp]
            \centering        
            \begin{subfigure}[t]{\textwidth}
                \centering
                \includegraphics[width=\textwidth, height=0.22\textheight, keepaspectratio]{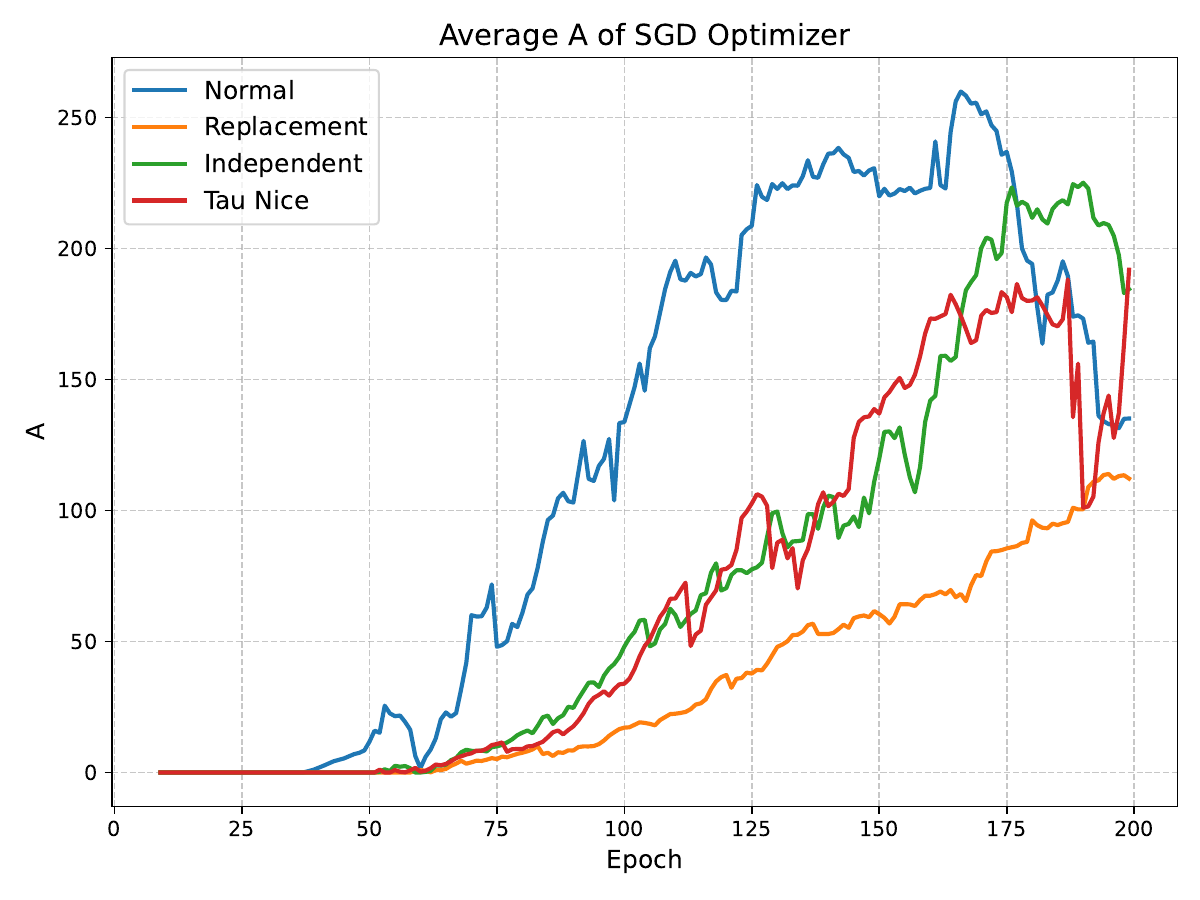}
                \caption{Coefficient $A$ (Suboptimality-related).}
                \label{fig:noise_A}
            \end{subfigure}
            
            \vspace{4pt}
            
            \begin{subfigure}[t]{\textwidth}
                \centering
                \includegraphics[width=\textwidth, height=0.22\textheight, keepaspectratio]{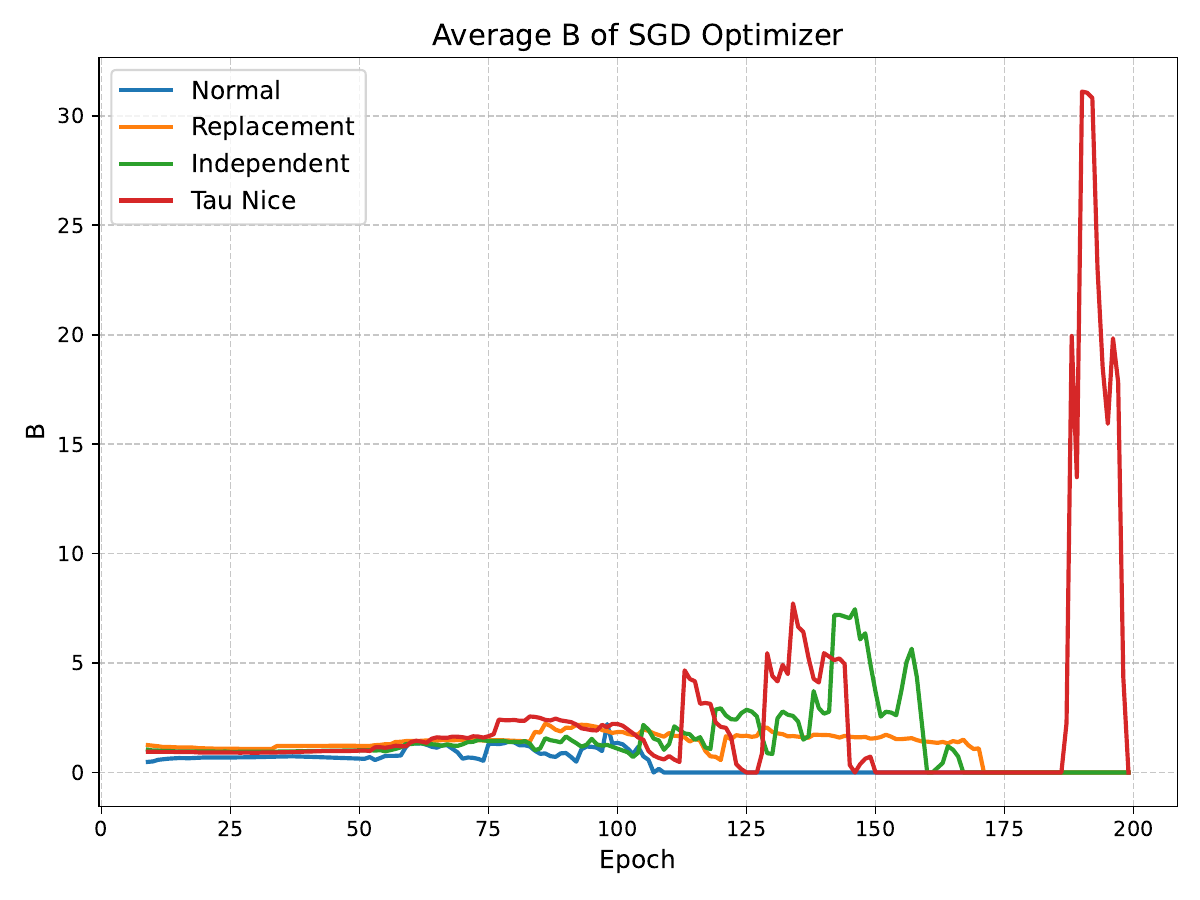}
                \caption{Coefficient $B$ (Geometry-related).}
                \label{fig:noise_B}
            \end{subfigure}
            
            \vspace{4pt}
            
            \begin{subfigure}[t]{\textwidth}
                \centering
                \includegraphics[width=\textwidth, height=0.22\textheight, keepaspectratio]{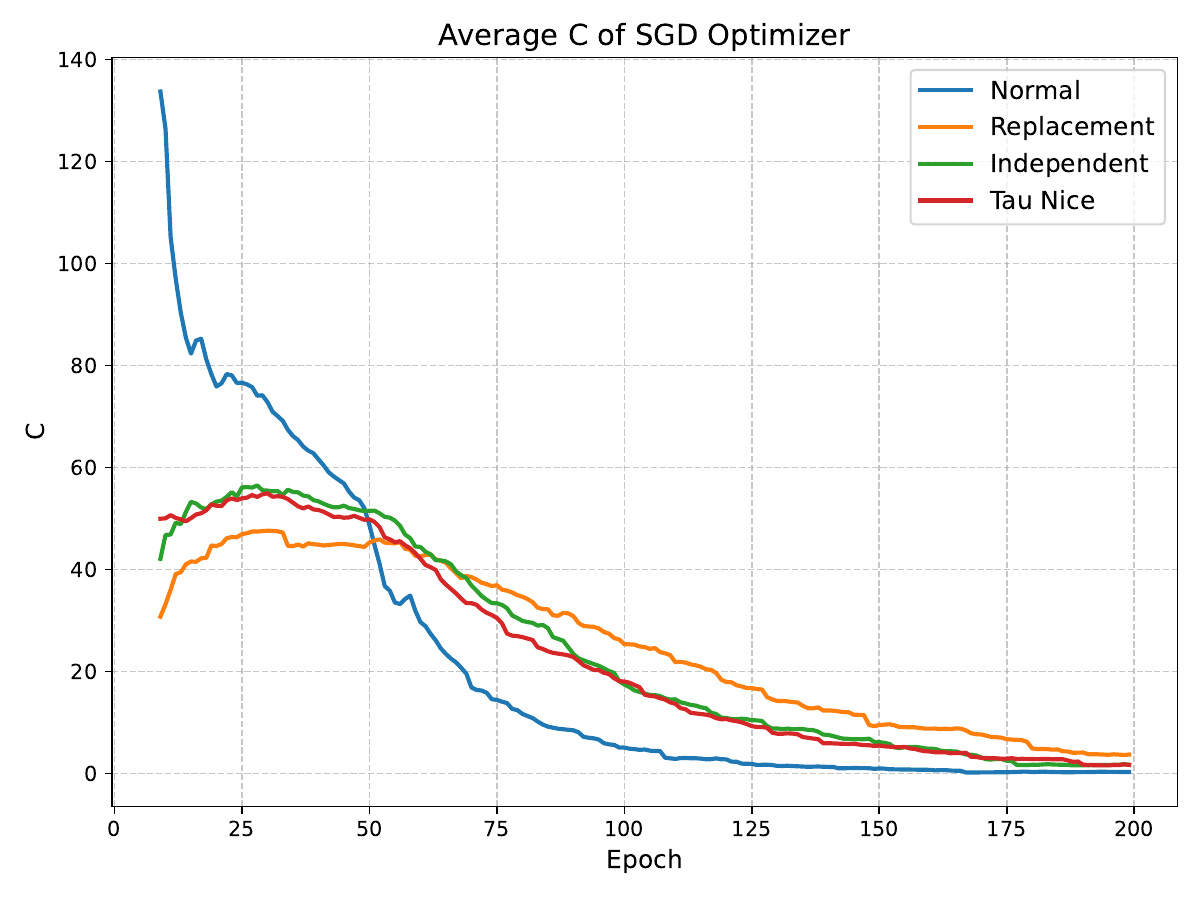}
                \caption{Coefficient $C$ (Baseline noise).}
                \label{fig:noise_C}
            \end{subfigure}
    
            \caption{
                Evolution of the estimated coefficients $A$, $B$, and $C$ in the second-moment model. A clear two-phase dynamic is observed. In the early stage, the variance is dominated by the suboptimality-dependent and baseline terms, $2A(f(\bm{x}_k)-f^\star)$ and $C$, both of which decay rapidly. In contrast, the geometry-dependent coefficient $B$ increases and stabilizes in the later stage, indicating a transition to a geometry-dominated regime. The transition between the two phases roughly corresponds to the point where the contribution of the geometry-dependent term becomes dominant over the suboptimality-dependent components.
            } \label{fig:noise_ABC}
        \end{figure}

\section{Related Work}
    The ES viewpoint was crystallized by \cite{khaled2020better}, who linked stochastic gradient second moments to progress in function value and gradient norm, yielding sharper rates. Recent studies analyzed coupled tuning of batch and step-size schedules, showing acceleration when both increase together \cite{umeda2024accelerates}. Our work unifies these ideas in a calculation-first manner and validates them empirically on a standard deep-learning benchmark.

\section{Conclusions}
    We provided a complete ES-based analysis of SGD with detailed proofs. Our empirical study on CIFAR-100 training confirms that the ES condition is not just a theoretical bound but an accurate descriptive model of stochastic gradient variance. We uncovered a two-phase structure in the gradient noise, transitioning from being dominated by sub-optimality to being dominated by the local geometry. These findings align theory with modern training practices and offer deeper insights into the optimization process.

\section*{Compliance with Ethical Standards} 
    The authors declare that there is no conflict of interest. This article does not contain any studies with human participants or animals performed by the authors.

\begin{acknowledgements}
    We are sincerely grateful to the two anonymous reviewers for helping us improve the original manuscript. This research is partly supported by the computational resources of the DGX A100 named TAIHO at Meiji University. This work was supported by JSPS KAKENHI Grant Number 24K14846. \\
\end{acknowledgements}

\clearpage
\appendix
\section{Auxiliary Proposition} \label{app:aux_prop}
    \begin{proposition}[Variance decomposition identity] \label{prop:var-decomp} 
        Under Assumption~\ref{as:unbiased}, for all $\bm{x} \in \mathbb{R}^{d}$, 
        \begin{align*}
            \mathbb{E}_{\xi} \left \Vert g_{\xi}(\bm{x}) - \nabla f(\bm{x}) \right \Vert^{2} = \mathbb{E}_{\xi} \left \Vert g_{\xi} \right \Vert^{2} - \left \Vert \nabla f(\bm{x}) \right \Vert^{2}.
        \end{align*}
    \end{proposition}

    \begin{proof}
        This is a direct consequence of the definition of variance: 
            \begin{align*}
                \mathbb{E}_{\xi} \left \Vert g_{\xi}(\bm{x}) - \nabla f(\bm{x}) \right \Vert^{2} &= \mathbb{E}_{\xi} \left \Vert g_{\xi}(\bm{x}) \right \Vert^{2} - 2 \left\langle \nabla f(\bm{x}), \mathbb{E}_{\xi} \left[g_{\xi}(\bm{x}) \right] \right\rangle + \left \Vert \nabla f(\bm{x}) \right \Vert^{2} \\
                &= \mathbb{E}_{\xi} \left \Vert g_{\xi}(\bm{x}) \right \Vert^{2} - \left \Vert \nabla f(\bm{x}) \right \Vert^{2}.
            \end{align*}
            \qed
    \end{proof}

\section{Auxiliary Lemmas} \label{app:aux_lemmas}
    \begin{lemma}[Gradient-function gap inequality] \label{lem:grad-gap}
        Let $f: \mathbb{R}^{d} \to \mathbb{R}$ be $L$-smooth and bounded below by $f^{\star}$. Then, for all $\bm{x} \in \mathbb{R}^{d}$,
        \[
            \left \Vert \nabla f(\bm{x}) \right \Vert^{2} \leq 2L \left( f(\bm{x}) - f^{\star} \right).
        \]
    \end{lemma}

    \begin{proof}
        By $L$-smoothness, we have for all $\bm{x}, \bm{y} \in \mathbb{R}^{d}$: 
        \[
            f(\bm{y}) \leq f(\bm{x}) + \langle \nabla f(\bm{x}), {\bm{y} - \bm{x}} \rangle + \frac{L}{2} \left \Vert \bm{y} - \bm{x} \right \Vert^{2}.
        \]
        Taking $\bm{y} = \bm{x} - \frac{1}{L} \nabla f(\bm{x})$ gives 
        \[
            f \left( \bm{x} - \frac{1}{L} \nabla f(\bm{x}) \right) \leq f(\bm{x}) - \frac{1}{2L} \left \Vert \nabla f(\bm{x}) \right \Vert^{2},
        \]
        which, combined with $f(\bm{y}) \geq f^{\star}$, implies
        \[
            \left \Vert \nabla f(\bm{x}) \right \Vert^{2} \leq 2L \left( f(\bm{x}) - f^{\star} \right).
        \]
        \qed
    \end{proof}

\section{Proof of Proposition~\ref{prop:es-valid}} \label{app:proof_prop_es_valid}
    \begin{proof}
        We follow the sampling-vector formalism of \cite{khaled2020better}.
        Let $\bm{v} = \left( v_{1}, \dots, v_{n} \right)^{\top}$ be a sampling vector satisfying $\mathbb{E}_{v_i}[v_i] = 1$ for every $i$. The stochastic gradient is given by
        \begin{equation}
            \nabla f_{\bm{v}}(\bm{x}) \coloneqq \frac{1}{n} \sum_{i=1}^{n} v_i \nabla f_i(\bm{x}).
            \label{eq:appC_stochastic_gradient}
        \end{equation}
        By expanding the squared norm in \eqref{eq:appC_stochastic_gradient}, we obtain
        \begin{align}
            \mathbb{E}_{\bm{v}} \left\| \nabla f_{\bm{v}}(\bm{x}) \right\|^2
            &= \frac{1}{n^2} \mathbb{E}_{\bm{v}}  \left\| \sum_{i=1}^{n} v_i \nabla f_i(\bm{x}) \right\|^2  \nonumber \\
            &= \frac{1}{n^2} \sum_{i=1}^{n} \mathbb{E}_{v_i}[v_i^2] \left\| \nabla f_i(\bm{x}) \right\|^2 \nonumber \\
            &\quad+ \frac{1}{n^2} \sum_{i\neq j} \mathbb{E}_{(v_i,v_j)}[v_i v_j] \left\langle \nabla f_i(\bm{x}), \nabla f_j(\bm{x}) \right\rangle.
            \label{eq:appC_general_expansion}
        \end{align}
        In the following, we write $\mathbb{E}_{\bm{v}}$ simply as $\mathbb{E}$.
    
        Before considering each sampling scheme, we establish a bound that will be used repeatedly. Since each $f_i$ is $L_i$-smooth and bounded below, applying the descent lemma applied to $\bm{y} = \bm{x} - \frac{1}{L_i} \nabla f_i(\bm{x})$ gives
        \begin{align}
            f_i\left( \bm{x} - \frac{1}{L_i} \nabla f_i(\bm{x}) \right)
            &\leq f_i(\bm{x}) - \frac{1}{L_i} \left\| \nabla f_i(\bm{x}) \right\|^2 + \frac{L_i}{2} \left\| \frac{1}{L_i} \nabla f_i(\bm{x}) \right\|^2 \nonumber \\
            &= f_i(\bm{x}) - \frac{1}{2L_i} \left\| \nabla f_i(\bm{x}) \right\|^2.
            \label{eq:appC_descent_component}
        \end{align}
        Since $f_i^{\inf} \leq f_i\left( \bm{x} - \frac{1}{L_i} \nabla f_i(\bm{x}) \right)$, it follows that
        \begin{equation}
            \left\| \nabla f_i(\bm{x}) \right\|^2 \leq 2L_i \left( f_i(\bm{x}) - f_i^{\inf} \right).
            \label{eq:appC_grad_bound}
        \end{equation}
        Moreover, by the definitions of $f$ and $\Delta^{\inf}$,
        \begin{align}
            \frac{1}{n} \sum_{i=1}^{n} \left( f_i(\bm{x}) - f_i^{\inf} \right)
            &= \frac{1}{n} \sum_{i=1}^{n} f_i(\bm{x}) - \frac{1}{n} \sum_{i=1}^{n} f_i^{\inf} \nonumber \\
            &= f(\bm{x}) - \frac{1}{n} \sum_{i=1}^{n} f_i^{\inf} \nonumber \\
            &= f(\bm{x}) - f^{\star} + \frac{1}{n} \sum_{i=1}^{n} \left( f^{\star} - f_i^{\inf} \right) \nonumber \\
            &= f(\bm{x}) - f^{\star} + \Delta^{\inf}.
            \label{eq:appC_gap_identity}
        \end{align}
    
        \subsection*{Case 1 (sampling with replacement)}
    
        Consider minibatch sampling with replacement of size $\tau$ and sampling probabilities $q_i$, where $\sum_{i=1}^{n} q_i = 1$. Let $S_i$ denote the number of times that index $i$ is selected in the minibatch. Then,
        \[
            S_i \sim \mathrm{Binomial}(\tau, q_i), \qquad v_i = \frac{S_i}{\tau q_i}.
        \]
        Since
        \[
            \mathbb{E}[S_i] = \tau q_i, \qquad \operatorname{Var}(S_i) = \tau q_i(1 - q_i),
        \]
        we have
        \begin{align}
            \mathbb{E}[S_i^2]
            &= \operatorname{Var}(S_i) + \mathbb{E}[S_i]^2 \nonumber \\
            &= \tau q_i(1 - q_i) + \tau^2 q_i^2.
            \label{eq:appC_case1_Si_second}
        \end{align}
        Therefore,
        \begin{align}
            \mathbb{E}[v_i^2]
            &= \frac{\mathbb{E}[S_i^2]}{\tau^2 q_i^2} \nonumber \\
            &= \frac{\tau q_i(1 - q_i) + \tau^2 q_i^2}{\tau^2 q_i^2} \nonumber \\
            &= 1 - \frac{1}{\tau} + \frac{1}{\tau q_i}.
            \label{eq:appC_case1_second}
        \end{align}
        For $i \neq j$, the covariance of $S_i$ and $S_j$ is $\operatorname{Cov}(S_i, S_j) = -\tau q_i q_j$, and hence
        \begin{align}
            \mathbb{E}[S_i S_j]
            &= \operatorname{Cov}(S_i, S_j) + \mathbb{E}[S_i] \mathbb{E}[S_j] \nonumber \\
            &= -\tau q_i q_j + \tau^2 q_i q_j \nonumber \\
            &= \tau(\tau - 1) q_i q_j.
            \label{eq:appC_case1_cross_S}
        \end{align}
        Thus,
        \begin{equation}
            \mathbb{E}[v_i v_j] = \frac{\mathbb{E}[S_i S_j]}{\tau^2 q_i q_j} = 1 - \frac{1}{\tau}, \qquad i \neq j.
            \label{eq:appC_case1_cross}
        \end{equation}
        Substituting \eqref{eq:appC_case1_second} and \eqref{eq:appC_case1_cross} into \eqref{eq:appC_general_expansion}, we obtain
        \begin{align}
            \mathbb{E} \left\| \nabla f_{\bm{v}}(\bm{x}) \right\|^2
            &= \frac{1}{n^2} \sum_{i=1}^{n} \left( 1 - \frac{1}{\tau} + \frac{1}{\tau q_i} \right) \left\| \nabla f_i(\bm{x}) \right\|^2 \nonumber \\
            &\quad+ \frac{1}{n^2} \sum_{i \neq j} \left( 1 - \frac{1}{\tau} \right) \left\langle \nabla f_i(\bm{x}), \nabla f_j(\bm{x}) \right\rangle \nonumber \\
            &= \left( 1 - \frac{1}{\tau} \right) \frac{1}{n^2} \left[ \sum_{i=1}^{n} \left\| \nabla f_i(\bm{x}) \right\|^2 + \sum_{i \neq j} \left\langle \nabla f_i(\bm{x}), \nabla f_j(\bm{x}) \right\rangle \right] \nonumber \\
            &\quad+ \frac{1}{\tau n^2} \sum_{i=1}^{n} \frac{1}{q_i} \left\| \nabla f_i(\bm{x}) \right\|^2 \nonumber \\
            &= \left( 1 - \frac{1}{\tau} \right) \left\| \nabla f(\bm{x}) \right\|^2 + \frac{1}{\tau n^2} \sum_{i=1}^{n} \frac{1}{q_i} \left\| \nabla f_i(\bm{x}) \right\|^2.
            \label{eq:appC_case1_expansion}
        \end{align}
        Here, we have used
        \begin{align}
            \sum_{i=1}^{n} \left\| \nabla f_i(\bm{x}) \right\|^2 + \sum_{i \neq j} \left\langle \nabla f_i(\bm{x}), \nabla f_j(\bm{x}) \right\rangle
            = \left\| \sum_{i=1}^{n} \nabla f_i(\bm{x}) \right\|^2 
            = n^2 \left\| \nabla f(\bm{x}) \right\|^2.
            \label{eq:appC_sum_identity}
        \end{align}
    
        Next, We bound the second term in \eqref{eq:appC_case1_expansion}. By \eqref{eq:appC_grad_bound},
        \begin{align}
            \frac{1}{\tau n^2} \sum_{i=1}^{n} \frac{1}{q_i} \left\| \nabla f_i(\bm{x}) \right\|^2
            &\leq \frac{2}{\tau n^2} \sum_{i=1}^{n} \frac{L_i}{q_i} \left( f_i(\bm{x}) - f_i^{\inf} \right) \nonumber \\
            &\leq \frac{2}{\tau n^2} \max_{1 \leq i \leq n} \frac{L_i}{q_i} \sum_{i=1}^{n} \left( f_i(\bm{x}) - f_i^{\inf} \right) \nonumber \\
            &= \frac{2}{\tau n} \max_{1 \leq i \leq n} \frac{L_i}{q_i} \left( f(\bm{x}) - f^{\star} + \Delta^{\inf} \right),
            \label{eq:appC_case1_bound}
        \end{align}
        where the last equality follows from \eqref{eq:appC_gap_identity}.
        Combining \eqref{eq:appC_case1_expansion} and \eqref{eq:appC_case1_bound}, we obtain
        \begin{align}
            \mathbb{E} \left\| \nabla f_{\bm{v}}(\bm{x}) \right\|^2
            &\leq 2 \left( \frac{1}{\tau n} \max_{1 \leq i \leq n} \frac{L_i}{q_i} \right) \left( f(\bm{x}) - f^{\star} \right) \nonumber \\
            &\quad+ \left( 1 - \frac{1}{\tau} \right) \left\| \nabla f(\bm{x}) \right\|^2 + 2 \left( \frac{1}{\tau n} \max_{1 \leq i \leq n} \frac{L_i}{q_i} \right) \Delta^{\inf}.
        \end{align}
        Therefore, Assumption~\ref{as:es} holds with
        \begin{align}
            A &= \frac{1}{\tau n} \max_{1 \leq i \leq n} \frac{L_i}{q_i}, &
            B &= 1 - \frac{1}{\tau}, &
            C &= 2A\Delta^{\inf}.
            \label{eq:appC_case1_ABC}
        \end{align}
    
        \subsection*{Case 2 (independent sampling without replacement)}
    
        Consider independent sampling without replacement, where each index $i$ is included in the sample with probability $p_i$. The corresponding sampling vector is
        \[
            v_i = \frac{\{i \in S\}}{p_i},
        \]
        where $\{i \in S\}$ denotes the indicator of the event that index $i$ is included in $S$. Since $\{i \in S\}$ is a Bernoulli random variable with parameter $p_i$,
        \begin{align}
            \mathbb{E}[v_i^2]
            &= \frac{1}{p_i^2} \mathbb{E}[\{i \in S\}^2] \nonumber \\
            &= \frac{1}{p_i^2} \mathbb{E}[\{i \in S\}] = \frac{1}{p_i}.
            \label{eq:appC_case2_second}
        \end{align}
        Moreover, under the independent sampling assumption, for $i \neq j$,
        \begin{align}
            \mathbb{E}[v_i v_j]
            &= \frac{1}{p_i p_j} \mathbb{E}[\{i \in S\}\{j \in S\}] \nonumber \\
            &= \frac{1}{p_i p_j} \mathbb{E}[\{i \in S\}] \mathbb{E}[\{j \in S\}] = 1.
            \label{eq:appC_case2_cross}
        \end{align}
        Substituting \eqref{eq:appC_case2_second} and \eqref{eq:appC_case2_cross} into \eqref{eq:appC_general_expansion}, we obtain
        \begin{align}
            \mathbb{E} \left\| \nabla f_{\bm{v}}(\bm{x}) \right\|^2
            &= \frac{1}{n^2} \sum_{i=1}^{n} \frac{1}{p_i} \left\| \nabla f_i(\bm{x}) \right\|^2 + \frac{1}{n^2} \sum_{i \neq j} \left\langle \nabla f_i(\bm{x}), \nabla f_j(\bm{x}) \right\rangle \nonumber \\
            &= \frac{1}{n^2} \left[ \sum_{i=1}^{n} \left\| \nabla f_i(\bm{x}) \right\|^2 + \sum_{i \neq j} \left\langle \nabla f_i(\bm{x}), \nabla f_j(\bm{x}) \right\rangle \right] \nonumber \\
            &\quad+ \frac{1}{n^2} \sum_{i=1}^{n} \left( \frac{1}{p_i} - 1 \right) \left\| \nabla f_i(\bm{x}) \right\|^2 \nonumber \\
            &= \left\| \nabla f(\bm{x}) \right\|^2 + \frac{1}{n^2} \sum_{i=1}^{n} \left( \frac{1}{p_i} - 1 \right) \left\| \nabla f_i(\bm{x}) \right\|^2.
            \label{eq:appC_case2_expansion}
        \end{align}
        Using \eqref{eq:appC_grad_bound}, we have
        \begin{align}
            \frac{1}{n^2} \sum_{i=1}^{n} \left( \frac{1}{p_i} - 1 \right) \left\| \nabla f_i(\bm{x}) \right\|^2
            &\leq \frac{2}{n^2} \sum_{i=1}^{n} \frac{(1 - p_i) L_i}{p_i} \left( f_i(\bm{x}) - f_i^{\inf} \right) \nonumber \\
            &\leq \frac{2}{n^2} \max_{1 \leq i \leq n} \frac{(1 - p_i) L_i}{p_i} \sum_{i=1}^{n} \left( f_i(\bm{x}) - f_i^{\inf} \right) \nonumber \\
            &= \frac{2}{n} \max_{1 \leq i \leq n} \frac{(1 - p_i) L_i}{p_i} \left( f(\bm{x}) - f^{\star} + \Delta^{\inf} \right).
            \label{eq:appC_case2_bound}
        \end{align}
        Combining \eqref{eq:appC_case2_expansion} and \eqref{eq:appC_case2_bound}, we obtain
        \begin{align}
            \mathbb{E} \left\| \nabla f_{\bm{v}}(\bm{x}) \right\|^2
            &\leq 2 \left( \frac{1}{n} \max_{1 \leq i \leq n} \frac{(1 - p_i) L_i}{p_i} \right) \left( f(\bm{x}) - f^{\star} \right) \nonumber \\
            &\quad+ \left\| \nabla f(\bm{x}) \right\|^2 + 2 \left( \frac{1}{n} \max_{1 \leq i \leq n} \frac{(1 - p_i) L_i}{p_i} \right) \Delta^{\inf}.
        \end{align}
        Hence, Assumption~\ref{as:es} holds with
        \begin{align}
            A &= \frac{1}{n} \max_{1 \leq i \leq n} \frac{(1 - p_i) L_i}{p_i}, &
            B &= 1, &
            C &= 2A\Delta^{\inf}.
            \label{eq:appC_case2_ABC}
        \end{align}
    
        \subsection*{Case 3 (\texorpdfstring{$\tau$}{tau}-Nice sampling without replacement)}
    
        Finally, consider uniform sampling without replacement of size $\tau$. For $n \geq 2$, let $S$ be sampled uniformly from all subsets of $\{1, \ldots, n\}$ having cardinality $\tau$, and define
        \[
            v_i = \frac{n}{\tau} \{i \in S\}.
        \]
        Since each index is included with probability $\tau / n$,
        \begin{align}
            \mathbb{E}[v_i^2]
            &= \frac{n^2}{\tau^2} \mathbb{E}[\{i \in S\}] = \frac{n^2}{\tau^2} \frac{\tau}{n} = \frac{n}{\tau}.
            \label{eq:appC_case3_second}
        \end{align}
        For $i \neq j$, the probability that both indices are included in a uniformly sampled subset of size $\tau$ is
        \[
            \mathbb{P}(i, j \in S) = \frac{\tau(\tau - 1)}{n(n - 1)}.
        \]
        Therefore,
        \begin{align}
            \mathbb{E}[v_i v_j]
            &= \frac{n^2}{\tau^2} \mathbb{P}(i, j \in S) \nonumber \\
            &= \frac{n^2}{\tau^2} \frac{\tau(\tau - 1)}{n(n - 1)} = \frac{n(\tau - 1)}{\tau(n - 1)}, \qquad i \neq j.
            \label{eq:appC_case3_cross}
        \end{align}
        Substituting \eqref{eq:appC_case3_second} and \eqref{eq:appC_case3_cross} into \eqref{eq:appC_general_expansion}, we obtain
        \begin{align}
            \mathbb{E} \left\| \nabla f_{\bm{v}}(\bm{x}) \right\|^2
            &= \frac{1}{n^2} \frac{n}{\tau} \sum_{i=1}^{n} \left\| \nabla f_i(\bm{x}) \right\|^2 \nonumber \\
            &\quad+ \frac{1}{n^2} \frac{n(\tau - 1)}{\tau(n - 1)} \sum_{i \neq j} \left\langle \nabla f_i(\bm{x}), \nabla f_j(\bm{x}) \right\rangle.
            \label{eq:appC_case3_initial}
        \end{align}
        Let
        \[
            c \coloneqq \frac{n(\tau - 1)}{\tau(n - 1)}.
        \]
        Then, using
        \[
            \sum_{i \neq j} \left\langle \nabla f_i(\bm{x}), \nabla f_j(\bm{x}) \right\rangle = n^2 \left\| \nabla f(\bm{x}) \right\|^2 - \sum_{i=1}^{n} \left\| \nabla f_i(\bm{x}) \right\|^2,
        \]
        equation \eqref{eq:appC_case3_initial} becomes
        \begin{align}
            \mathbb{E} \left\| \nabla f_{\bm{v}}(\bm{x}) \right\|^2
            &= c \left\| \nabla f(\bm{x}) \right\|^2 + \left( \frac{1}{n\tau} - \frac{c}{n^2} \right) \sum_{i=1}^{n} \left\| \nabla f_i(\bm{x}) \right\|^2.
            \label{eq:appC_case3_rearranged}
        \end{align}
        The coefficient of the component-gradient term can be simplified as
        \begin{align}
            \frac{1}{n\tau} - \frac{c}{n^2}
            &= \frac{1}{n\tau} - \frac{\tau - 1}{n\tau(n - 1)} \nonumber \\
            &= \frac{n - 1 - (\tau - 1)}{n\tau(n - 1)} \nonumber \\
            &= \frac{n - \tau}{n\tau(n - 1)}.
            \label{eq:appC_case3_coefficient}
        \end{align}
        Hence,
        \begin{align}
            \mathbb{E} \left\| \nabla f_{\bm{v}}(\bm{x}) \right\|^2
            &= \frac{n(\tau - 1)}{\tau(n - 1)} \left\| \nabla f(\bm{x}) \right\|^2 \nonumber \\
            &\quad+ \frac{n - \tau}{\tau(n - 1)} \frac{1}{n} \sum_{i=1}^{n} \left\| \nabla f_i(\bm{x}) \right\|^2.
            \label{eq:appC_case3_expansion}
        \end{align}
        We now bound the second term in \eqref{eq:appC_case3_expansion}. By \eqref{eq:appC_grad_bound},
        \begin{align}
            \frac{n - \tau}{\tau(n - 1)} \frac{1}{n} \sum_{i=1}^{n} \left\| \nabla f_i(\bm{x}) \right\|^2
            &\leq \frac{2(n - \tau)}{\tau(n - 1)} \frac{1}{n} \sum_{i=1}^{n} L_i \left( f_i(\bm{x}) - f_i^{\inf} \right) \nonumber \\
            &\leq \frac{2(n - \tau)}{\tau(n - 1)} \max_{1 \leq i \leq n} L_i \frac{1}{n} \sum_{i=1}^{n} \left( f_i(\bm{x}) - f_i^{\inf} \right) \nonumber \\
            &= \frac{2(n - \tau)}{\tau(n - 1)} \max_{1 \leq i \leq n} L_i \left( f(\bm{x}) - f^{\star} + \Delta^{\inf} \right).
            \label{eq:appC_case3_bound}
        \end{align}
        Combining \eqref{eq:appC_case3_expansion} and \eqref{eq:appC_case3_bound}, we obtain
        \begin{align}
            \mathbb{E} \left\| \nabla f_{\bm{v}}(\bm{x}) \right\|^2
            &\leq 2 \left( \frac{n - \tau}{\tau(n - 1)} \max_{1 \leq i \leq n} L_i \right) \left( f(\bm{x}) - f^{\star} \right) \nonumber \\
            &\quad+ \frac{n(\tau - 1)}{\tau(n - 1)} \left\| \nabla f(\bm{x}) \right\|^2 \nonumber \\
            &\quad+ 2 \left( \frac{n - \tau}{\tau(n - 1)} \max_{1 \leq i \leq n} L_i \right) \Delta^{\inf}.
        \end{align}
        Therefore, for $n \geq 2$, Assumption~\ref{as:es} holds with
        \begin{align}
            A &= \frac{n - \tau}{\tau(n - 1)} \max_{1 \leq i \leq n} L_i, &
            B &= \frac{n(\tau - 1)}{\tau(n - 1)}, &
            C &= 2A\Delta^{\inf}.
            \label{eq:appC_case3_ABC}
        \end{align}
    
        For completeness, consider the degenerate case $n = 1$. Then, $\tau = 1$ necessarily, and the sampling is deterministic. Since $v_1 = 1$, we have
        \begin{equation}
            \nabla f_{\bm{v}}(\bm{x}) = \nabla f(\bm{x}).
            \label{eq:appC_n1}
        \end{equation}
        Hence,
        \begin{equation}
            \mathbb{E} \left\| \nabla f_{\bm{v}}(\bm{x}) \right\|^2 = \left\| \nabla f(\bm{x}) \right\|^2,
        \end{equation}
        and Assumption~\ref{as:es} holds trivially with
        \begin{equation}
            (A, B, C) = (0, 1, 0).
        \end{equation}
        The formula in Case 3 is stated only for $n \geq 2$ because it contains $n - 1$ in the denominator.
        This completes the proof of Proposition~\ref{prop:es-valid}.
        \qed
    \end{proof}

\section{Proof of Corollary~\ref{cor:general_rate}} \label{app:proof_cor_general_rate}
\begin{proof}
        Let
    \begin{equation}
        D_k \coloneqq \mathbb{E} \left[ f(\bm{x}_k) - f^{\star} \right], \qquad G_k \coloneqq \mathbb{E} \left\| \nabla f(\bm{x}_k) \right\|^2.
        \label{eq:proof_cor_DG}
    \end{equation}
    By taking the total expectation in the one-step inequality \eqref{eq:onestep_descent}, we obtain
    \begin{equation}
        \begin{aligned}
            D_{k+1} &\leq D_k - \eta_k \left( 1 - \frac{LB\eta_k}{2} \right) G_k + LA\eta_k^2 D_k + \frac{LC}{2}\eta_k^2.
        \end{aligned}
        \label{eq:cor_one_step}
    \end{equation}
    By the step-size condition in Theorem~\ref{thm:1}, we have
    \begin{equation}
        2 - LB\eta_k \geq 2 - LB\eta_{\max} > 0,
        \label{eq:cor_stepsize_lower}
    \end{equation}
    and hence,
    \begin{equation}
        1 - \frac{LB\eta_k}{2} \geq \frac{2 - LB\eta_{\max}}{2}.
        \label{eq:cor_descent_coeff}
    \end{equation}
    Therefore,
    \begin{equation}
        \begin{aligned}
            \frac{2 - LB\eta_{\max}}{2} \eta_k G_k &\leq D_k - D_{k+1} + LA\eta_k^2 D_k + \frac{LC}{2}\eta_k^2.
        \end{aligned}
        \label{eq:cor_rearranged}
    \end{equation}
    Summing \eqref{eq:cor_rearranged} from $k=0$ to $K-1$ gives
    \begin{equation}
        \begin{aligned}
            \frac{2 - LB\eta_{\max}}{2} \sum_{k=0}^{K-1} \eta_k G_k &\leq D_0 - D_K + LA\sum_{k=0}^{K-1} \eta_k^2 D_k + \frac{LC}{2}\sum_{k=0}^{K-1} \eta_k^2.
        \end{aligned}
        \label{eq:cor_summed}
    \end{equation}
    Since $f^{\star}$ is the infimum of $f$, we have
    \begin{equation}
        D_K = \mathbb{E} \left[ f(\bm{x}_K) - f^{\star} \right] \geq 0.
        \label{eq:DK_nonnegative}
    \end{equation}
    Moreover, by the definition of $M_K$, we have
    \begin{equation}
        D_k \leq M_K \qquad (0 \leq k \leq K-1),
        \label{eq:Dk_MK}
    \end{equation}
    and consequently
    \begin{equation}
        \sum_{k=0}^{K-1} \eta_k^2 D_k \leq M_K \sum_{k=0}^{K-1} \eta_k^2.
        \label{eq:sum_Dk_bound}
    \end{equation}
    Using \eqref{eq:DK_nonnegative} and \eqref{eq:sum_Dk_bound} in \eqref{eq:cor_summed}, we obtain
    \begin{equation}
        \begin{aligned}
            \frac{2 - LB\eta_{\max}}{2} \sum_{k=0}^{K-1} \eta_k G_k &\leq D_0 + \left( LA M_K + \frac{LC}{2} \right) \sum_{k=0}^{K-1} \eta_k^2.
        \end{aligned}
        \label{eq:cor_weighted_bound}
    \end{equation}
    Since $\eta_k > 0$,
    \begin{equation}
        \begin{aligned}
            \min_{0 \leq k \leq K-1} G_k \sum_{k=0}^{K-1} \eta_k &\leq \sum_{k=0}^{K-1} \eta_k G_k.
        \end{aligned}
        \label{eq:weighted_min}
    \end{equation}
    Combining \eqref{eq:cor_weighted_bound} and \eqref{eq:weighted_min}, and dividing by $\sum_{k=0}^{K-1} \eta_k$, yields
    \begin{equation}
        \begin{aligned}
            \min_{0 \leq k \leq K-1} G_k 
            &\leq \frac{2D_0}{\left( 2 - LB\eta_{\max} \right) \sum_{k=0}^{K-1} \eta_k} \\
            &\quad + \frac{2LA M_K + LC}{2 - LB\eta_{\max}} \frac{\sum_{k=0}^{K-1} \eta_k^2}{\sum_{k=0}^{K-1} \eta_k}.
        \end{aligned}
        \label{eq:cor_final_bound}
    \end{equation}
    Substituting the definitions of $D_0$, $S_1(K)$, and $S_2(K)$ completes the proof.
    \qed
\end{proof}

\section{Proof of Proposition~\ref{prop:function-gap-bounded}} \label{app:proof_function_gap_bounded}
\begin{proof}
    By dropping the nonpositive gradient term from \eqref{eq:cor_one_step}, we obtain
    \begin{equation}
        D_{k+1} \leq \left( 1 + LA\eta_k^2 \right) D_k + \frac{LC}{2}\eta_k^2.
        \label{eq:gap_recursion}
    \end{equation}
    Since
    \begin{equation}
        1 + t \leq e^t \qquad (t \geq 0),
        \label{eq:exp_inequality}
    \end{equation}
    we have
    \begin{equation}
        1 + LA\eta_k^2 \leq \exp(LA\eta_k^2).
        \label{eq:factor_exp}
    \end{equation}
    Iterating \eqref{eq:gap_recursion} yields
    \begin{equation}
        \begin{aligned}
            D_k &\leq D_0 \prod_{j=0}^{k-1} \left( 1 + LA\eta_j^2 \right) \\
            &\quad + \frac{LC}{2} \sum_{i=0}^{k-1} \eta_i^2 \prod_{j=i+1}^{k-1} \left( 1 + LA\eta_j^2 \right).
        \end{aligned}
        \label{eq:gap_iterated}
    \end{equation}
    Using \eqref{eq:factor_exp},
    \begin{equation}
        \begin{aligned}
            \prod_{j=0}^{k-1} \left( 1 + LA\eta_j^2 \right) 
            &\leq \exp\left( LA\sum_{j=0}^{k-1} \eta_j^2 \right) \\
            &\leq \exp(LAQ).
        \end{aligned}
        \label{eq:first_product_bound}
    \end{equation}
    Similarly, for every $0 \leq i \leq k-1$,
    \begin{equation}
        \begin{aligned}
            \prod_{j=i+1}^{k-1} \left( 1 + LA\eta_j^2 \right) 
            &\leq \exp\left( LA\sum_{j=i+1}^{k-1} \eta_j^2 \right) \\
            &\leq \exp(LAQ).
        \end{aligned}
        \label{eq:second_product_bound}
    \end{equation}
    Substituting \eqref{eq:first_product_bound} and \eqref{eq:second_product_bound} into \eqref{eq:gap_iterated}, we obtain
    \begin{equation}
        \begin{aligned}
            D_k &\leq \exp(LAQ)D_0 + \frac{LC}{2} \exp(LAQ) \sum_{i=0}^{k-1} \eta_i^2 \\
            &\leq \exp(LAQ) \left( D_0 + \frac{LC}{2}Q \right) = M.
        \end{aligned}
        \label{eq:Dk_uniform_bound}
    \end{equation}
    Therefore,
    \begin{equation}
        D_k \leq M \qquad \text{for all } k \geq 0,
    \end{equation}
    which immediately implies
    \begin{equation}
        M_K = \max_{0 \leq k \leq K-1} D_k \leq M.
    \end{equation}
    This completes the proof.
    \qed
    \end{proof}

\nocite{*}
\bibliographystyle{spmpsci}
\bibliography{references}

\begin{thebibliography}{10}
\providecommand{\url}[1]{{#1}}
\providecommand{\urlprefix}{URL }
\expandafter\ifx\csname urlstyle\endcsname\relax
  \providecommand{\doi}[1]{DOI~\discretionary{}{}{}#1}\else
  \providecommand{\doi}{DOI~\discretionary{}{}{}\begingroup \urlstyle{rm}\Url}\fi

\bibitem{allen2019sgd}
Allen-Zhu, Z., Li, Y.: On the convergence rate of stochastic gradient descent with momentum for nonconvex optimization.
\newblock arXiv preprint arXiv:1905.09997  (2019)

\bibitem{bottou2018optimization}
Bottou, L., Curtis, F.E., Nocedal, J.: Optimization methods for large-scale machine learning.
\newblock SIAM Review \textbf{60}(2), 223--311 (2018)

\bibitem{chen2020simple}
Chen, T., Kornblith, S., Norouzi, M., Hinton, G.: A simple framework for contrastive learning of visual representations.
\newblock In: International Conference on Machine Learning (ICML) (2020)

\bibitem{cubuk2019autoaugment}
Cubuk, E.D., Zoph, B., Mane, D., Vasudevan, V., Le, Q.V.: {AutoAugment}: Learning augmentation policies from data.
\newblock In: Proceedings of the IEEE/CVF Conference on Computer Vision and Pattern Recognition (CVPR) (2019)

\bibitem{dean2012large}
Dean, J., Corrado, G., Monga, R., Chen, K., Devin, M., Mao, M., Ranzato, M., Senior, A., Tucker, P., Yang, K., Le, Q.V., Ng, A.Y.: Large scale distributed deep networks.
\newblock In: Advances in Neural Information Processing Systems (NeurIPS), pp. 1223--1231 (2012)

\bibitem{defazio2014saga}
Defazio, A., Bach, F., Lacoste-Julien, S.: {SAGA}: A fast incremental gradient method with support for non-strongly convex composite objectives.
\newblock In: Advances in Neural Information Processing Systems (NeurIPS), vol.~27 (2014)

\bibitem{deng2009imagenet}
Deng, J., Dong, W., Socher, R., Li, L.J., Li, K., Fei-Fei, L.: {ImageNet}: A large-scale hierarchical image database.
\newblock In: Proceedings of the IEEE Conference on Computer Vision and Pattern Recognition (CVPR), pp. 248--255 (2009)

\bibitem{devries2017cutout}
DeVries, T., Taylor, G.W.: Improved regularization of convolutional neural networks with {Cutout}.
\newblock arXiv preprint arXiv:1708.04552  (2017)

\bibitem{dinh2017sharp}
Dinh, L., Pascanu, R., Bengio, S., Bengio, Y.: Sharp minima can generalize for deep nets.
\newblock In: International Conference on Machine Learning (ICML), pp. 1019--1028. PMLR (2017)

\bibitem{duchi2011adaptive}
Duchi, J., Hazan, E., Singer, Y.: Adaptive subgradient methods for online learning and stochastic optimization.
\newblock Journal of Machine Learning Research \textbf{12}(7), 2121--2159 (2011)

\bibitem{glorot2010understanding}
Glorot, X., Bengio, Y.: Understanding the difficulty of training deep feedforward neural networks.
\newblock In: Proceedings of the thirteenth international conference on artificial intelligence and statistics (AISTATS), pp. 249--256. JMLR Workshop and Conference Proceedings (2010)

\bibitem{gorbunov2023unified}
Gorbunov, E., Stich, S.U., Richt{\'a}rik, P.: Unified theory of stochastic first-order methods for non-convex optimization.
\newblock In: The Eleventh International Conference on Learning Representations (ICLR) (2023)

\bibitem{gower2019sgd}
Gower, R., Loizou, N., Qian, X., Sailanbayev, A., Shulgin, E., Richt{\'a}rik, P.: {SGD}: General analysis and improved rates.
\newblock In: Proceedings of the 36th International Conference on Machine Learning (ICML), pp. 2322--2331. PMLR (2019)

\bibitem{goyal2017accurate}
Goyal, P., Doll{\'a}r, P., Girshick, R., Noordhuis, P., Wesolowski, L., Kyrola, A., Tulloch, A., Jia, Y., He, K.: Accurate, large minibatch {SGD}: Training {ImageNet} in 1 hour.
\newblock In: Proceedings of the IEEE Conference on Computer Vision and Pattern Recognition (CVPR), pp. 3290--3298 (2017)

\bibitem{hardt2016train}
Hardt, M., Recht, B., Singer, Y.: Train faster, generalize better: Stability of stochastic gradient descent.
\newblock In: International conference on machine learning (ICML), pp. 1225--1234. PMLR (2016)

\bibitem{he2015delving}
He, K., Zhang, X., Ren, S., Sun, J.: Delving deep into rectifiers: Surpassing human-level performance on {ImageNet} classification.
\newblock In: Proceedings of the IEEE international conference on computer vision (ICCV), pp. 1026--1034 (2015)

\bibitem{he2016deep}
He, K., Zhang, X., Ren, S., Sun, J.: Deep residual learning for image recognition.
\newblock In: Proceedings of the IEEE conference on computer vision and pattern recognition (CVPR), pp. 770--778 (2016)

\bibitem{hinton2012rmsprop}
Hinton, G.: {RMSProp} optimization algorithm.
\newblock Lecture slides from "Neural Networks for Machine Learning" (Coursera) (2012).
\newblock URL: \url{https://www.cs.toronto.edu/~tijmen/csc321/slides/lecture_slides_lec6.pdf}

\bibitem{hoffer2017trainlonger}
Hoffer, E., Hubara, I., Soudry, D.: Train longer, generalize better: closing the generalization gap in large batch training of neural networks.
\newblock In: Advances in Neural Information Processing Systems (NeurIPS) (2017)

\bibitem{hornik1989multilayer}
Hornik, K., Stinchcombe, M., White, H.: Multilayer feedforward networks are universal approximators.
\newblock Neural networks \textbf{2}(5), 359--366 (1989)

\bibitem{ioffe2015batch}
Ioffe, S., Szegedy, C.: Batch normalization: Accelerating deep network training by reducing internal covariate shift.
\newblock International Conference on Machine Learning (ICML) pp. 448--456 (2015)

\bibitem{johnson2013accelerating}
Johnson, R., Zhang, T.: Accelerating stochastic gradient descent using predictive variance reduction.
\newblock In: Advances in Neural Information Processing Systems (NeurIPS), vol.~26 (2013)

\bibitem{karimi2016linear}
Karimi, H., Nutini, J., Schmidt, M.: Linear convergence of gradient and proximal-gradient methods under the {Polyak}-{\l}ojasiewicz condition.
\newblock In: European conference on machine learning and principles and practice of knowledge discovery in databases (ECML-PKDD), pp. 795--811. Springer (2016)

\bibitem{karimireddy2020mime}
Karimireddy, S.P., Kale, S., Mohri, M., Reddi, S., Stich, S., Suresh, A.T.: {Mime}: Mimicking centralized stochastic algorithms in federated learning.
\newblock In: International conference on machine learning (ICML), pp. 5257--5267. PMLR (2020)

\bibitem{keskar2016largebatch}
Keskar, N.S., Mudigere, D., Nocedal, J., Smelyanskiy, M., Tang, P.: On large-batch training for deep learning: Generalization gap and sharp minima.
\newblock arXiv preprint arXiv:1609.04836  (2016)

\bibitem{keskar2017on}
Keskar, N.S., Mudigere, D., Nocedal, J., Smelyanskiy, M., Tang, P.T.P.: On large-batch training for deep learning: Generalization gap and sharp minima.
\newblock In: International Conference on Learning Representations (ICLR) (2017)

\bibitem{khaled2020better}
Khaled, A., Richt{\'a}rik, P.: Better theory for {SGD} in the nonconvex world.
\newblock arXiv preprint arXiv:2002.03329  (2020)

\bibitem{kiefer1952stochastic}
Kiefer, J., Wolfowitz, J.: Stochastic estimation of the maximum of a regression function.
\newblock Annals of Mathematical Statistics \textbf{23}(3), 462--466 (1952)

\bibitem{kingma2015adam}
Kingma, D.P., Ba, J.: {Adam}: A method for stochastic optimization.
\newblock In: International Conference on Learning Representations (ICLR) (2015)

\bibitem{krizhevsky2009learning}
Krizhevsky, A.: Learning multiple layers of features from tiny images.
\newblock Tech. rep., University of Toronto (2009)

\bibitem{krizhevsky2012imagenet}
Krizhevsky, A., Sutskever, I., Hinton, G.E.: {ImageNet} classification with deep convolutional neural networks.
\newblock In: Advances in Neural Information Processing Systems (NeurIPS), pp. 1097--1105 (2012)

\bibitem{lecun1998mnist}
LeCun, Y., Bottou, L., Bengio, Y., Haffner, P.: Gradient-based learning applied to document recognition.
\newblock Proceedings of the IEEE \textbf{86}(11), 2278--2324 (1998)

\bibitem{li2014scaling}
Li, M., Andersen, D.G., Park, J.W., Smola, A.J., Ahmed, A., Josifovski, V., Long, J., Shekita, E.J., Su, B.Y.: Scaling distributed machine learning with the parameter server.
\newblock In: Proceedings of the USENIX Symposium on Operating Systems Design and Implementation (OSDI) (2014)

\bibitem{lin2014microsoft}
Lin, T.Y., Maire, M., Belongie, S., Hays, J., Perona, P., Ramanan, D., Doll{\'a}r, P., Zitnick, C.L.: {Microsoft} {COCO}: Common objects in context.
\newblock European conference on computer vision (ECCV) pp. 740--755 (2014)

\bibitem{loshchilov2017sgdr}
Loshchilov, I., Hutter, F.: {SGDR}: Stochastic gradient descent with warm restarts.
\newblock In: International Conference on Learning Representations (ICLR) Workshop / arXiv:1608.03983 (2017)

\bibitem{mcmahan2017communication}
McMahan, B., Moore, E., Ramage, D., Hampson, S., y~Arcas, B.A.: Communication-efficient learning of deep networks from decentralized data.
\newblock In: Artificial intelligence and statistics (AISTATS), pp. 1273--1282. PMLR (2017)

\bibitem{muller2019does}
M{\"u}ller, R., Kornblith, S., Hinton, G.: When does label smoothing help?
\newblock In: Advances in Neural Information Processing Systems (NeurIPS) (2019)

\bibitem{nemirovski2009robust}
Nemirovski, A., Juditsky, A., Lan, G., Shapiro, A.: Robust stochastic approximation approach to stochastic programming.
\newblock SIAM Journal on Optimization \textbf{19}(4), 1574--1609 (2009)

\bibitem{nesterov1983method}
Nesterov, Y.: A method for solving the convex programming problem with convergence rate {O}(1/k$^2$).
\newblock Soviet Mathematics Doklady \textbf{27}(2), 372--376 (1983)

\bibitem{neyshabur2015path}
Neyshabur, B., Tomioka, R., Srebro, N.: Path-{SGD}: Path-normalized optimization in deep neural networks.
\newblock In: Advances in neural information processing systems (NeurIPS), vol.~28 (2015)

\bibitem{polyak1964some}
Polyak, B.T.: Some methods of speeding up the convergence of iteration methods.
\newblock USSR Computational Mathematics and Mathematical Physics \textbf{4}(5), 1--17 (1964)

\bibitem{polyak1992acceleration}
Polyak, B.T., Juditsky, A.B.: Acceleration of stochastic approximation by averaging.
\newblock SIAM Journal on Control and Optimization \textbf{30}(4), 838--855 (1992)

\bibitem{reddi2018on}
Reddi, S.J., Kale, S., Kumar, S.: On the convergence of {Adam} and beyond.
\newblock In: International Conference on Learning Representations (ICLR) (2018)

\bibitem{robbins1951stochastic}
Robbins, H., Monro, S.: A stochastic approximation method.
\newblock Annals of Mathematical Statistics \textbf{22}(3), 400--407 (1951)

\bibitem{ruder2016overview}
Ruder, S.: An overview of gradient descent optimization algorithms.
\newblock arXiv preprint arXiv:1609.04747  (2016)

\bibitem{rumelhart1986learning}
Rumelhart, D.E., Hinton, G.E., Williams, R.J.: Learning representations by back-propagating errors.
\newblock Nature \textbf{323}(6088), 533--536 (1986)

\bibitem{schmidt2013fast}
Schmidt, M., Le~Roux, N., Bach, F.: Fast convergence of stochastic gradient descent under a strong growth condition.
\newblock In: Proceedings of the International Conference on Artificial Intelligence and Statistics (AISTATS), pp. 768--776 (2013)

\bibitem{schmidt2017minimizing}
Schmidt, M., Le~Roux, N., Bach, F.: Minimizing finite sums with the stochastic average gradient.
\newblock Mathematical Programming \textbf{162}, 83--112 (2017)

\bibitem{simonyan2014very}
Simonyan, K., Zisserman, A.: Very deep convolutional networks for large-scale image recognition.
\newblock arXiv preprint arXiv:1409.1556  (2014).
\newblock Published as ICLR 2015 paper

\bibitem{smith2015cyclical}
Smith, L.N.: Cyclical learning rates for training neural networks.
\newblock arXiv preprint arXiv:1506.01186  (2015).
\newblock See also WACV 2017 version

\bibitem{smith2017superconvergence}
Smith, L.N., Topin, N.: Super-convergence: Very fast training of neural networks using large learning rates.
\newblock arXiv preprint arXiv:1708.07120  (2017)

\bibitem{smith2017dontdecay}
Smith, S., Le, P.: Don't decay the learning rate, increase the batch size.
\newblock arXiv preprint arXiv:1711.00489  (2017)

\bibitem{srivastava2014dropout}
Srivastava, N., Hinton, G., Krizhevsky, A., Sutskever, I., Salakhutdinov, R.: {Dropout}: a simple way to prevent neural networks from overfitting.
\newblock The journal of machine learning research \textbf{15}(1), 1929--1958 (2014)

\bibitem{stich2019local}
Stich, S.U.: Local {SGD} converges fast and communicates little.
\newblock In: International Conference on Learning Representations (ICLR) (2019)

\bibitem{sutskever2013importance}
Sutskever, I., Martens, J., Dahl, G., Hinton, G.: On the importance of initialization and momentum in deep learning.
\newblock In: International conference on machine learning (ICML), pp. 1139--1147. PMLR (2013)

\bibitem{szegedy2014going}
Szegedy, C., Liu, W., Jia, Y., Sermanet, P., Reed, S., Anguelov, D., Erhan, D., Vanhoucke, V., Rabinovich, A.: Going deeper with convolutions.
\newblock In: Proceedings of the IEEE Conference on Computer Vision and Pattern Recognition (CVPR) (2015).
\newblock Original arXiv:1409.4842

\bibitem{szegedy2016rethinking}
Szegedy, C., Vanhoucke, V., Ioffe, S., Shlens, J., Wojna, Z.: Rethinking the {Inception} architecture for computer vision.
\newblock In: Proceedings of the IEEE Conference on Computer Vision and Pattern Recognition (CVPR), pp. 2818--2826 (2016)

\bibitem{torralba2008tiny}
Torralba, A., Fergus, R., Freeman, W.T.: 80 million tiny images: A large dataset for non-parametric object and scene recognition.
\newblock Tech. rep., MIT-CSAIL (2008)

\bibitem{umeda2024accelerates}
Umeda, H., Iiduka, H.: Increasing both batch size and learning rate accelerates stochastic gradient descent.
\newblock Transactions on Machine Learning Research  (2025)

\bibitem{vaswani2017attention}
Vaswani, A., Shazeer, N., Parmar, N., Uszkoreit, J., Jones, L., Gomez, A.N., Kaiser, {\L}., Polosukhin, I.: Attention is all you need.
\newblock In: Advances in Neural Information Processing Systems (NeurIPS) (2017)

\bibitem{vaswani2019fast}
Vaswani, S., Bach, F., Schmidt, M.: Fast and faster convergence of {SGD} for over-parameterized models and an accelerated variant.
\newblock In: Advances in Neural Information Processing Systems (NeurIPS), vol.~32 (2019)

\bibitem{you2017lars}
You, Y., Gitman, I., Ginsburg, B.: Large batch training of convolutional networks.
\newblock arXiv preprint arXiv:1708.03888  (2017).
\newblock {Layer-wise Adaptive Rate Scaling (LARS)}

\bibitem{yun2019cutmix}
Yun, S., Han, D., Oh, S.J., Chun, S., Choe, J., Yoo, Y.: {CutMix}: Regularization strategy to train strong classifiers with localizable features.
\newblock In: Proceedings of the IEEE International Conference on Computer Vision (ICCV) (2019)

\bibitem{zagoruyko2016wrn}
Zagoruyko, S., Komodakis, N.: Wide residual networks.
\newblock arXiv preprint arXiv:1605.07146  (2016)

\bibitem{zagoruyko2016wide}
Zagoruyko, S., Komodakis, N.: Wide residual networks.
\newblock In: British Machine Vision Conference (BMVC) (2016)

\bibitem{zeiler2012adadelta}
Zeiler, M.D.: {ADADELTA}: an adaptive learning rate method.
\newblock arXiv preprint arXiv:1212.5701  (2012)

\bibitem{zhang2017mixup}
Zhang, H., Cisse, M., Dauphin, Y.N., Lopez-Paz, D.: mixup: Beyond empirical risk minimization.
\newblock arXiv preprint arXiv:1710.09412  (2017)

\bibitem{zhang2018shufflenet}
Zhang, X., Zhou, X., Lin, M., Sun, J.: {ShuffleNet}: An extremely efficient convolutional neural network for mobile devices.
\newblock In: Proceedings of the IEEE Conference on Computer Vision and Pattern Recognition (CVPR) (2018)

\bibitem{zhong2020randomerasing}
Zhong, Z., Zheng, L., Kang, G., Li, S., Yang, Y.: Random erasing data augmentation.
\newblock In: Proceedings of the AAAI Conference on Artificial Intelligence (2020)

\bibitem{zou2005regularization}
Zou, H., Hastie, T.: Regularization and variable selection via the elastic net.
\newblock Journal of the Royal Statistical Society: Series B (Statistical Methodology) \textbf{67}(2), 301--320 (2005)

\end{thebibliography}

\end{document}